\documentclass[nohyperref]{article}

\usepackage{microtype}
\usepackage{graphicx}
\usepackage{subfigure}
\usepackage{booktabs} 

\usepackage{hyperref}

\usepackage[accepted]{icml2023}

\usepackage{amsmath}
\usepackage{amssymb}
\usepackage{mathtools}
\usepackage{amsthm}

\usepackage[capitalize,noabbrev]{cleveref}

\usepackage[textsize=tiny]{todonotes}

\usepackage{header}
\theoremstyle{plain}
\newtheorem{theorem}{Theorem}[section]
\newtheorem{proposition}[theorem]{Proposition}
\newtheorem{lemma}[theorem]{Lemma}
\newtheorem{corollary}[theorem]{Corollary}
\theoremstyle{definition}
\newtheorem{definition}[theorem]{Definition}

\theoremstyle{remark}
\newtheorem{remark}[theorem]{Remark}
\usepackage{wrapfig}

\usepackage[toc,page,header]{appendix}
\usepackage{minitoc}

\icmltitlerunning{Reinforcement Learning with History-Dependent Dynamic Contexts}

\begin{document}

\twocolumn[
\icmltitle{Reinforcement Learning with History-Dependent Dynamic Contexts}



\icmlsetsymbol{equal}{*}

\begin{icmlauthorlist}
\icmlauthor{Guy Tennenholtz}{equal,google}
\icmlauthor{Nadav Merlis}{equal,crest}
\icmlauthor{Lior Shani}{google}
\icmlauthor{Martin Mladenov}{google}
\icmlauthor{Craig Boutilier}{google}
\end{icmlauthorlist}

\icmlaffiliation{google}{Google Research}
\icmlaffiliation{crest}{CREST, ENSAE}

\icmlcorrespondingauthor{Guy Tennenholtz}{guytenn@gmail.com}

\icmlkeywords{Machine Learning, ICML}

\vskip 0.3in
]



\printAffiliationsAndNotice{\icmlEqualContribution} 

\begin{abstract}
    We introduce \emph{Dynamic Contextual Markov Decision Processes (DCMDPs)}, a novel reinforcement learning framework for history-dependent environments that generalizes the contextual MDP framework to handle non-Markov environments, where contexts change over time. We consider special cases of the model, with a focus on \emph{logistic DCMDPs}, which break the exponential dependence on history length by leveraging aggregation functions to determine context transitions. This special structure allows us to derive an upper-confidence-bound style algorithm for which we establish regret bounds. Motivated by our theoretical results, we introduce a practical model-based algorithm for logistic DCMDPs that plans in a latent space and uses optimism over history-dependent features. We demonstrate the efficacy of our approach on a recommendation task (using MovieLens data) where user behavior dynamics evolve in response to recommendations.
\end{abstract}

\section{Introduction}

Reinforcement learning (RL) is a paradigm in which an agent learns to act in an environment to maximize long-term reward. RL has been applied to numerous domains, including recommender systems, robot control, video games, and autonomous vehicles \citep{afsar2022reinforcement,tessler2019distributional,mnih2015human,fayjie2018driverless}. While typical RL approaches rely on a Markov property of both the reward process and environment dynamics,
many scenarios are inherently history-dependent \cite{BBG:NMRDPs,ronca:ijcai21}, particularly, when humans are involved. 
As one example, the behavior of users in recommender systems often exhibits non-Markovian characteristics reflective of a user's latent state, including: user preference elicitation sessions, where users respond to a sequence of feedback-gathering interventions (e.g., ratings, comparisons, annotations) \cite{chen_critiquing_survey:umuai2012,zhao13interactive}; user \emph{ad blindness} (i.e., the tendency to gradually ignore ads) \cite{hohnhold:kdd15}; and the long-term evolution of user satisfaction \cite{wilhelm:cikm18,advamp:ijcai19}. Many aspects of a user's latent state determine their disposition towards specific actions. For example, a user's level of frustration, trust, receptivity, and overall satisfaction, may affect their tendency toward accepting recommendations, providing feedback, or abandoning a session. Notably, such features are cumulatively impacted by the user's long-term history, which makes RL especially challenging due to difficult credit assignment, where the impact of any individual action is usually small and noisy. \footnote{A similar problem occurs in medical settings, where a patient's previous reactions to certain treatments could implicitly affect the physician's receptivity for treatment recommendations over long horizons. Another example includes human driver interventions in autonomous vehicles, where humans may take control of a vehicle for short periods of time.}

In this paper, we introduce \emph{Dynamic Contextual Markov Decision Processes (DCMDPs)} to model such environment dynamics in a \emph{\textbf{history-dependent} contextual} fashion. DCMDPs decompose the state space to include dynamic history-dependent contexts, where each context represents a different MDP, e.g., preferences of a human interacting with an agent, being affected by previous interactions.  Particularly, we introduce a special class of \emph{logistic DCMDPs}, in which context dynamics are determined by the aggregation of a set of feature vectors---functions of the immediate context, state and action---over time. This model is inspired by various psychological studies of human learning and conditioning; in particular, the Rescorla-Wagner (RW) model \cite{rescorla1972theory}, a neuroscience model which describes the diminishing impact of repeated exposure to a stimulus due to historical conditioning. 
Critically, this structure allows us to develop tractable, UCB-style algorithms \citep{auer2008near} for logistic DCMDPs that break the exponential dependence on history length in general DCMDPs.

Our contributions are as follows: (1) We introduce DCMDPs, a model that captures non-Markov context dynamics. (2) We introduce a subclass of DCMDPs for which state-action-context features are aggregated over time to determine context dynamics. We show how such problems can be solved by devising sample efficient and computationally tractable solutions, for which we establish regret bounds.  (3) Inspired by our theoretical results, we construct a practical algorithm, based on MuZero \citep{schrittwieser2020mastering}, and demonstrate its effectiveness on a recommendation system benchmark with long history-dependent contexts.

\begin{figure}[t!]
\centering
\includegraphics[width=\linewidth]{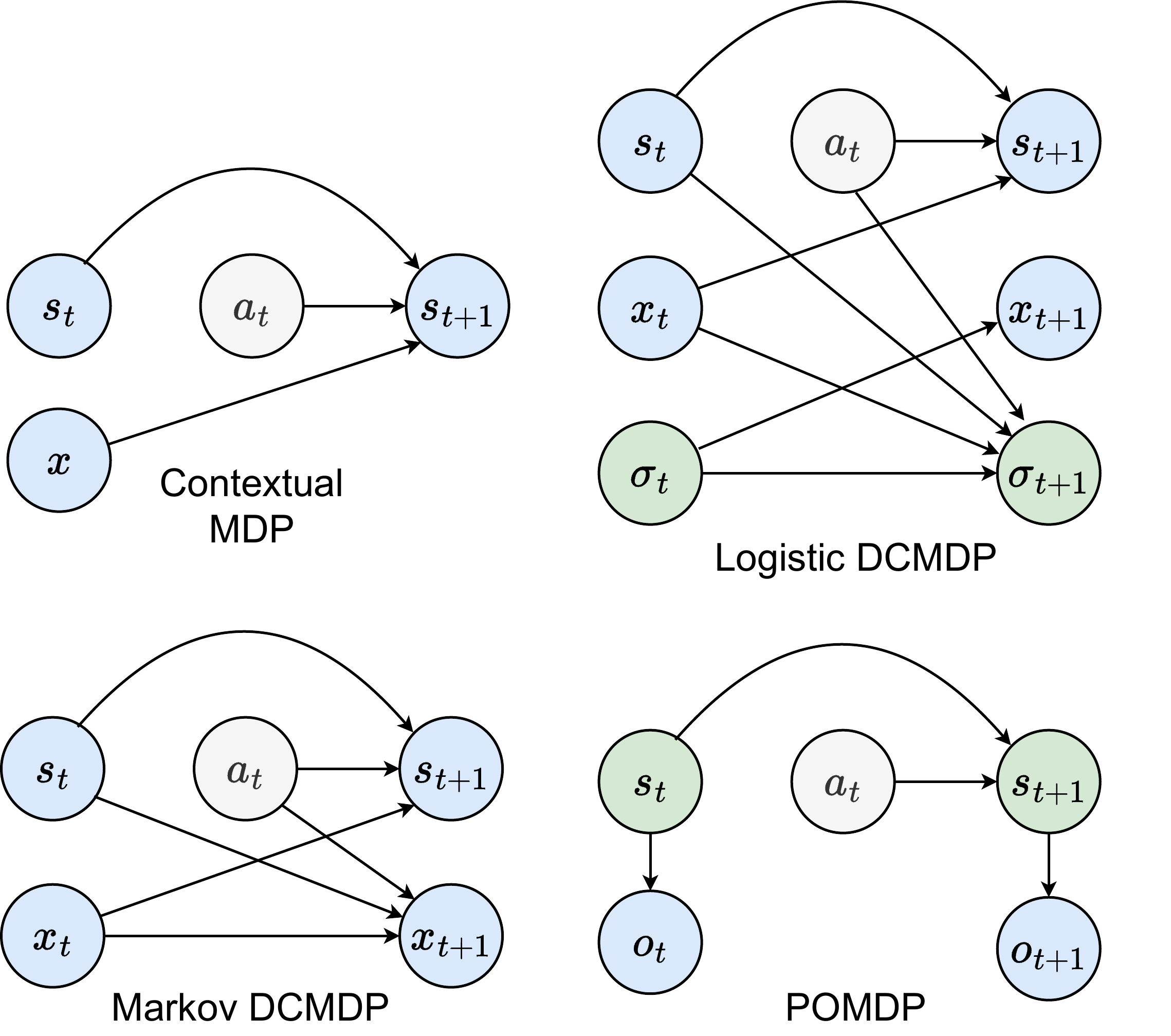}
\caption{\small Causal diagrams comparing Contextual MDPs \citep{hallak2015contextual}, Markov DCMDPs (\Cref{section: markov dcmdp}) and Logistic DCMDPs (\Cref{section: logistic DCMDPs}). Logistic DCMDPs are history dependent, where ${\sigma_t = \sum_{l=0}^{t-1}
\alpha^{t-l-1}\paramvectrue_l(s_l, a_l, x_l)}$, and $\paramvectrue_l:\s \times \A \times \X \mapsto \R^M$ are \emph{unknown}, non-stationary, vector valued feature mappings. Green circles represent unobserved variables.}
\label{fig: DCMDP causal diagram}
\end{figure}

\section{Dynamic Contextual MDPs}
\label{section: dcmdp}

We begin by defining \emph{Dynamic Contextual MDPs (DCMDPs)}, a general framework for modeling history-dependent contexts\footnote{The term ``context", as opposed to ``state", differentiates between the Markov part of the state and the history dependent part of the state. Additionally, contexts often quantify characteristics of the environment (e.g., types of humans-in-the-loop), which can evolve in a distinct fashion, in contrast to the rest of the state.}. Let $\s$, $\A$ and $\X$ be state, action, and context spaces, with cardinalities $S, A, X$, respectively. For any time $t \geq 1$, let $\mathcal{H}_t = \brk[c]*{\brk*{s_1, a_1, x_1, \hdots, s_t, a_{t-1}, x_{t-1}}}$
be the set of histories up to time $t$; and let $\mathcal{H} = \bigcup_t \mathcal{H}_t$.
We denote $(s_0, a_0, x_0) = \emptyset$. 

A DCMDP is given by the tuple $\brk*{\X, \s, \A, r, P, H}$, where, $r: \s \times \A \times \X \mapsto [0,1]$ is a reward function, $P: \mathcal{H} \times \s \times \A \mapsto \Delta_{\s}$ is a history-dependent transition function, and $H$ is the horizon. DCMDP dynamics proceeds in discrete episodes $k=1,2, \hdots, K$. At the beginning of episode $k$, the agent is initialized at state $s_1^k$. At any time $h$, the agent is in state $s_h^k$, has observed a history $\hist_h^k = (s_1^k, a_1^k, x_1^k, \hdots, s_{h-1}^k, a_{h-1}^k, x_{h-1}^k) \in \mathcal{H}_h$, and selects an action $a_h^k \in \A$.
Then, the next context $x_h^k$ occurs with (history-dependent) probability $P(x_h^k | \hist_h^k)$, the agent receives reward $r(s_h^k, a_h^k, x_h^k)$, and the environment transitions to state $s_{h+1}^k$ with probability $P_h(s_{h+1}^k | s_h^k, a_h^k, x_h^k)$. 

A policy $\pi: \s \times \mathcal{H} \mapsto \Delta_{\A}$ maps states and histories to distributions over actions. The value of $\pi$ at time $h$ is defined as $V_h^\pi(s, \hist) = \expect*{\pi}{\sum_{t=h}^H r(s_t, a_t, x_t) | s_h = s, \hist_h = \hist}$, where $a_t \sim \pi(s_t, \hist_t)$, and $x_t \sim P(\cdot | \hist_t)$. An optimal policy $\pi^*$ maximizes the value over all states and histories
; we denote its value function by $V^*$. We measure the performance of an RL agent by its \emph{regret} -- the difference between its value and that of an optimal policy: $\text{Reg}(K) = \sum_{k=1}^K V_1^*(s_1^k) - V_1^{\pi^k}(s_1^k)$. 

\Cref{fig: DCMDP causal diagram} depicts causal diagrams comparing general POMDPs to different types of DCMDPs, including three special cases: Contextual MDPs \citep{hallak2015contextual}
, Markov DCMDPs, and logistic DCMDPs (defined in the next two sections). DCMDPs are closely related to POMDPs, yet their causal structure allows us to devise more tractable solution (characterized by an aggregation function, as we'll see in \Cref{section: logistic DCMDPs}) which can efficiently and tractably capture very long histories. In the next section, we describe a simple instance of DCMDPs, for which contexts are Markov, and show that standard MDP solutions can be applied. Then, in \Cref{section: logistic DCMDPs}, we describe a more general DCMDP model, which uses aggregated features to represent histories, for which we provide sample efficient solutions and strong regret guarantees.

\subsection{Markov DCMDPs}
\label{section: markov dcmdp}

As a warm-up, we consider a simple version of DCMDPs in which context distributions are Markov w.r.t. the state and previous context. Specifically, we define a \emph{Markov DCMDP} as a DCMDP which satisfies for all $h \in [H]$, $\hist_h = (x_1, s_1, a_1, \hdots, x_{h-1}, s_{h-1}, a_{h-1}) \in \mathcal{H}_h$
    $
        P(x_h | \hist_h) = P(x_h | s_{h-1}, a_{h-1}, x_{h-1}).
    $
A Markov DCMDP $\M = (\X, \s, \A, r, P, H)$ can be reduced to an MDP  by augmenting the state space to include the context. To see this, we define the augmented MDP $\overline{\M} = (\bar \s, \A, \bar r, \bar P, H)$, where $\bar \s = \s \times \X$ and $
    \bar r(\bar s_t, a_t) 
    = 
    r(s_t, a_t, x_t)$,
    $\bar P(\bar s_{t+1} | \bar s_t, a_t) 
    =
    P(s_{t+1} | s_t, a_t, x_t) P(x_{t+1} | s_t, a_t, x_t).
$
As a consequence, the Markov DCMDP $\M$ and the MDP $\overline{\M}$ have the ``same'' optimal policy and value, and $\M$ can be solved using standard RL methods, e.g., using UCBVI \citep{azar2017minimax} one can obtain
${
    \Reg{K} \leq \tilde{\mathcal{O}} \brk*{\sqrt{H^3SAXK}}.}
$
Markov DCMDPs also generalize contextual MDPs in an especially simple way; but they fail to capture the history dependence of contexts embodied by general DCMDPs. In the next section, we turn to a special case of DCMDPs that does so, but also admits tractable solution methods.


\section{Logistic DCMDPs}
\label{section: logistic DCMDPs}

We introduce a general class of DCMDPs, called \emph{logistic DCMDPs}, where history dependence is structured using an aggregation of state-action-context-dependent features. Unlike Markov DCMDPs, logistic DCMDPs allow for context transitions to depend on history.

We define the softmax function $z_i: \R^M \mapsto [0,1]$, with temperature $\eta > 0$ as
\begin{align}
\label{eq: softmax}
    z_i(\bs{u})
    =
    \frac{\exp \brk*{\eta u_i}}
    {1 + \sum_{m=1}^M\exp \brk*{\eta u_m}}
\end{align}
for $i \in [M]$, $\bs{u} \in \R^{M}$, and $z_{M+1}(\bs{u}) = 1 - \sum_{i=1}^M z_i(\bs{u})$. 

\begin{definition}[Logistic DCMDP]
\label{def: logistic dcmdp}
    A \emph{logistic DCMDP} with latent feature maps $\brk[c]*{\paramvectrue_{h}: \s \times \A \times \X \mapsto \R^M}_{h=0}^{H-1}$ is a DCMDP with context space $\X = \brk[c]*{\x[i]}_{i=1}^{M+1}$, which satisfies, for all $h \in [H]$, $\hist_h = (s_1, a_1,x_1, \hdots, s_{h-1}, a_{h-1}, x_{h-1}) \in \mathcal{H}_h$, and $i \in [M+1]$:
    \begin{align*}
    P_{\paramvectrue}(\x[i]_h | \hist_h) 
    =  
    z_i\brk*{\sum_{t=0}^{h-1} \alpha^{h-t-1}\paramvectrue_t(s_t, a_t, x_t))},
    \end{align*}
    where $\alpha \in [0,1]$ is a \emph{history discount factor}.
\end{definition}
Note that the latent functions $\paramvectrue_h$ are vector-valued and \emph{unknown}. In a recommender system, $\paramvectrue_h$ may represent a user's unknown degree of trust in the system, or the effect of a sequence of recommendations on their satisfaction. The discount $\alpha$ allows for immediate effects to diminish over time (if less than 1).

A logistic DCMDP is denoted by $\brk*{\X, \s, \A, r, P, H, \paramvectrue, \alpha}$. We assume $\paramvectrue$ is $\ell_2$-bounded
with ${\sqrt{\sum \paramfunc_{h,i}^{*2}(s,a,x)} \leq L}$, and we denote
\begin{align}
\label{eq: param set}
\paramset = \brk[c]*{\paramvec: \abs{\paramfunc_{h,i}(s,a,x)} \leq b_{h,i}(s,a,x)}
\end{align}
the (rectangular) set where $b_{h,i}(s,a,x)$ are upper bounds on $\paramvectrue$. Throughout our analysis we denote the effective history horizon ${\Halpha = \frac{\alpha^{2H}-1}{\alpha-1}}$, and without loss of generality scale transitions in $z_i$ (\Cref{eq: softmax}) with temperature $\eta = H_\alpha^{-1}$.\footnote{We set $\eta=\Halpha^{-1/2}$ for convenience. Different choices of $\eta$ are equivalent to varying the bounds on $\paramset$ in \Cref{eq: param set}.} For clarity, we write $r(s, a, \x[i]) = r_i(s,a)$, $P(s'| s, a, \x[i]) = P_i(s' | s, a)$, and $\bs{r}(s,a) = (r_1(s,a), \hdots, r_{M+1}(s,a))^T$, $\bs{P}(s'|s,a) = (P_1(s'|s,a), \hdots, P_{M+1}(s'|s,a))^T$. We also denote by $n_h^k(s,a,x)$ the number of visits to $s,a,x$ at time step $h$ of episode $k-1$. 

Next, we define a sufficient statistic for logistic DCMDPs that will prove valuable in our solution methods that follow.
\begin{definition}[Sufficient Statistic] 
\label{defition: sufficient statistic}
Given a logistic DCMDP with feature maps $\paramvec$, define  ${\suff: \mathcal{H} \mapsto R^M}$ as
$
\suff(\hist_h; \paramvec) :=
\sum_{t=0}^{h-1}
\alpha^{h-t-1}\paramvec_t(s_t, a_t, x_t),
$
and the set of sufficient statistics by $\suffset(\paramvec) := \brk[c]*{\suff(\hist; \paramvec)}_{\hist \in \mathcal{H}}$.
\end{definition}
In \Cref{appendix: sufficient statistic}, we prove that $\suff(\hist_h; \paramvec)$ is a sufficient statistic of the history for purposes of computing the optimal policy at time $h$. We do so by defining an equivalent MDP with state space $\s \times \suffset(\paramvec)$ with well-defined dynamics and reward, and an equivalent optimal policy, which achieves the same optimal value. 

Finally, similar to previous work on logistic and multinomial bandits \citep{abeille2021instance,amani2021ucb}, we define a problem-dependent constant for logistic DCMDPs which plays a key role in characterizing the behavior of $M \geq 1$ mulitnomial logit bandit algorithms. For $\bs{x} \in \R^{M+1}$ and ${\hist \in \mathcal{H}}$, let $\bs{z}(\bs{x}) = \brk*{z_0(\bs{x}), \hdots, z_{M+1}(\bs{x})}^T$, $\bs{A}(\hist ; \paramvec) = \text{diag}(\bs{z}(\suff(\hist; \paramvec))) - \bs{z}(\suff(\hist; \paramvec))\bs{z}(\suff(\hist; \paramvec))^T$, and
$
    1/\kappamin
    = 
    \inf_{\hist \in \mathcal{H}} \lambda_{\min}\brk[c]*{\bs{A}(\hist ; \paramvectrue)}.
$ Informally, $\kappamin$ is related to saturation of the softmax $z_i$. For logistic DCMDPs, it is related to a worst-case context distribution w.r.t.\ $\paramvectrue$ and $\hist \in \mathcal{H}$. We refer to \citet{abeille2021instance,amani2021ucb} for details, as well as lower bounds using this constant in logistic bandits. 

\paragraph{The Rescorla-Wagner Model in Recommenders.}
Before continuing to provide sample efficient methods for solving logistic DCMDPs, we turn to motivate the aggregated model of history through the lens of the Rescola-Wagner (RW) model \citep{rescorla1972theory} in a recommendation setting.

Logistic DCMDPs generate context transitions based on the sum of specific features of prior states, actions, and contexts, as captured by $\paramvectrue$,  with backward discounting to diminish the effect of past features or experiences, as captured by $\alpha$.
Such a model can be used to capture a (very simple) RW formulation of user behavior in an interactive recommender system.
Let $I = \{i_1,\ldots, i_n\}$ be a set of items. A user may like, dislike, or be unfamiliar with any of these items, represented by $u \in \{1, 0, -1\}^n$. Let $g_t$ be the user's (latent) current degree of satisfaction or engagement with the system. At each time $t$, the system asks the user for their disposition (e.g., rating) of an item $i_t \in I$. The user decides to answer the question with probability $z_1(g_t)$ (\Cref{eq: softmax}), which is strictly increasing with higher degrees of engagement level. The engagement level then evolves as $g_{t+1} = \alpha g_t + \beta u_{i_t}$, where $\alpha \in [0,1]$, and $\beta$ is a user-specific sensitivity factor. This model gives rise to a logistic DCMDP, whose solution gives the optimal recommender system policy. Specifically, actions $a_t := i_t \in {I}$ are the questions asked by the system, $\paramfunc^*(s_t, a_t, x_t) = \beta u_{a_t}$ depends only on $a_t$, user engagement is $ g_h = \sum_{t=0}^{h-1} \alpha^{h-t-1} \paramfunc^*(s_t, a_t, x_t) = \sum_{t=0}^{h-1} \alpha^{h-t-1} \beta u_{i_t}$, $x_t$ is the decision whether to answer, and $s_t$ is the observation of the answer. 

\section{Optimistic Methods for Logistic DCMDPs}

Logistic DCMDPs' aggregation of features allow us to obtain sample efficient and computationally tractable solutions; namely, solutions which do not depend exponentially on history. In this section, we describe an optimistic algorithm for solving logistic DCMDPs and provide regret bounds. We focus on theoretical motivations here, and address computational tractability in the next section.

We first develop \emph{Logistic Dynamic Context Upper Confidence Bound (LDC-UCB)}, a general RL method for logistic DCMDPs with unknown latent features (see \Cref{alg: LDC-UCB}). At each episode $k$, LDC-UCB uses estimates of rewards $\hat r_{x,h}^k(s,a) = \frac{\sum_{k'=1}^k\indicator{x_h^{k'} = x, s_h^{k'} = s, a_h^{k'} = a}r_{h}^{k'}}{n_h^k(s,a,x)}$, transitions ${\hat P_{x,h}^k(s'|s,a) = \frac{\sum_{k'=1}^k\indicator{x_h^{k'} = x, s_h^{k'} = s, a_h^{k'} = a, s_{h+1}^{k'} = s'}}{n_h^k(s,a,x)}}$, and a projected estimate of $\estparamvec$, calculated by maximizing the regularized log likelihood: 
\begin{align}
    \label{eq: likelihood}
   \mathcal{L}^k_\lambda(\paramvec) 
   = 
   \sum_{k'=1}^{k} 
   \sum_{h=1}^{H-1} 
   \sum_{i=1}^{M+1}
   \indicator{x_h^k = i}
   \ell_{i,h}^k(\paramvec)
    - 
    \lambda\norm{\paramvec}_2^2,
\end{align}
where $\ell_{i,h}^k(\paramvec) = \log \brk*{z_i(\suff(\hist_h^k; \paramvec))}$, $\lambda > 0$, and recall that $\suff(\hist_h^k; \paramvec) = \sum_{t=0}^{h-1} \alpha^{h-t-1}\paramvec_t(s_t^k, a_t^k, x_t^k)$.

\begin{algorithm}[t!]
\caption{LDC-UCB}
\label{alg: LDC-UCB}
\begin{algorithmic}[1]
\FOR{$k=1, \hdots, K$}
    \STATE $\bar r_{i,h}^{k}(s,a) = \hat r_{i,h}^{k}(s,a) + b^k_{i,h}(s,a), \forall i,h,s,a$ 
    \STATE $\bar \pi^k \gets \text{Optimistic Planner}\brk*{\bar \M_k(\delta)}$
    \hfill {\color{gray}// Eq.~\ref{eq: optimistic dcmdp set}}
    \STATE Rollout a trajectory by acting $\bar \pi^k$
    \STATE $\estparamvec^k \in \arg\max_{\paramvec \in \C_k(\delta)} \mathcal{L}^k_\lambda(\paramvec)$ \hfill {\color{gray}// Eq.~\ref{eq: likelihood}}
    \STATE Update $\hat{P}_i^{k+1}(s,a), \hat{r}_i^{k+1}(s,a), n^{k+1}(s,a,x)$ over rollout trajectory
\ENDFOR
\end{algorithmic}
\end{algorithm}

We account for uncertainty in these estimates
by incorporating optimism. For rewards and transitions, we add a bonus term $b^k_{i,h}$ (see \Cref{appendix: regret analysis ldc-ucb} for explicit definitions) to the estimated reward (line 2). To incorporate optimism in the latent features 
$\estparamvec$, we build on results from multinomial logistic bandits \citep{amani2021ucb}. Specifically, we derive a confidence bound over $\estparamvec$, for which with probability at least $1-\delta$
\begin{align}
    \norm{g_k(\paramvectrue) - g_k(\estparamvec_t)}_{\bs{H}_k^{-1}(\paramvectrue)} \leq \beta_k(\delta),
    \label{eq: global bound}
\end{align}
where $H_k(\paramvec) = -\nabla^2_{\paramvec} \mathcal{L}^k_\lambda(\paramvec)$, ${g_k(\paramvec) = -\nabla_{\paramvec}\mathcal{L}^k_\lambda(\paramvec) + D_k}$, $\beta_k(\delta) = \frac{M^{5/2}SAH}{\sqrt{\lambda}}\brk*{\log\brk*{1 + \frac{k}{d\lambda}} + 2\log\brk*{\frac{2}{\delta}}} + \sqrt{\frac{\lambda}{4M}} + \sqrt{\lambda}L$. See \Cref{appendix: confidence sets} for exact expressions and a proof of the bound in \Cref{eq: global bound}.

Next, we leverage the bound in \Cref{eq: global bound} to construct a feasible set of logistic DCMDPs. Specifically, we define the confidence set
\begin{align}
\label{eq: global confidence set}
\C_k(\delta)
    =
    \brk[c]*{\paramvec \in \paramset : \norm{g_k(\paramvec) - g_k(\hat \paramvec_t)}_{\bs{H}_k^{-1}(\paramvec)} \leq \beta_k(\delta)}.
\end{align}
and the following set of logistic DCMDPs:
\begin{align}
    \bar \M_k(\delta) = \brk[c]*{\brk*{\X, \s, \A, \bar r, \hat P, H, \paramvec, \alpha} : \paramvec \in \C_k(\delta)}.
\label{eq: optimistic dcmdp set}
\end{align}
The optimistic policy $\bar \pi^k$ (line 3) is that with greatest value over all DCMDPs in $\bar \M_k(\delta)$, i.e., $\bar \pi^k$ corresponding to $\max_{\bar m \in \bar \M_k(\delta)} V^*(s_1 ; \bar m)$. 

Combining the above, we prove the following regret guarantee for \Cref{alg: LDC-UCB}.

\begin{restatable}{theorem}{ldcucb}
\label{thm: regret ldc-ucb}
    Let $\lambda = \Theta(\frac{H M^{2.5}SA}{L})$. With probability at least $1-\delta$, the regret of \Cref{alg: LDC-UCB} is 
    \begin{align*}
        \Reg{K} 
        \leq 
        \tilde{\Ob}\brk{\sqrt{H^6 M^{4.5} S^2 A^2L^2 \kappamin K}}.
    \end{align*}
\end{restatable}
The proof of \Cref{thm: regret ldc-ucb} can be found in \Cref{appendix: regret analysis ldc-ucb}. We note that computing the optimistic policy over $\bar \M_k(\delta)$ (line 3) is computationally difficult, especially due the history dependence of $\pi$ on the accumulated latent features $\sum_{t=1}^h \alpha^{h-t}\paramvec(s_t, a_t, x_t)$. We address this challenge next.

\section{Mitigating Computational Complexity}
\label{sec: mitigating computational complexity}

In this section we show how to relax LDC-UCB (\Cref{alg: LDC-UCB}) to mitigate its high computational complexity. Importantly, we maintain regret guarantees similar to those of \Cref{thm: regret ldc-ucb} while obtaining an exponential improvement to computational cost. We later use these results to construct a practical model-based algorithm in \Cref{section: dczero}.

To address the computational challenges of \Cref{alg: LDC-UCB}, we focus on two problems. The first involves the set $\C_k(\delta)$ (\Cref{eq: global confidence set} and line 5 of \Cref{alg: LDC-UCB}) -- where computation of the maximum likelihood constrained to set $\C_k(\delta)$ is intractable. To address this, we prove that the constraint on the maximum likelihood estimator can be replaced by a simpler, rectangular set, enabling efficient calculation of the projected maximum likelihood.
The second challenge is the complexity of the optimistic planner (\Cref{eq: optimistic dcmdp set} and line 3 of \Cref{alg: LDC-UCB}). To overcome this, we develop a \emph{local} confidence bound, for every state-action-context triple $(s,a,x)$, and show it can be leveraged to design an optimistic planner, using a novel thresholding mechanism for optimism in logistic DCMDPs.
Pseudocode for this tractable variant of LDC-UCB is presented in \Cref{alg: LDC-UCB Tractable}.

\subsection{A Tractable Estimator}
\label{section: tractable estimator}

We begin by constructing a tractable estimator
for the latent feature maps $\paramvectrue$ which, instead of projecting to the set $\C_k(\delta)$, solves for projected maximum likelihood on the rectangular set $\paramset$ (\Cref{eq: param set}). Let $\gamma_k(\delta) = \brk*{2+2L\sqrt{MH} + \sqrt{2(1+L)}}\beta_k + \sqrt{\frac{2(1+L)HM}{\lambda}}\beta_k^2(\delta)$. We define the tractable maximum likelihood estimator 
$
\estparamvec^k_T \in \arg\max_{\paramvec \in \paramset} \mathcal{L}^k_\lambda(\paramvec), 
$
and have the following bound.

\begin{lemma}
\label{prop:convex relaxation bound} With probability at least $1 - \delta$, for all $k \in [K]$,
\begin{align}
    \norm{\estparamvec^k_T -\paramvectrue}_{{\bs{H}}_k(\paramvectrue)} \leq \gamma_k(\delta).
    \label{eq: f-normalized upper bound}
\end{align}
\end{lemma}
The proof (see \Cref{appendix: convex relaxation}) uses a convex relaxation of the set $\C_k(\delta)$. Notice that the confidence region for $\estparamvec^k_T$ is looser than that for $\estparamvec^k$ (see \Cref{eq: global bound}), as $\beta_k(\delta) < \gamma_k(\delta)$. Nevertheless, its computation is tractable.

Next we can exploit the confidence bound in \Cref{eq: f-normalized upper bound} to construct a \emph{local} bound for every state-action-context triple $(s,a,x)$ using the number of visits
to $(s,a,x)$, i.e., $n_h^k(s,a,x)$. The following result uses structural properties of logistic DCMDPs to achieve a local bound for $\estparamvec^k_T$. Its proof generalizes the local confidence bound in \citet{tennenholtz2022reinforcement}, and can be found in \Cref{appendix: local confidence bound}. 

\begin{lemma}[Local Estimation Confidence Bound]
\label{lemma: local feature confidence}
For any $\delta>0$, with probability of at least $1-\delta$, for all $k\in [K], h\in[H], i\in [M]$ and  $s,a,x\in\s\times\A\times\X$, it holds that 
\begin{align*}
    \Big\lvert \big(\hat \paramvec^k_T(s,a,x)\big)_{i,h} \!- \! \big(\paramvectrue(s,a,x)\big)_{i,h} \Big\rvert
    \leq  
    \frac{2\sqrt{\kappamin}\gamma_k(\delta)}{\sqrt{n_h^k(s,a,x) \!+\! 4\lambda}}.
\end{align*}
\end{lemma}
\Cref{lemma: local feature confidence} allows one to reason about the unknown features locally for any visited $(s,a,x)$, a vital step toward an efficient optimistic planner. Indeed, as we see in the next section, the cost of planning in logistic DCMDPs can be reduced significantly 
using this bound.

\subsection{Threshold Optimistic Planning}
\label{section: optimistic planning}

We now address the major computational challenge of \Cref{alg: LDC-UCB} -- the complexity of optimistic planning (line 3 of \Cref{alg: LDC-UCB}). To do this, we leverage the local bound in \Cref{lemma: local feature confidence} and construct an optimistic planner using a novel threshold mechanism, as we describe next.


Recall the set of sufficient statistics $\suffset(\paramvec) = \brk[c]*{\suff(\hist; \paramvec)}_{\hist \in \mathcal{H}}$ (\Cref{defition: sufficient statistic}), which is a finite, vector-valued set with cardinality $\abs{\suffset(\paramvec)} = \mathcal{O}\brk*{(SAMH)^{MH}}$, making planning in state space $S \times \suffset$ exponentially hard. Consequently, searching for the optimistic DCMDP in the space of feature maps satisfying $\paramvec \in \C_k(\delta)$ (\Cref{eq: optimistic dcmdp set}) requires searching over an exponentially large space. 

We mitigate this problem exponentially, by leveraging the local confidence bound in \Cref{lemma: local feature confidence}. Let $\rectset \subset \R^M \times \R^M$ be the rectangular cuboid of all candidate confidence intervals satisfying the bound in \Cref{lemma: local feature confidence}. That is, $\rectset$ is the set of all $M$ dimensional intervals $\brk[s]*{\bs{l}_h^k(s,a,x), \bs{u}_h^k(s,a,x)}$, such that for all $h,s,a,x$, $\paramvectrue_h(s,a,x) 
    \in 
    \brk[s]*{\bs{l}_h^k(s,a,x), \bs{u}_h^k(s,a,x)}$,
where,
$
    \bs{u}_h^k(s,a,x), \bs{l}_h^k(s,a,x)
    =
    {\estparamvec^k_T 
    \pm \brk*{\frac{2\sqrt{\kappamin}\gamma_k(\delta)}{\sqrt{n^k_{h}(s,a,\x[1]) + 4\lambda}}, \hdots, \frac{2\sqrt{\kappamin}\gamma_k(\delta)}{\sqrt{n^k_{h}(s,a,\x[M]) + 4\lambda}}}^T}
$.
In what follows, we identify key characteristics of the optimistic value when optimized over 
$\rectset$. Specifically, we show that an optimistic solution lies on the extreme points of $\rectset$, but more importantly, at one of \emph{$M$ specific extreme points}. This limits the search required by optimistic planning to a much smaller set, which can be approximated effectively in practice.

\paragraph{Optimism in intervals. } 

Instead of augmenting the state space with $\suffset(\paramvec)$, we use the set of confidence intervals defined by $\rectset$. We denote by $\ci_h^k: \suffset(\estparamvec^k) \mapsto \R^M \times \R^M$ the confidence interval of the sufficient statistic $\suff(\hist_h^k, \estparamvec^k)$. That is,
\begin{align*}
    &\ci_h^k
    =
    \ci(\suff(\hist_h^k; \estparamvec^k)) 
    \\
    &=
    \brk[s]*{
        \sum_{t=0}^{h-1}
        \alpha^{h-t-1}
        \bs{l}_h^k(s_t^k,a_t^k,x_t^k), 
        \sum_{t=0}^{h-1}
        \alpha^{h-t-1}
        \bs{u}_h^k(s_t^k,a_t^k,x_t^k)
    }.
\end{align*}
We also denote by $\ciset^k = \brk[c]*{\ci(\suff(\hist, \estparamvec^k_T))}_{\hist \in \mathcal{H}}$ the set of possible confidence intervals over $\rectset$ in episode $k$.

Next, we augment the state space $\s$ at every episode $k$ by $\s \times \ciset^k$, and define the augmented state-action optimistic value for context $i \in [M+1]$ and confidence interval $\ci_h^k = \ci(\suff(\hist_h^k, \paramvec^k))$ at time step $h \in [H]$ by 
\begin{align*}
    \bar Q_i(s,a, \ci_h^k)) 
    =
    \bar r_i(s,a) 
    + \expect*{s' \sim \hat P_i(\cdot|s,a)}{\bar V_{h+1}(s', \ci_{h+1}^k)},
\end{align*}
where, with slight abuse of notation, we used $\ci_{h+1}^k = \ci\brk*{\suff\brk*{\hist_h^k \cup \brk[c]*{s, a, \x[i]}, \estparamvec^k_T}}$ to denote the next aggregated confidence interval.
The optimistic value $\bar V_h$ is defined by maximizing over sufficient statistics in the confidence set $\ci_h^k$ and $a\in\A$. That is,
\begin{align}
    \bar V_h(s, \ci_h^k)
    =
    \max_{a \in \A}
    \max_{\bar \suff \in \ci_h^k} 
    \sum_{i=1}^{M+1}
    z_i(\bar \suff) Q_i(s,a, \ci_h^k)
    \label{eq: value hard maximization}
\end{align}
Indeed, $\bar V_h$ is an optimistic value, as shown by the following proposition. Its proof is provided in \Cref{appendix: proof of optimistic value proposition}.
\begin{restatable}[Optimistic Value]{proposition}{OptimisticValue}
\label{prop: optimism tractable ldc-ucb}
    Let $\bar V_h$ as defined in \Cref{eq: value hard maximization}. Then, w.h.p. $\bar V_1(s_1^k, \ci_1^k) \geq V^*_1(s_1^k)$.
\end{restatable}

\begin{algorithm}[t!]
\caption{Tractable LDC-UCB}
\label{alg: LDC-UCB Tractable}
\begin{algorithmic}[1]
\FOR{$k=1, \hdots, K$}
    \STATE $\bar r_{i,h}^{k}(s,a) = \hat r_{i,h}^{k}(s,a) + b^k_{i,h}(s,a), \forall i,h,s,a$ 
    \STATE $\bar \pi^k \gets \texttt{Optimistic DP}\brk*{\bar r^k, \hat P^k, \rectset}$ \hfill {\color{gray}// Eq.~\ref{eq: value hard maximization}}
    \STATE Rollout a trajectory by acting $\bar \pi^k$
    \STATE $\estparamvec^{k+1}_T \in \arg\max_{\paramvec \in \paramset} \mathcal{L}^k_\lambda(\paramvec)$ \hfill {\color{gray}// Eq.~\ref{eq: likelihood}}
    \STATE Update $\hat{P}_i^{k+1}(s,a), \hat{r}_i^{k+1}(s,a), n^{k+1}(s,a,x)$ over rollout trajectory
\ENDFOR
\end{algorithmic}
\end{algorithm}

Next, we turn to show that the maximization problem in \Cref{eq: value hard maximization} can be solved efficiently, though $\ci_h^k$ is an exponentially large set. Notice that the inner term $\sum_{i=0}^M z_i(\bar \suff) Q_i(s,a, \ci_h^k)$ in \Cref{eq: value hard maximization} is not convex. Still, our analysis shows that a solution to the inner maximization problem lies in the set of extreme points of $\ci_h^k$. That said, these $2^M$ extreme points make exhaustive search intractable. Fortunately, we can also show that the optimal solution lies in a space of exactly $M$ solutions -- a linearly sized, tractable search space.

To this end, we define the threshold set, which we will use to construct the (linear) set of feasible extreme points.
\begin{definition}
    For a rectangular cuboid defined by the interval $\ci = [\bs{l}, \bs{u}] \subseteq \R^{M+1} \times \R^{M+1}$, vector $\bs{y} \in \R^{M+1}$ and real number $t \in \R$ we define $\thr{t}{\bs{y},\ci} \in \R^{M+1}$ by
    \begin{align*}
        \brk[s]*{\thr{t}{\bs{y},\ci}}_i
        =
        \begin{cases}
            l_i & y_i < t\\
            u_i & \text{o.w.}
        \end{cases}
    \end{align*}
\end{definition}
\begin{definition}
\label{def: threshold set}
    For a vector $\bs{Q} \in \R^{M+1}$, we define the threshold set $\thrset{\bs{Q}} = \brk[c]*{\frac{Q_i + Q_{i+1}}{2}}_{i=1}^{M}$.
\end{definition}

We use these definitions to show that the optimal solution to \Cref{eq: value hard maximization} lies in the threshold set of $Q$-values (see proof in \Cref{appendix: proof of threshold optimism}).
\begin{restatable}[Threshold Optimism]{lemma}{ThresholdOptimism}
\label{lemma: threshold optimism}
    Let $\bs{Q} \in \R^{M+1}$. For any $\bs{x} \in \R^{M+1}$ such that $x_i = 0$ define $f(\bs{x}) = \sum_{i=1}^{M+1} z_i(\bs{x}) Q_i$. Let $\ci = [\bs{l}, \bs{u}] \subseteq \R^{M+1} \times \R^{M+1}$ and assume that $\bs{l} < \bs{u}$. Then, there exists $t \in \thrset{\bs{Q}}$ such that $\thr{t}{\bs{Q},\ci} \in \arg\max_{\bs{x} \in \ci} f(\bs{x})$.
\end{restatable}

We can now leverage \Cref{lemma: threshold optimism} to solve the inner maximization in \Cref{eq: value hard maximization}.
For notational convenience, we write $\bar Q_i = \bar Q_i(s,a, \ci_h^k)$ and $\bs{Q} = \brk*{Q_1, \hdots, Q_{M+1}}^T$. Applying \Cref{lemma: threshold optimism}, we get that
\begin{align}
    \max_{\bar \suff \in \ci_h^k} 
    \sum_{i=1}^{M+1}
    z_i(\bar \suff)
    \bar Q_i 
    =
    \max_{t \in \thrset{\bs{Q}}}
    \sum_{i=1}^{M+1}
    z_i\brk*{\thr{t}{\bs{Q},\ci_h^k}}
    \bar Q_i.
    \label{eq: threshold maximization}
\end{align}
As a result, the non-convex maximization problem in \Cref{eq: threshold maximization} reduces the search space to $M$ optimistic candidates.

\subsection{Putting It All Together} 

Using \Cref{lemma: threshold optimism} and particularly its derived corollary in \Cref{eq: threshold maximization}, we construct an optimistic planner, denoted by \texttt{Optimistic DP}, which plans via dynamic programming using \Cref{eq: threshold maximization}; we refer to \Cref{appendix: threshold optimistic planning} for an explicit formulation of the optimistic planner. Finally, using the tractable estimator $\estparamvec^k_T$, and the threshold optimistic planner, we present a tractable variant of LDC-UCB in \Cref{alg: LDC-UCB Tractable}, for which we have the following regret guarantee.
\begin{restatable}{theorem}{tractableLdcUcb}
\label{thm: regret tractable ldc-ucb}
    Let $\lambda = \Theta(\frac{HM^{2.5}SA}{L})$. With probability at least $1-\delta$, the regret of \Cref{alg: LDC-UCB Tractable} is 
    \begin{align*}
        R(K) \leq 
         \tilde \Ob\brk*{\sqrt{H^8 M^{6.5} S^2 A^2  L^4 \kappamin K}}.
    \end{align*}
\end{restatable}
The proof of the theorem can be found in \Cref{appendix: regret analysis tractable algorithm}. As expected, the regret upper bound in \Cref{alg: LDC-UCB Tractable} is worse than that of \Cref{alg: LDC-UCB} by a factor of $\tilde{\Ob}\brk*{HML}$. This result is strongly affected by the looser bound for the tractable feature maps in \Cref{lemma: local feature confidence}. Nevertheless, the intractability of \Cref{alg: LDC-UCB} compared to the tractability of \Cref{alg: LDC-UCB Tractable} suggests this is a more-than-reasonable tradeoff. Moreover, our tractable variant of LDC-UCB gives rise to practical optimistic algorithms, as we demonstrate next.

\section{DCZero}
\label{section: dczero}

Motivated by our theoretical results, we present a practical model-based optimistic algorithm for solving DCMDPs. We build on MuZero \citep{schrittwieser2020mastering}, a recent model-based algorithm which constructs a model in latent space and acts using Monte Carlo Tree Search (MCTS, \citet{coulom2007efficient}). MuZero uses representation, transition, and prediction networks for training and acting. The representation network first embeds observations in a latent space, after which planning takes place using the transition and prediction networks through a variant of MCTS. Importantly, instead of predicting the next state (e.g., using world models \citep{hafner2023mastering}), MuZero trains its latent space by predicting three quantities---the reward, value, and current policy---by rolling out trajectories in latent space (see \citet{schrittwieser2020mastering} for further details).

We develop DCZero, an algorithm based on MuZero for DCMDPs (see \Cref{alg: DCZero}). Like MuZero, DCZero uses representation, transition, and prediction networks to learn and act in the environment. In contrast to MuZero, DCZero trains an additional ensemble of networks to estimate the unknown features $\paramvectrue$ using cross-entropy. Estimated quantities of the ensemble are used to construct confidence intervals for the sufficient statistics, which are used to augment the state. DCZero uses $M+1$ transition networks (one for each context), and predicts $M+1$ reward functions. To incorporate optimism, the value function is trained optimistically using the thresholding technique in the previous section, where rewards for unseen actions are sampled from the trained reward models $r_i$ and next states are sampled from the trained models $P_i$.

\paragraph{Movie Recommendation Environment.}

To evaluate the effectiveness of DCZero, we develop a movie recommendation environment based on the MovieLens dataset \citep{harper2015movielens}. Users and items are represented in embedding space computed using SVD of the MovieLens ratings matrix. Each of $n$ users is assigned a set of $M$ possible user embeddings; i.e., each user $u \in \brk[c]*{u^{(i)}}_{i=1}^n$ is assigned a set of preference vectors $\bs{x} = \brk[c]*{\bs{x}^{(j)}}_{j=1}^{M+1}, \bs{x}^{(j)} \in \R^d$. Intuitively, these vectors reflect distinct user preferences corresponding to some aspect of the user's latent state (e.g., mood or current interest \cite{cen:kdd20}; location, companions, or activity; level of trust or satisfaction with the system) and hence influence $u$'s behavior.

The recommendation agent interacting with a user selects an item $x$ from a random set of $A$ movies, $\brk[c]{\bs{v}^{(a)}}_{a=1}^A, \bs{v}^{(a)} \in \R^d$, and recommends it. The user context then evolves according to some history-dependent dynamics represented by a logistic DCMDP. Specifically, we assume unknown latent features $\paramvectrue(\bs{x},\bs{v})$ with the user's aggregated features (at time $h \in [H]$, episode $k$) being:  $\suff_{k,h} = \sum_{t=0}^{h-1} \alpha^{h-t-1} \paramvectrue(\bs{x}_{k}^{(j_t)},\bs{v}^{(a_t)})$.
The agent recommends movie $\bs{v}^{(a)}$ to the user, while the user preference vector is sampled as $\bs{x}_k^{j_h} \sim z(\suff_{k,h})$. The agent then receives a reward $r_j(\bs{x}, a) = (\bs{x}^{(j_h)}_k)^T \bs{\Sigma} \bs{v}^{(a)}$ reflecting the user's (current) preference for the movie, and the user's latent state transitions given the unknown function $\paramvectrue(\bs{x},\bs{v})$ and discount $\alpha$; that is, $\suff_{k,h+1} = \alpha \suff_{k,h} + \paramvectrue(\bs{x}^{(j_h)}_{k}, \bs{v}^{(a)})$.

We test our methods in two variants of this environment. In the first, ``AttractionEnv'', user latent features $\paramvectrue$ are correlated with the user's 
degree of preference for the recommended movie:
\begin{align*}
    \paramvectrue(\bs{x}^{(j)}, \bs{v})
    =
    \mu\brk1{(\bs{x}^{(j)})^T \bs{\Sigma} \bs{v}}, 
    \tag{Attraction}
\end{align*} 
where $\mu$ is a component-wise monotonically increasing function. AttractionEnv reflects users with a tendency to desire content similar to those they most recently consumed. This may reflect the positive influence of exposure to new types of content, increased familiarity increasing preference, or content domains (such as music) where some mild consistency of experience is preferred to jarring shifts in style or genre.
The second environment, ``NoveltyEnv'', reflects a contrasting dynamics in which user latent features evolve such that
$\paramvectrue$ is anti-correlated
with the user's preference for the recommended movie:
\begin{align*}
    \paramvectrue_i(\bs{x}^{(j)}, \bs{v})
    =
    \begin{cases}
        -\mu\brk*{(\bs{x}^{(j)})^T \bs{\Sigma} \bs{v}}
        & ,j = i \\
        \mu\brk*{(\bs{x}^{(j)})^T \bs{\Sigma} \bs{v}}
        & ,\text{o.w.}
    \end{cases}
    \tag{Novelty}
\end{align*}
As a result, movies that previously appealed to the user become less preferred, reflecting a desire for novelty over short time periods.
\begin{algorithm}[t!]
\caption{DCZero}
\label{alg: DCZero}
\begin{algorithmic}[1]
\STATE{ \textbf{require:}} Size of ensemble $B$
\STATE{ \textbf{init:}} Replay buffer $\mathcal{R} \gets \emptyset$
\FOR{$k=1, 2, \hdots$}
    \STATE Train bootstrap ensemble of $B$ feature maps $\brk[c]*{\estparamvec_{\bs{\theta}_b}: \s \times \A \times \X \mapsto \R^M}_{b=1}^B$ over $\mathcal{R}$ using cross-entropy loss.
    \STATE Augment state $s$ with aggregated feature confidence ${\ci_h \gets \sum_{t=0}^{h-1} \alpha^{h-t-1} \text{std}_{b} \brk*{\brk[c]*{ \estparamvec_{\bs{\theta}_b}(s_t, a_t, x_t)}_{b=1}^B}}$.
    \STATE Train threshold optimistic value estimator over $\mathcal{R}$
    \begin{footnotesize}
    \begin{align*}
        &\bar Q_{i,\bs{\psi}}(s,a, \ci_h)) 
        =
        \hat r_i(s,a) 
        + 
        \gamma \mathbb{E}_{s' \sim \hat P_i(s,a)}\bar V_{\bs{\psi}}(s', \ci_{h+1}), \\
        &\bar V_{\bs{\psi}}(s, \ci_h)
        =
        \max_{t \in \thrset{\bs{Q}_{\bs{\psi}}}}
        \sum_{i=1}^{M+1}
        z_i\brk*{\thr{t}{\bs{Q}_{\bs{\psi}},\ci_h^k}}
        \bar Q_{i,\bs{\psi}}^{\pi_{\bs{\phi}}(s)}.
    \end{align*}
    \end{footnotesize}
    \STATE Act and train representation network, $M$ transition networks (for each $\hat P_i$), and $M$ prediction networks (for each $\hat r_i$) using \texttt{MuZero-ALG} with the optimistic value $\bar V_{\bs{\psi}}$ in MCTS. Return replay buffer $\mathcal{R}$.
\ENDFOR
\end{algorithmic}
\end{algorithm}

\begin{figure*}[t!]
\centering
\includegraphics[width=\linewidth]{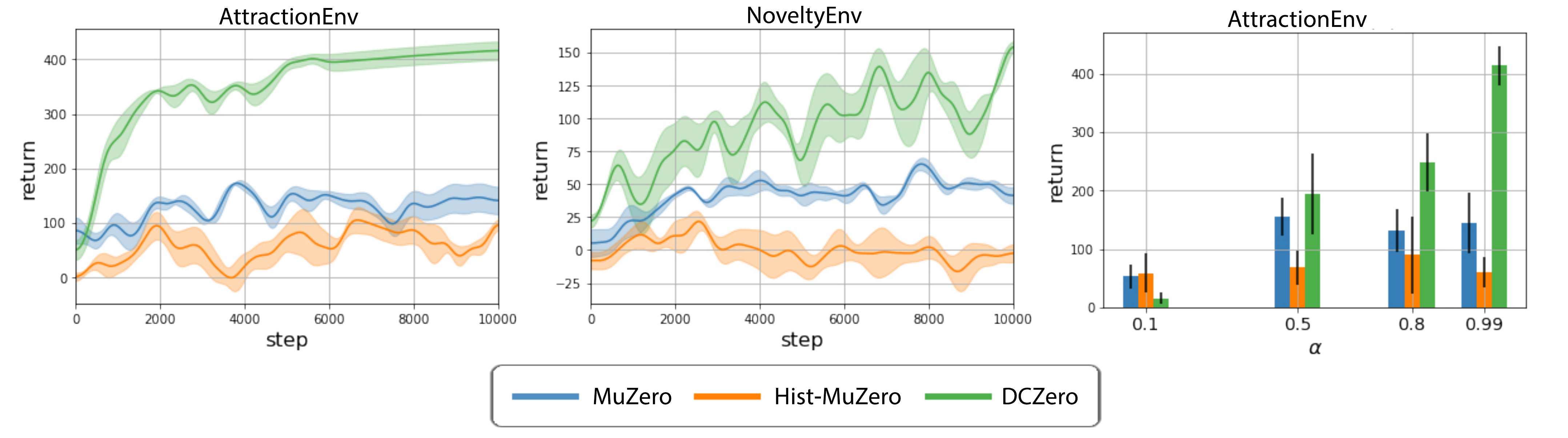}
\caption{\small Plots comparing MuZero, Hist-Muzero, and DCZero on the AttractionEnv(left) and NoveltyEnv (middle). We also compare results for different values of $\alpha$ (right). All experiments show mean scores with 95\% confidence intervals.
}
\label{fig: experiments}
\end{figure*}

\paragraph{Experiments.}

All experiments used a horizon of ${H=300}$, $M=6$ user classes, $A=6$ slate items (changing every reset), and a user embedding dimension of $d=20$. We used default parameters for MuZero and applied the same parameters to DCZero. We compared DCZero and MuZero on the AttractionEnv and NoveltyEnv environments. We also tested a history-dependent variant of MuZero, which uses the sequence of past movies and contexts to densely represent history. More specifically, Hist-MuZero uses a stack of $30$ previous observations as its state. We implemented both MLP and Transformer-based model architectures, but present results for the Transformer, as both had similar performance.

\Cref{fig: experiments} shows these comparisons. The plots compare the return of DCZero with the two baselines on AttractionEnv and NoveltyEnv with $\alpha=0.99$; we also vary the values of $\alpha$ on the AttractionEnv. We see that DCZero is able to outperform both baselines, with significant increases in performance for larger values of $\alpha$ (i.e., longer history dependence). This suggests that DCZero can be especially beneficial in problems that exhibit long history dependence. Interestingly, we note that using a dense history-dependent Transformer hurts performance, except for very small values of $\alpha$ (indeed, only for $\alpha = 0.1$ does the sequence model outperform the other methods).

\section{Related Work}


Contextual MDPs \citep{hallak2015contextual} have proven useful in a numerous studies \citep{jiang2017contextual,zintgraf2019varibad,kwon2021rl}. Contexts are sampled once and are fixed throughout the episode. DCMDPs can be seen as a generalization of contextual MDPs, where contexts can change over time in a realistic, history-dependent fashion. Other forms of DCMDPs, are interesting directions for future work, including DCMDPs for which contexts change slowly in time. In \citet{chenadaptive} a latent context variable changes abruptly at discrete points in time. Our logistic DCMDP considers history-dependent dynamics of contexts, which can depend on previous states and actions. Moreover, our model can capture smoother behavior which changes very slowly over time (over long histories). Finally, in contrast to \citet{chenadaptive}, our work provides theoretical guarantees, showing statistical and computational efficiency of our approach. In \citet{maovariance}, a non-stationary contextual environment is considered, yet the contexts are not allowed to depend on previous states and actions. \citet{renreinforcement} propose a Bayesian approach for learning contextual MDPs for which contexts can change dynamically. Nevertheless, their model assumes dynamics that are not state-action dependent, and not history dependent.

Partially observable MDPs are widely studied \citep{papadimitriou1987complexity,vlassis2012computational,krishnamurthy2016pac,tennenholtz2020off,xiong2022sublinear}. As POMDPs are inherently history dependent, recent work has identified models and assumptions for which sample-efficient algorithms can be derived \citep{xiong2022sublinear,liu2022partially,liu2022optimistic}. Nevertheless, such solutions are often computationally intractable, impeding their practical implementation. With DCMPDs, we focus on specific forms of history-dependence,
and show them to be computationally tractable, as well as effectively deployable.

\citet{tennenholtz2022reinforcement} define TerMDPs, a framework which models exogenous, non-Markovian termination in the environment. Once terminated, the agent stops acting and accrues no further rewards. TerMDPs capture various scenarios in which exogenous actors disengage with the agent (e.g., passengers in autonomous vehicles or users abandoning a recommender), and can be shown to be a special case of logistic DCMDPs (see \Cref{appendix: termdp}). As such, logistic DCMDPs support reasoning about optimizing more general contextual behavior, including: those involving notions of trust (e.g., where users become more or less receptive to agent recommendations); situations where humans override an agent for short periods; and modeling the effects of user satisfaction, moods, etc.


\section{Discussion and Future Work}

In this work we presented DCMDPs, and logistic DCMDPs in particular---a general history-dependent contextual framework which admits sample and computationally efficient solutions. The aggregation structure of logistic DCMDPs gives rise to efficient estimation of the unknown feature maps. We provided regret guarantees and developed a tractable realization of LDC-UCB using a computational estimator and a novel planning procedure. Finally, we tested DCZero, a model-based implementation of LDC-UCB, demonstrating its efficacy on a recommendation benchmark.

While logistic DCMDPs assume linear aggregations of past features, other variants with more complex parametric function classes over history are possible. Nevertheless, such complex function classes often require sample-inefficient techniques, suggesting that logistic DCMDPs may be especially well-suited to capturing extended, long history dependence.
In particular, they admit sample and computationally efficient solutions, which can be implemented in practice. As future work, a hybrid approach which considers combining dense models (such as Transformers) for short-history dependence, and aggregated models (such as logistic DCMDPs) for very long history dependence, may offer the ``best of both worlds" in practice. 


\section*{Acknowledgements}
\begin{wrapfigure}{r}{0.1\linewidth}
\vspace{-.25cm}
\hspace{-.35cm}
\includegraphics[width=1.2\linewidth]{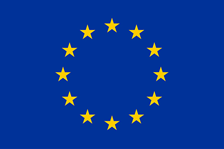}
\vspace{-.2cm}
\end{wrapfigure}
This project has received funding from the European Union’s Horizon 2020 research and innovation programme under the Marie Skłodowska-Curie grant agreement No 101034255. 

Nadav Merlis is partially supported by the Viterbi Fellowship, Technion.

\bibliography{bibliography}
\bibliographystyle{plainnat}

\clearpage
\onecolumn


\appendix
\addcontentsline{toc}{section}{Appendix}
\part{Appendix}
\parttoc


{\hypersetup{linkcolor=black}
\parttoc}


The appendix is organized as follows.
First, i \Cref{appendix: notation}, we define additional notations that are used throughout this work. We then show that $\suff$ is indeed a sufficient statistic for calculating the optimal policy in a logistic DCMDP (\Cref{appendix: sufficient statistic}). Next, we provide regret guarantees for our algorithms for solving logistic DCMDPs: in \Cref{appendix: regret analysis ldc-ucb,appendix: regret analysis tractable algorithm}, we bound the regret of LDC-UCB (\Cref{thm: regret ldc-ucb}) and its tractable variant (\Cref{thm: regret tractable ldc-ucb}), respectively.
Finally, \Cref{appendix: useful lemmas,appendix: threshold optimistic planning,appendix: confidence sets} contain technical lemmas which are crucial for deriving the above regret guarantees.
Specifically, \Cref{appendix: useful lemmas} is dedicated for optimism lemmas and decomposing the regret for logisitic DCMDPs; \Cref{appendix: threshold optimistic planning} deals with the threshold optimistic planning required for the tractable version of LDC-UCB; and \Cref{appendix: confidence sets} provides confidence sets for the regularized log likelihood procedure, following \citet{abeille2021instance,amani2021ucb}.

\newpage
\section{Additional Notations}\label{appendix: notation}
In this section, we define additional notation that will be of use throughout the proof.
We work with the natural filtration 
\begin{align*}
    \F_k=\sigma\brk*{\brk[c]*{\hist_{H+1}^{k'}}_{k'\in[k]},s_1^{k+1}}
    =\sigma\brk*{\brk[c]*{(s_h^1,a_h^1,x_h^1,R_h^1)}_{h=1}^H,\dots,\brk[c]*{(s_h^k,a_h^k,x_h^1,R_h^k)}_{h=1}^H,s_1^{k+1}},
\end{align*}
and notice that the policy $\pi^k$, which might depend on $s_1^k$, is $\F_{k-1}$-measurable. 
For brevity, for any episode $k\in[K]$ and time step $h\in[H]$, we define the probability distribution over the contexts by $z_h^k = \bs{z}(\suff(\hist_h^k; \paramvectrue)$, with $z_{i,h}^k = z_i(\suff(\hist_h^k; \paramvectrue)$ for any $i\in\X$.

With slight abuse of notation, we treat the latent features $\paramvec$ as vectors in $\paramset\subseteq\R^{\s\times\A\times[H]\times\X\times[M]}$ instead of a mapping $\paramvectrue_h: \s \times \A \times \X \mapsto \R^M, \forall h\in[H]$ and use the notations $\bs{f}_i(s,a,h,x)= \bs{f}(s,a,x,h,i)$. We also let $\bs{d}_h^k$ be the empirical discounted visitation vector at episode $k$ up to time step $h$, i.e.,
\begin{align*}
    \bs{d}_h^k(s,a,t,x) = \Halphainvhalf\alpha^{h-t-1}\indicator{s_t^k=s,a_t^k=a,x_t^k=x},
\end{align*}
where $\Halpha\triangleq\frac{1-\alpha^{2H}}{1-\alpha}\leq\min\brk[c]*{H,\frac{1}{1-\alpha}}$ is a normalization factor describing the effective historical horizon.
Then, one can write $\suff_i(\hist_h^k; \paramvec) = \brk[a]{\bs{f}_i,\bs{d}_h^k(s,a,t,x)}$. Notice that $\bs{d}_h^k$ is a vector containing zeros except for $h$ elements with the values $\brk[c]{\Halphainvhalf,\Halphainvhalf\alpha,\hdots,\Halphainvhalf\alpha^{h-1}}$, where each value appears exactly once. We denote the set of all possible vectors of such form for any $h\in[H]$ by $\D$, and notice that for all $\bs{d}\in\D$,
\begin{align*}
    \norm{\bs{d}}_2^2 \le \sum_{h=1}^{H-1}\Halphainv\alpha^{2h} = \frac{1-\alpha^{2H}}{1-\alpha}\Halphainv = 1.
\end{align*}

Next, we define the following summation operators:
\begin{itemize}
    \item For any fixed $h\in[H]$ and $i\in\X$, if $P_{i,h}:\s\times\A\mapsto\Delta_{\s}$ is a transition kernel and $V:\s\times\mathcal{H}_h\mapsto\R$ is a value function, the expected value is denoted by 
    \begin{align*}
    [P_{i,h} V] (s,a,\hist_h) 
    = P_{i,h}(\cdot\vert s,a)^T V(\cdot,\hist_h) 
    = \sum_{s' \in \s} P_i(s_{h+1} = s' | s_h=s, a_h=a)V(s', \hist_h).
    \end{align*}
    and, in general, use similar notations for any transition kernel $P:\mathcal{Y}\mapsto\Delta_{\s}$ from arbitrary space $\mathcal{Y}$.
    \item We denote the vectorized version of $P_{i,h} V$ by 
    \begin{align*}
        P_h V &= 
    \begin{pmatrix}
    P_{1,h} V, \hdots, P_{M,h} V
    \end{pmatrix}^T.
    \end{align*}
    \item If $Z:\mathcal{Y}\mapsto \Delta_{\X}$ is a mapping to the probability simplex over $\X$ and $U:\mathcal{Z}\times\X\mapsto \R$, where $\mathcal{Y},\mathcal{W}$ are some arbitrary spaces, we let 
    \begin{align*}
        [ZU](y,w) = \sum_{i=1}^{M+1} Z(y,i)U(w,i),
    \end{align*}
    and in particular, we use $Z_h^{\paramvec}(\hist_h)=\bs{z}(\suff(\hist_h; \paramvec)$ and $Z_h^k=\bs{z}_h^k$.
    \item Finally, given a transition kernel $P$ and latent feature $\paramvec$, we denote the transition operator over a value $V$ by
    \begin{align*}
    [T_h^{P,\paramvec}V] (s,a,\tau_h)
    =
    [Z^{\paramvec}_hP_hV] (s,a,\tau_h)
    =
    \sum_{s' \in \s} \sum_{i=1}^M P_i(s'| s, a)z_i(\suff(\hist_h; \paramvec)V(s', \hist_h).
\end{align*}
we similarly use the notation $T_h^{P,Z}$ when for general context distribution that are not necessarily by latent features $\paramvec$.
\end{itemize}

\clearpage

\section{Logistic DCMDPs}

\subsection{Sufficient Statistic}
\label{appendix: sufficient statistic}

We prove that the $\suff$ is a sufficient statistic for calculating the optimal policy. We begin by defining an augmented MDP $\M_{\text{aug}} = \brk*{\sset_{\text{aug}}, \aset_{\text{aug}}, P_{\text{aug}}, r_{\text{aug}}, H}$, where $\sset_{\text{aug}} = \sset \times \suffset(\paramvectrue)$ is the augmented state space, and $\aset_{\text{aug}} = \aset$ is the (unchanged) action space. The augmented transition function is defined for $s, \suff \in \sset \times \suffset(\paramvectrue), a \in \aset, s', \suff' \in \sset \times \suffset(\paramvectrue)$
\begin{align*}
    P_{\text{aug}}(s', \suff' | s, \suff, a)
    =
    \indicator{\suff' = \alpha\suff + \paramvectrue(s,a,x)} 
    \sum_{i=1}^{M+1}z_i(\suff)P_i(s'|s,a).
\end{align*}
Finally, the augmented reward function $r_{\text{aug}}$ satisfies
\begin{align*}
    r_{\text{aug}}(s, \suff, a) = \sum_{i=1}^{M+1}z_i(\suff)r_{i}(s,a).
\end{align*}
The augmented MDP $\M_{\text{aug}}$ is closely related to the logistic DCMDP $\brk*{\X, \s, \A, r, P, H, \paramvectrue, \alpha}$. In fact, as we will show next, they both achieve the same optimal value. To see this, consider an MDP defined by the tuple $(\sset_1 \times \sset_2, \A, P, r, H)$, and let $\phi: \sset_2 \mapsto D$, where $D$ is some known domain. Define the following set of deterministic policies
\begin{align*}
    \Pi_{\text{aug}} = \brk[c]*{\pi: \sset_1 \times \sset_2 \mapsto \A :  \exists \eta:\sset_1 \times D \mapsto [0,1], \pi(s_1, s_2) = \eta(s_1, \phi(s_2) }.
\end{align*}
Define the augmented optimal value for some $s \in \sset_1 \times \sset_2$
\begin{align*}
    V^*_{\text{aug},1}(s_1, s_2) 
    = 
    \max_{\pi \in \Pi_{\text{aug}}}
    \E{\sum_{t=1}^H r_t(s_t, a_t) | s_1=s_1, s_2 = s_2, a_t \sim \pi_t(s_1, s_2)}.
\end{align*}

We apply the following proposition using the decomposition $\sset_1 = \s$, and $\sset_2=\mathcal{H}$ as the set of possible trajectories in the known logistic DCMDP, where
\begin{align*}
    \phi(\tau_h) 
    := 
    \suff(\hist_h; \paramvec) =
    \sum_{t=0}^{h-1}
    \alpha^{h-t-1}\paramvec_t(s_t, a_t, x_t).
\end{align*}

\begin{proposition}[\citet{tennenholtz2022reinforcement}]
    Let $\M = (\sset_1 \times \sset_2, \A, P, r, H)$. Assume for any $s_1, s_2 \in \sset_1 \times \sset_2$, $a \in \A$, $P(s'_1, \phi(s'_2) | s_1, s_2, a) = P(s'_1, \phi(s'_2) | s_1, \phi(s_2), a)$ and $r(s_1, s_2, a) = g(s_1, a)$, for some deterministic function ${g:\sset_1 \times \A \mapsto [0,1]}$. Then, for any $s_1, s_2 \in \sset_1 \times \sset_2$,
    \begin{align*}
        V^*_{\text{aug},1}(s_1, s_2) = V^*_1(s_1, s_2).
    \end{align*}
\end{proposition}

This concludes our claim, proving that $\sigma$ is indeed sufficient, as playing any policy in $\Pi_{\text{aug}}$ achieves the same value.

\subsection{Relation to TerMDPs}
\label{appendix: termdp}

A special case of logistic DCMDPs are TerMDPs \citep{tennenholtz2022reinforcement}, which model exogenous, non-Markov termination in the environment. When terminated, the agent stops interacting with the environment and cannot collect additional rewards. This setup describes various real-world scenarios, such as passengers in autonomous vehicles or users abandoning a recommender systems. To model a TerMDP as a logistic DCMDP we let $\X = \brk[c]*{0, 1} = \brk[c]*{\text{term}, \text{no term}}$, and define $r_i(s, a), P_i(s'|s,a)$ such that $s_{\text{term}}$ is a sink state for which $r_i(s_{\text{term}}, a) = 0$. The reward in all other states is defined by $r_1(s, a)$. The transition probabilities are defined by $
P_i(s'|s,a)
=
\begin{cases}
    s_{\text{term}}, &s=s_{\text{term}} \lor i=0 \\
    P_1(s'|s,a), &\text{o.w.}
\end{cases}
$.
TerMDPs use a cost functions $c_h(s_t, a_t)$ to define the probability of transitioning to the termination state, as
${
P(x_h = \text{term} | \hist_h)
=
z_0\brk*{\sum_{t=1}^{h-1} c_t(s_t, a_t)}}
$
-- a special case of logistic DCMDPs with a two-class, context-independent feature map, and a choice of $\alpha = 1$. Indeed, this choice of parameters defines a TerMDP as proposed in \citet{tennenholtz2022reinforcement}.

Logistic DCMDPs let us consider generalized notions of such models, for which classes can reflect notions of trust, where humans become less susceptible to following recommendations from an agent, through situations where humans override an agent for short periods, to modeling the effects of changing moods.

\clearpage
\section{Regret Analysis of LDC-UCB}
\label{appendix: regret analysis ldc-ucb}


In this section, we prove the regret bounds of \Cref{thm: regret ldc-ucb}. We start by defining the good event, which holds uniformly for all episode with probability $1-\delta$. Then, we show that LDC-UCB is optimistic under the good event. Next, we decompose the regret to error terms of the reward, transition and latent features, and analyzing each of these terms result with the desired regret bounds.

We start by stating the bonuses which the algorithm uses:
\begin{align*}
    & b_{x,h}^{r,k}(s,a) = \min\brk[c]*{\sqrt{\frac{\log \frac{8SAMHK}{\delta}}{n_h^k(s,a,x)\vee 1}},1} \\
    & b_{x,h}^{p,k}(s,a) = \min\brk[c]*{H\sqrt{\frac{4S\log\frac{8SAMHK}{\delta}}{n_h^k(s,a,x)\vee 1
    }},2H}
\end{align*}
\subsection{Failure Events}

We define the following failure events.
\begin{align*}
    &F^r_k=\brk[c]*{\exists s\in \s,a \in \A,x\in\X,h \in [H]:\ |r_{x,h}(s,a) - \hat{r}_{x,h}^k(s,a)| > \min\brk[c]*{\sqrt{\frac{\log \frac{2SAMHK}{\delta'}}{n_h^k(s,a,x)\vee 1}},1} }\\
    &F^p_k=\brk[c]*{\exists s\in \s,a \in \A,x\in\X,h \in [H]:\ \norm{P_{x,h}(\cdot\mid s,a)- \hat{P}_{x,h}^k(\cdot\mid s,a)}_1 > \min\brk[c]*{\sqrt{\frac{4S\log\frac{2SAMHK}{\delta'}}{n_h^k(s,a,x)\vee 1
    }},2}}\\
    &F^n = \brk[c]*{\sum_{k=1}^K \sum_{h=1}^H \sum_{i=1}^{M+1}
    \E{ \frac{z_{i,h}^k}{\sqrt{n_h^k(s_h^k,a_h^k,i)\vee 1}} | F_{k-1}} > 18H^2\log\brk*{\frac{1}{\delta'}}+2HS(M+1)A +4\sqrt{H^2S(M+1)AK}}\\
    & F^{\paramvec, \text{global}}_k =  \brk[c]*{\paramvectrue \notin \C_k(\delta')},
\end{align*}
where the definition of $\C_k(\delta)$ can be found at \Cref{appendix: confidence sets}. 

Then, we define the good event, where none of the aforementioned failure events ever occur, i.e.,
\begin{align*}
    \G=\brk*{\cup_{k\in[K]}\bar{F}^r_k}\cap\brk*{\cup_{k\in[K]}\bar{F}^p_k} \cap \brk*{\cup_{k\in[K]}\bar{F}^{\paramvec, \text{global}}_k} \cap \bar{F}^n  
\end{align*}
\begin{lemma}
\label{lemma: lcb-ucb good event}
Letting $\delta'=\delta/4$, the event $\G$ holds with probability at least $1-\delta$. 
\end{lemma}
\begin{proof}
We show that the probability that the events do not hold for all $k\in[K]$ is smaller than $\delta'=\delta/4$. 
\begin{itemize}
    \item \textbf{Reward concentration.} First observe that both the empirical and real rewards are bounded in $[0,1]$, so if the minimizer in $F_k^r$ is $1$, the event never holds. Otherwise, for any fixed episode $k$, number of plays $n$, state $s$, action $a$, context $x$ and timestep $h$, by Hoeffding's inequality, the estimation error is bounded w.p. $1-\delta'$ by $\sqrt{\frac{\log\frac{2}{\delta'}}{n}}$. Taking the union bound over all possible values of $k,n\ge1,s,a,x$ and $h$, w.p. at least $1-\delta'$, for all $k\in[K], s\in\s,a\in\A,x\in\X,h\in[H]$, the estimation error is bounded by 
    \begin{align*}
        |r_{x,h}(s,a) - \hat{r}_{x,h}^k(s,a)| 
        \leq \sqrt{\frac{\log \frac{2SAMHK^2}{\delta'}}{2n_h^k(s,a,x)\vee 1}} 
        \leq \sqrt{\frac{\log \frac{2SAMHK}{\delta'}}{n_h^k(s,a,x)\vee 1}} .
    \end{align*}
    Finally, we remark that since $\delta'\le1/4$, the event $F_k^r$ never holds when $n_h^k(s,a,x)=0$ since the bound is larger than $1$. 
    
    In other words, $\Pr\brk[c]*{\cup_{k\in[K]}\bar{F}^r_k}\leq \delta'$.
    \item \textbf{Transition concentration.} By the exact same arguments as the reward concentration, while replacing Hoeffding's inequality by the concentration of the $L_1$ error of a probability estimator \citep{weissman2003inequalities}, we also get $\Pr\brk[c]*{\cup_{k\in[K]}\bar{F}^p_k}\leq \delta'$.  Notice that the $L_1$ distance between any two probability distributions is bounded by $2$, which justifies the minimization in the event.
    \item \textbf{Global feature estimation.} By \Cref{lemma: mnl confidence set}, we have that 
    \begin{align*}
        \Pr\brk[c]*{\cup_{k\in[K]}\bar{F}^{\paramvec, \text{global}}_k}
        \leq \Pr\brk[c]*{\exists k\ge1 : \paramvectrue \notin \C_k(\delta')}
        \leq\delta'.
    \end{align*}
    \item \textbf{Expected counts concentration.} By \Cref{lemma: expected cumulative visitation bound}, we have that $\Pr\brk[c]*{\bar{F}^n}\leq \delta'$.
\end{itemize}
Fixing $\delta'=\delta/4$ and taking the union bound concludes the proof.
\end{proof}

\subsection{Regret Analysis -- Proof of \Cref{thm: regret ldc-ucb}}

\ldcucb*

\begin{proof}
Under the good event, the conditions of the regret decomposition lemma (\Cref{lemma: regret decomposition}) hold with $\bar{r},\hat{P},Z^{\bar \paramvec_k}$ and $c=4$ due to the truncated value iteration, truncated bonuses and value optimism lemma (\Cref{prop: optimism ldc-ucb}). Therefore, the regret can be decomposed in the following way
\begin{align*}
    \Reg{K} 
    &\leq \underbrace{\sum_{k=1}^K \sum_{h=1}^H\sum_{i=1}^{M+1}\E{ z_{i,h}^k\abs{\bar r_{i,h}^k(s_h^k, a_h^k) - r_{i,h}(s_h^k, a_h^k)} | F_{k-1}}}_{(i)} \\
    &\quad+ \underbrace{H\sum_{k=1}^K \sum_{h=1}^H\sum_{i=1}^{M+1}\E{ z_{i,h}^k\norm{\brk*{\hat P_h^k - P_h}(\cdot|s_h^k, a_h^k)}_1 | F_{k-1}}}_{(ii)} \\
    &\quad+ \underbrace{5H\sum_{k=1}^K\sum_{h=1}^H\E{\norm{Z^{\bar \paramvec_k}_h - Z^{\paramvectrue}_h}_1 | F_{k-1}}}_{(iii)}.
\end{align*}
By plugging in \Cref{lemma: reward concentration LDC-UCB}, \Cref{lemma: transition concentration LDC-UCB} and \Cref{lemma: latent feature concentration LDC-UCB}, which bound terms (i),(ii) and (iii), respectively, we get,
\begin{align*}
    \Reg{K} 
    &\leq \Ob\brk*{H^2S\sqrt{MAK\log\frac{SAMHK}{\delta}}}\\
    & + \Ob\brk*{H^2S\sqrt{MAK\log\frac{SAMHK}{\delta}}}\\
    & + \tilde{\Ob}\brk{\sqrt{S^2A^2H^6 M^{4.5}L^2\kappamin K}}
\end{align*}
Noticing that the last term is the dominant, we get 
\begin{align*}
     \Reg{K} \leq \tilde{\Ob}\brk{\sqrt{S^2A^2H^6 M^{4.5}L^2\kappamin K}},
\end{align*}
which concludes the proof.
\end{proof}

Now, we prove the lemmas that bounds the three terms in the regret decomposition in \Cref{thm: regret ldc-ucb}.

\begin{lemma}[Reward Concentration]\label{lemma: reward concentration LDC-UCB}
Under the good event, we have that:
\begin{align*}
    (i) &= \sum_{k=1}^K \sum_{h=1}^H\sum_{i=1}^{M+1}\E{ z_{i,h}^k\abs{\bar r_{i,h}^k(s_h^k, a_h^k) - r_{i,h}(s_h^k, a_h^k)} | F_{k-1}} \\
    &\leq 2H\sqrt{S\log\frac{8SAMHK}{\delta}}\cdot\brk*{18H^2\log\brk*{\frac{4}{\delta}}+2HS(M+1)A +4\sqrt{H^2S(M+1)AK}}\tag{Under $\G$} \\
    & = \Ob\brk*{H^2S\sqrt{MAK\log\frac{SAMHK}{\delta}}}
\end{align*}
\end{lemma}
\vspace{-.3cm}
\begin{proof}
\begin{align*}
    \sum_{k=1}^K \sum_{h=1}^H\sum_{i=1}^{M+1}&\E{ z_{i,h}^k\abs{\bar r_{i,h}^k(s_h^k, a_h^k) - r_{i,h}(s_h^k, a_h^k)} | F_{k-1}} \\
    &\leq  \sum_{k=1}^K \sum_{h=1}^H\sum_{i=1}^{M+1}\E{ z_{i,h}^k\brk*{b_{i,h}^{r,k}(s_h^k, a_h^k) + b_{i,h}^{p,k}(s_h^k, a_h^k) + \abs{\hat r_{i,h}^k(s_h^k, a_h^k) - r_{i,h}(s_h^k, a_h^k)}} | F_{k-1}} \\
    & \leq\sum_{k=1}^K \sum_{h=1}^H\sum_{i=1}^{M+1}\E{ z_{i,h}^k\brk*{2b_{i,h}^{r,k}(s_h^k, a_h^k) + b_{i,h}^{p,k}(s_h^k, a_h^k)} | F_{k-1}} \tag{Under $\G$} \\
    & \leq 2H\sqrt{S\log\frac{8SAMHK}{\delta}}\sum_{k=1}^K \sum_{h=1}^H\sum_{i=1}^{M+1}\E{ \frac{z_{i,h}^k}{\sqrt{n_h^k(s_h^k,a_h^k,i)\vee 1}
    } | F_{k-1}} \\
    & \leq 2H\sqrt{S\log\frac{8SAMHK}{\delta}}\cdot\brk*{18H^2\log\brk*{\frac{4}{\delta}}+2HS(M+1)A +4\sqrt{H^2S(M+1)AK}}\tag{Under $\G$} \\
    & = \Ob\brk*{H^2S\sqrt{MAK\log\frac{SAMHK}{\delta}}}
\end{align*}
\end{proof}

\begin{lemma}[Transition Concentration]\label{lemma: transition concentration LDC-UCB}
Under the good event, we have that:
\begin{align*}
    (ii) &= H\sum_{k=1}^K \sum_{h=1}^H\sum_{i=1}^{M+1}\E{ z_{i,h}^k\norm{\brk*{\hat P_h^k - P_h}(\cdot|s_h^k, a_h^k)}_1 | F_{k-1}} \\
    & \leq \Ob\brk*{H^2S\sqrt{MAK\log\frac{SAMHK}{\delta}}}
\end{align*}
\end{lemma}
\vspace{-.3cm}
\begin{proof}
\begin{align*}
    H\sum_{k=1}^K \sum_{h=1}^H\sum_{i=1}^{M+1}&\E{ z_{i,h}^k\norm{\brk*{\hat P_h^k - P_h}(\cdot|s_h^k, a_h^k)}_1 | F_{k-1}} \\
    &\leq
    H\sum_{k=1}^K \sum_{h=1}^H \sum_{i=1}^M
    \E{z_{i,h}^k \cdot \frac{1}{H}b_{i,h}^{p,k}(s_h^k, a_h^k) | F_{k-1}} \tag{Under $\G$}\\
    &\leq
    H\sqrt{4S\log\frac{8SAMHK}{\delta}}\sum_{k=1}^K \sum_{h=1}^H \sum_{i=1}^M
    \E{\frac{z_{i,h}^k}{ \sqrt{n_h^k(s_h^k,a_h^k,i)\vee 1}} | F_{k-1}} \\
    &\leq H\sqrt{4S\log\frac{8SAMHK}{\delta}}\cdot\brk*{18H^2\log\brk*{\frac{4}{\delta}}+2HS(M+1)A +4\sqrt{H^2S(M+1)AK}} \tag{Under $\G$} \\
    & = \Ob\brk*{H^2S\sqrt{MAK\log\frac{SAMHK}{\delta}}}
\end{align*}
\end{proof}

\begin{lemma}[Latent Features Concentration]\label{lemma: latent feature concentration LDC-UCB}
Under the good event, if $L=\Omega(1)$ and $\lambda = \Theta(\frac{SAHM^{2.5}}{L})$, we have that:
\begin{align*}
    (iii) &= 5H\sum_{k=1}^K\sum_{h=1}^H\E{\norm{Z^{\bar \paramvec_k}_h - Z^{\paramvectrue}_h}_1 | F_{k-1}} \\
    & \leq  \tilde{\Ob}\brk{\sqrt{S^2A^2H^6 M^{4.5}L^2\kappamin K}}
\end{align*}
\end{lemma}
\begin{proof}
\begin{align*}
    5H\sum_{k=1}^K\sum_{h=1}^H&\E{\norm{Z^{\bar \paramvec_k}_h - Z^{\paramvectrue}_h}_1 | F_{k-1}} \\
    &\leq 5H\sqrt{M+1}\sum_{k=1}^K\sum_{h=1}^H\E{\norm{Z^{\bar \paramvec_k}_h - Z^{\paramvectrue}_h}_2 | F_{k-1}} \\
    &\leq
    10H\sqrt{(1+2L)(M+1)\kappamin}
    \sum_{k=1}^K \sum_{h=1}^H
    \E{\beta_k(4\delta)
    \norm{\bs{d}_h^k }_{\bs{V}_k^{-1}} | F_{k-1}} \tag{\Cref{lemma: multinomial error norm}}\\
    &\leq
    10H\beta_K(4\delta)\sqrt{(1+2L)(M+1)\kappamin}
    \sum_{k=1}^K \sum_{h=1}^H
    \E{\norm{\bs{d}_h^k }_{\bs{V}_k^{-1}} | F_{k-1}}  \\
    &\leq
    10H\beta_K(4\delta)\sqrt{(1+2L)(M+1)\kappamin}
    \frac{\sqrt{2KH^3MSA\log\frac{\kappamin\lambda HSMA+k}{\kappamin\lambda HSMA}}}{\max\brk[c]{1,1/\sqrt{\lambda}}} \tag{\Cref{corollary: elliptical potential lemma}} \\
    & = \Ob\brk*{\frac{\beta_K(4\delta)}{\max\brk[c]{1,1/\sqrt{\lambda}}}\sqrt{H^5M^2SAL\kappamin K\log\frac{\kappamin\lambda HSMA+k}{\kappamin\lambda HSMA}}}\\
    & = \tilde{\Ob}\brk{\sqrt{S^2A^2H^6 M^{4.5}L^2\kappamin K}}
\end{align*}
For both the lemma and the corollary, we remind that $\bs{d}_h^k\leq 1$.
For the last relation, recall that 
$$\beta_k(\delta) = \frac{M^{3/2}(M+1)SAH}{\sqrt{\lambda}}\brk*{\log\brk*{1 + \frac{k}{(M+1)SA\lambda}} + 2\log\brk*{\frac{2}{\delta}}} + \sqrt{\frac{\lambda}{4M}} + \sqrt{\lambda}L,$$
and assuming that $L=\Omega(1)$, we take $\lambda = \Theta(\frac{SAHM^{2.5}}{L})$, so $\beta_K(4\delta)=\tilde{\Ob}\brk{\sqrt{SAHM^{2.5}L}}$. 
\end{proof}

\subsection{Optimism in Logistic DCMDPs}

In this section, we prove \Cref{prop: optimism ldc-ucb}, which allows us to apply the regret decomposition (\Cref{lemma: regret decomposition}) necessary for proving \Cref{thm: regret ldc-ucb}.

We start by clearly stating the output value of the planning algorithm. For any $\paramvec\in\paramset$, we define the truncated optimistic value under $\paramvec$ as the solution to the following value iteration problem:
\begin{align*}
    &\bar{V}_{H+1}^{k,\paramvec}(s,\hist_{H+1})=0, &\forall s\in\s,\hist_{H+1}\in\mathcal{H}_{H+1}\\
    &\bar{V}_h^{k,\paramvec}(s,\hist_h)=\min\brk[c]*{H, \max_a  \brk[c]*{ \brk[s]*{Z^{\paramvec}_h \bar r_h^k}(s, a, \hist_h) + \brk[s]*{T_h^{\hat P_h^k, \paramvec}\bar{V}_{h+1}^{k,\paramvec}}(s,a, \hist_h)} }, &\forall h\in[H],s\in\s,\hist_h\in\mathcal{H}_h.
\end{align*}
Then, given an initial state $s$, we define the optimistic value of the DCMDP by $\bar{V}_1^{k}(s,\hist_h) = \max_{\paramvec\in\C_k(\delta)}\bar{V}_1^{k,\paramvec}(s,\hist_h)$. For this value, the following holds:
\begin{proposition}
\label{prop: optimism ldc-ucb}
Under the good event $\G$, for any $k\ge1$ and any initial state $s\in\s$, it holds that $\bar{V}_1^k(s)\ge V_1^*(s)$.
\end{proposition}
\begin{proof}
Assume that $\G$ holds, 
and let $\bar{V}_h^{k,\paramvec}(s,\hist_h)$, $\bar{V}_1^{k}(s,\hist_h)$ as defined by the beginning of the section. 
In the proof, we will show that for any $k\in[K]$, $s\in\s$, $h\in[H]$ and $\hist_h\in\mathcal{H}_h$ and $\paramvec\in\paramset$, it holds that $\bar{V}_h^{k,\paramvec}(s,\hist_h)\ge V^{*,\paramvec}(s,\hist_h)$. Since under $\G$, we know that $\paramvec\in\C_k(\delta)$, we then have that
\begin{align*}
    \bar{V}_1^k(s)
    =\max_{\paramvec\in\C_k(\delta)}\bar{V}_1^{k,\paramvec}(s,\hist_h)
    \ge \max_{\paramvec\in\C_k(\delta)}V^{*,\paramvec}_1(s,\hist_h)
    \ge V_1^*(s),
\end{align*}
which would conclude the prove. Throughout this proof, we assume w.l.o.g. that all optimistic values are smaller than $H$; otherwise, they will be truncated to $H$, which still always optimistic since the rewards are in $[0,1]$ and the horizon is $H$.

We prove that $\bar{V}_h^{k,\paramvec}(s,\hist_h)\ge V^{*,\paramvec}(s,\hist_h)$ by backward-induction. First notice that the claim holds for $h=H$, since
\begin{align*}
    \bar{V}_H^{k,\paramvec}(s,\hist_H;\paramvec) - V_H^{*,\paramvec}\brk*{s, \hist_H}
    &= 
    \max_a  \brk[c]*{ \brk[s]*{Z^{\paramvec}_H \bar r_H^k}(s, a, \hist_H)} 
    - \max_a \brk[c]*{\brk[s]*{Z^{\paramvec}_H r_H}(s, a, \hist_H)}\\
    & \overset{(1)}\geq 
    \brk[s]*{Z^{\paramvec}_H \bar r_H^k}(s, a^*, \hist_H)
    -
    \brk[s]*{Z^{\paramvec}_H r_H}(s, a^*, \hist_H) \tag{for $a^*\in\arg\max_a \brk[c]*{\brk[s]*{Z^{\paramvec}_H r_H}(s, a, \hist_H)}$}\\
    &\ge \brk[s]*{Z^{\paramvec}_H (\hat r_H^{k} + b_H^{r,k} - r)}(s,a^*,\hist_H) \\
    &\geq 0 \tag{\Cref{lemma: reward optimism}}
\end{align*}

Now let $h\in[H-1]$ and assume that the claim holds for $h+1$. Then, for 
$$a^*\in\arg\max_a \brk[c]*{\brk[s]*{Z^{\paramvec}_h r_h}(s, a, \hist_h) + \brk[s]*{T_hV_{h+1}^{*,\paramvec}}(s,a, \hist_h)},$$
we have

\begin{align*}
    \bar V_h^{k,\paramvec}&\brk*{s, \hist_h} - V_h^{*,\paramvec}\brk*{s, \hist_h} \\
    &= 
    \max_a  \brk[c]*{ \brk[s]*{Z^{\paramvec}_h \bar r_h^k}(s, a, \hist_h) + \brk[s]*{T_h^{\hat P_h^k, \paramvec}\bar{V}_{h+1}^{k,\paramvec}}(s,a, \hist_h)} 
    - \max_a \brk[c]*{\brk[s]*{Z^{\paramvec}_h r_h}(s, a, \hist_h) + \brk[s]*{T_hV_{h+1}^{*,\paramvec}}(s,a, \hist_h)}\\
    & \overset{(1)}\geq 
    \brk[s]*{Z^{\paramvec}_h \bar r_h^k}(s, a^*, \hist_h) + \brk[s]*{T_h^{\hat P_h^k, \paramvec}\bar{V}_{h+1}^{k,\paramvec}}(s,a^*, \hist_h)
    -
    \brk[s]*{Z^{\paramvec}_h r_h}(s, a^*, \hist_h) - \brk[s]*{T_hV_{h+1}^{*,\paramvec}}(s,a^*, \hist_h) \\
    &=
    \brk[s]*{Z^{\paramvec}_h \bar r_h^k}(s, a^*, \hist_h) - \brk[s]*{Z^{\paramvec}_h r_h}(s, a^*, \hist_h)
     +
    \brk[s]*{\brk*{T_h^{\hat P_h^k, \paramvec} - T_h}\bar{V}_{h+1}^{k,\paramvec}}(s,a^*, \hist_h) 
    +
    \brk[s]*{T_h\brk*{\bar{V}_{h+1}^{k,\paramvec} - V_{h+1}^{*,\paramvec}}}(s,a^*, \hist_h) \\
    & \overset{(2)}\geq 
     \brk[s]*{Z^{\paramvec}_h \bar r_h^k}(s, a^*, \hist_h) - \brk[s]*{Z^{\paramvec}_h r_h}(s, a^*, \hist_h)
     +
    \brk[s]*{\brk*{T_h^{\hat P_h^k, \paramvec} - T_h}\bar{V}_{h+1}^{k,\paramvec}}(s,a^*, \hist_h) \\
    & = \brk[s]*{Z^{\paramvec}_h \brk*{\bar r_h^k - r_h}}(s, a^*, \hist_h) + \brk[s]*{ Z^{\paramvec}_h \brk*{\hat P_h^k - P_h}\bar{V}_{h+1}^{k,\paramvec}}(s,a^*, \hist_h),
\end{align*}
where in $(1)$ we used the definition of the max operator, and in $(2)$ the induction step. 
Overall, replacing $\bar{r}_h^k$ with its definition, we get that
\begin{align*}
    \bar V_h^k\brk*{s, \hist_h} - V_h^*\brk*{s, \hist_h}
    &\geq
    \brk[s]*{Z^{\paramvec}_h \brk*{\hat r_h^k + b_h^{r,k} - r_h}}(s, a^*, \hist_h) +
    \brk[s]*{ Z^{\paramvec}_h \brk*{\brk*{\hat P_h^k - P_h}\bar{V}_{h+1}^{k,\paramvec} + b_h^{p,k}}}(s,a^*, \hist_h)\\
    &\ge 0,
\end{align*}
where the second inequality is by \Cref{lemma: reward optimism} and \Cref{lemma: transition optimism}, which hold under $\G$.

\end{proof}

\clearpage

\section{Regret Analysis for Tractable LDC-UCB}
\label{appendix: regret analysis tractable algorithm}

In this section, we prove the regret bounds of \Cref{thm: regret tractable ldc-ucb}. We start by defining the good event, which holds uniformly for all episode with probability $1-\delta$. Then, we show that the Tractable LDC-UCB is optimistic under the good event. Next, we decompose the regret to error terms of the reward, transition and latent features, and analyzing each of these terms result with the desired regret bounds.

We start by stating the bonuses which the algorithm uses:
\begin{align*}
    & b_{x,h}^{r,k}(s,a) = \min\brk[c]*{\sqrt{\frac{\log \frac{8SAMHK}{\delta}}{n_h^k(s,a,x)\vee 1}},1} \\
    & b_{x,h}^{p,k}(s,a) = \min\brk[c]*{H\sqrt{\frac{4S\log\frac{8SAMHK}{\delta}}{n_h^k(s,a,x)\vee 1
    }},2H} \\
    & b_{x,h}^{\paramvec,k}(s,a) = \frac{2\sqrt{\kappamin}\gamma_k(4\delta)}{\sqrt{n_h^k(s,a,x) + 4\lambda}}
\end{align*}
The confidence intervals can then be written as 
\begin{align*}
    \ci_h^k = \brk[s]*{\suff(\hist_h^k; \estparamvec^k) - \sum_{t=0}^{h-1}\alpha^{h-t-1}b_{x_t^k,t}^{\paramvec,k}(s_t^k,a_t^k),\quad \suff(\hist_h^k; \estparamvec^k) + \sum_{t=0}^{h-1}\alpha^{h-t-1}b_{x_t^k,t}^{\paramvec,k}(s_t^k,a_t^k)}
\end{align*}

\subsection{Failure Events}

We define the following failure events.
\begin{align*}
    &F^r_k=\brk[c]*{\exists s\in \s,a \in \A,x\in\X,h \in [H]:\ |r_{x,h}(s,a) - \hat{r}_{x,h}^k(s,a)| > \min\brk[c]*{\sqrt{\frac{\log \frac{2SAMHK}{\delta'}}{n_h^k(s,a,x)\vee 1}},1} }\\
    &F^p_k=\brk[c]*{\exists s\in \s,a \in \A,x\in\X,h \in [H]:\ \norm{P_{x,h}(\cdot\mid s,a)- \hat{P}_{x,h}^k(\cdot\mid s,a)}_1 > \min\brk[c]*{\sqrt{\frac{4S\log\frac{2SAMHK}{\delta'}}{n_h^k(s,a,x)\vee 1
    }},2}}\\
    &F^n = \brk[c]*{\sum_{k=1}^K \sum_{h=1}^H \sum_{i=1}^{M+1}
    \E{ \frac{z_{i,h}^k}{\sqrt{n_h^k(s_h^k,a_h^k,i)\vee 1}} | F_{k-1}} > 18H^2\log\brk*{\frac{1}{\delta'}}+2HS(M+1)A +4\sqrt{H^2S(M+1)AK}}\\
    & F^{\paramvec, \text{local}}_k =  \brk[c]*{\exists s\in \s,a \in \A, x \in \X, i\in [M],h \in [H]: \abs{\hat \paramfunc_{i,h}^k(s,a,x) - \paramfunctrue_{i,h}(s,a,x) } 
    >
    \frac{2\sqrt{\kappamin}\gamma_k(\delta')}{\sqrt{n_h^k(s,a,x) + 4\lambda}}}
\end{align*}
where $\gamma_k(\delta)$ is defined in \Cref{prop:convex relaxation}. 

Then, we define the good event, where none of the aforementioned failure events ever occur, i.e.,
\begin{align*}
    \G=\brk*{\cup_{k\in[K]}\bar{F}^r_k}\cap\brk*{\cup_{k\in[K]}\bar{F}^p_k} \cap \brk*{\cup_{k\in[K]}\bar{F}^{\paramvec, \text{local}}_k} \cap \bar{F}^n  
\end{align*}
\begin{lemma}
\label{lemma: tractable lcb-ucb good event}
Letting $\delta'=\delta/4$, the event $\G$ holds with probability at least $1-\delta$. 
\end{lemma}
\begin{proof}
The proof is almost identical to the one of \Cref{lemma: lcb-ucb good event}. We only need to prove that $\Pr\brk[c]*{\cup_{k\in[K]}\bar{F}^{\paramvec, \text{local}}_k} \leq \delta'$, which directly follows by \Cref{lemma: local feature confidence}.
\end{proof}


\subsection{Regret Analysis -- Proof of \Cref{thm: regret tractable ldc-ucb}}

\tractableLdcUcb*
\begin{proof}
Let
$$\bar\suff_h^k(s,\hist_h) \in \arg\max_{\bar \suff\in\ci_h^k} \max_a\brk[c]*{ \sum_{i=1}^{M+1}z_i(\bar \suff)\bar r_{i,h}^k(s, a) + \sum_{i=1}^{M+1}z_i(\bar \suff) \hat P_{i,h}^k(\cdot | s,a)^T\bar{V}_{h+1}^k(\cdot, \hist_{h+1})},$$
and denote $\bar{Z}_h^k(s,\hist_h) = \bs{z}(\bar\suff_h^k(s,\hist_h))$.

Under the good event, the conditions of the regret decomposition lemma (\Cref{lemma: regret decomposition}) hold with $\bar{r},\hat{P},\bar{Z}$ and $c=4$ due to the truncated value iteration, truncated bonuses and value optimism lemma (\Cref{prop: optimism ldc-ucb}). Therefore, the regret can be decomposed in the following way,

\begin{align*}
    \Reg{K} 
    &\leq \underbrace{\sum_{k=1}^K  \sum_{h=1}^H\sum_{i=1}^{M+1}\E{ z_{i,h}^k\abs{\bar r_{i,h}^k(s_h^k, a_h^k) - r_{i,h}(s_h^k, a_h^k)} | F_{k-1}}}_{(i)} \\
    &\quad+ \underbrace{H\sum_{k=1}^K \sum_{h=1}^H\sum_{i=1}^{M+1}\E{ z_{i,h}^k\norm{\brk*{\hat P_h^k - P_h}(\cdot|s_h^k, a_h^k)}_1 | F_{k-1}}}_{(ii)} \\
    &\quad+ \underbrace{5H\sum_{k=1}^K\sum_{h=1}^H\E{\norm{\bar{Z}^k_h - Z^{\paramvectrue}_h}_1 | F_{k-1}}}_{(iii)}.
\end{align*}
The terms $(i)$ and $(iii)$ are identical to the ones in the proof of \Cref{thm: regret ldc-ucb} (as the reward bonuses are identical), and thus can be bounded by \Cref{lemma: reward concentration LDC-UCB} and \Cref{lemma: transition concentration LDC-UCB}. Term $(iii)$ can be bounded by \Cref{lemma: latent feature concentration tractable LDC-UCB}.

Thus, we obtain,
\begin{align*}
    \Reg{K} &\leq \Ob\brk*{H^2S\sqrt{MAK\log\frac{SAMHK}{\delta}}}\\
    & + \Ob\brk*{H^2S\sqrt{MAK\log\frac{SAMHK}{\delta}}}\\
    & + \tilde \Ob\brk*{\sqrt{H^8 S^2 A^2 M^{6.5} L^4 \kappamin K}} \\
    & \leq \tilde\Ob\brk*{\sqrt{H^8 S^2 A^2 M^{6.5} L^4 \kappamin K}}
\end{align*}
\end{proof}

\begin{lemma}[Latent Features Concentration]\label{lemma: latent feature concentration tractable LDC-UCB}
Under the good event, it holds that
\begin{align*}
    5H\sum_{k=1}^K\sum_{h=1}^H\E{\norm{Z^{\bar \paramvec_k}_h - Z^{\paramvectrue}_h}_1 | F_{k-1}}& \leq
    \Ob\brk*{\brk*{\frac{1}{\sqrt{\lambda}}\vee1}\gamma_K(4\delta)\sqrt{\kappamin H^6SM^3AK}} \\
    &\leq \tilde \Ob\brk*{\sqrt{H^8 S^2 A^2 M^{6.5} L^4 \kappamin K}}
\end{align*}
\end{lemma}
\begin{proof}
\begin{align*}
    5H\sum_{k=1}^K\sum_{h=1}^H&\E{\norm{Z^{\bar \paramvec_k}_h - Z^{\paramvectrue}_h}_1 | F_{k-1}} \\
    &= 5H\sum_{k=1}^K\sum_{h=1}^H\E{\norm{\bs{z}(\bar\suff_h^k(s,\hist_h^k)) - \bs{z}(\suff(\hist_h^k; \paramvectrue)}_1 | F_{k-1}}\\
    & \overset{(1)}\leq 10H\sum_{k=1}^K\sum_{h=1}^H\sum_{i=1}^M\E{\abs{z_i(\bar\suff_h^k(s,\hist_h^k)) - z_i(\suff(\hist_h^k; \paramvectrue)} | F_{k-1}}\\
    &\overset{(2)}{\le} 2.5H\sum_{k=1}^K\sum_{h=1}^H\sum_{i=1}^M\E{\abs{\bar\suff_{i,h}^k(s,\hist_h^k) - \suff_i(\hist_h^k; \paramvectrue)} | F_{k-1}}
\end{align*}
where relation $(1)$ holds by substituting $z_{M+1}(\bs{x}) = 1-\sum_{i=1}^Mz_i(\bs{x})$ and applying the triangle inequality, and relation $(2)$ is since the function $f(x)=e^x/(a+e^x)$ is $\frac{1}{4}$-Lipschitz, and $z_i(\bs{x})$ can be represented as such a function of $x_i$.
Then,
\begin{align*}
    5H\sum_{k=1}^K\sum_{h=1}^H&\E{\norm{Z^{\bar \paramvec_k}_h - Z^{\paramvectrue}_h}_1 | F_{k-1}} \\
    &\leq
    2.5H\sum_{k=1}^K\sum_{h=1}^H\sum_{i=1}^M\E{\abs{\bar\suff_{i,h}^k(s,\hist_h^k) - \suff_i(\hist_h^k; \paramvectrue} | F_{k-1}} \\
    &\le 2.5H\sum_{k=1}^K\sum_{h=1}^H\sum_{i=1}^M\E{\abs{\suff_i(\hist_h; \estparamvec^k) - \suff_i(\hist_h^k; \paramvectrue} | F_{k-1}} \\
    &\quad+ 2.5H\sum_{k=1}^K\sum_{h=1}^H\sum_{i=1}^M\E{\sum_{t=0}^{h-1}\alpha^{h-t-1}b_{x_t^k,t}^{\paramvec,k}(s_t^k,a_t^k) | F_{k-1}}\\ 
    &\le 2.5H\sum_{k=1}^K\sum_{h=1}^H\sum_{i=1}^M\E{\sum_{t=0}^{h-1}\alpha^{h-t-1}b_{x_t^k,t}^{\paramvec,k}(s_t^k,a_t^k) | F_{k-1}} \\
    &\quad+ 2.5H\sum_{k=1}^K\sum_{h=1}^H\sum_{i=1}^M\E{\sum_{t=0}^{h-1}\alpha^{h-t-1}b_{x_t^k,t}^{\paramvec,k}(s_t^k,a_t^k) | F_{k-1}}\tag{Under $\G$}\\ 
    &= 5HM\sum_{k=1}^K\sum_{h=1}^H\E{\sum_{t=0}^{h-1}\alpha^{h-t-1}b_{x_t^k,t}^{\paramvec,k}(s_t^k,a_t^k) | F_{k-1}}\\ 
    &\le 5H^2M\sum_{k=1}^K\sum_{h=1}^H\E{b_{x_h^k,h}^{\paramvec,k}(s_h^k,a_h^k) | F_{k-1}} \\
    & \leq  10H^2M\sqrt{\kappamin}\gamma_K(4\delta)\sum_{k=1}^K\sum_{h=1}^H\E{\frac{1}{\sqrt{n_h^k(s_h^k,a_h^k,x_h^k) + 4\lambda}} | F_{k-1}} \\
    &\leq 10H^2M\brk*{\frac{1}{2\sqrt{\lambda}}\vee1}\sqrt{\kappamin}\gamma_K(4\delta)\sum_{k=1}^K\sum_{h=1}^H\E{\frac{1}{\sqrt{n_h^k(s_h^k,a_h^k,x_h^k)\vee 1}} | F_{k-1}} \\
    & \leq 10H^2M\brk*{\frac{1}{2\sqrt{\lambda}}\vee1}\sqrt{\kappamin}\gamma_K(4\delta)\brk*{18H^2\log\brk*{\frac{1}{\delta}}+2HS(M+1)A +4\sqrt{H^2S(M+1)AK}}\tag{\Cref{lemma: expected cumulative visitation bound}} \\
    & = \Ob\brk*{\brk*{\frac{1}{\sqrt{\lambda}}\vee1}\gamma_K(4\delta)\sqrt{\kappamin H^6SM^3AK}}\\
\end{align*}
Now, note that by the definition in \Cref{prop:convex relaxation}, $\gamma_k(\delta) := \brk*{2+2L + \sqrt{2(1+L)}}\beta_k(\delta) + \sqrt{\frac{2(1+L)HM}{\lambda}}\beta_k^2(\delta)$.
Also, by \Cref{eq: c confidence set}, $\beta_k(\delta) = \frac{M^{3/2}(M+1)SAH}{\sqrt{\lambda}}\brk*{\log\brk*{1 + \frac{k}{(M+1)SA\lambda}} + 2\log\brk*{\frac{2}{\delta}}} + \sqrt{\frac{\lambda}{4M}} + \sqrt{\lambda}L$. 
Plugging in these definitions we have that,
\begin{align*}
    \gamma_k(\delta) &= \brk*{2+2L + \sqrt{2(1+L)}}\brk*{ \frac{M^{3/2}(M+1)SAH}{\sqrt{\lambda}}\brk*{\log\brk*{1 + \frac{k}{(M+1)SA\lambda}} + 2\log\brk*{\frac{2}{\delta}}} + \sqrt{\frac{\lambda}{4M}} + \sqrt{\lambda}L}\\
    & + \sqrt{\frac{2(1+L)HM}{\lambda}}\brk*{ \frac{M^{3/2}(M+1)SAH}{\sqrt{\lambda}}\brk*{\log\brk*{1 + \frac{k}{(M+1)SA\lambda}} + 2\log\brk*{\frac{2}{\delta}}} + \sqrt{\frac{\lambda}{4M}} + \sqrt{\lambda}L}^2 \\
    & \leq \Ob\brk*{\frac{LM^{5/2}SAH}{\lambda}
    \log\brk*{\frac{1}{\delta}+\frac{k}{MSA\lambda \delta}} + L\sqrt{\frac{\lambda}{M}} + \sqrt{\lambda} L^2} \\
    & + \sqrt{\frac{2(1+L)HM}{\lambda}}\Ob\brk*{\frac{M^5S^2AH^2}{\lambda}
    \log^2\brk*{\frac{1}{\delta}+\frac{k}{MSA\lambda \delta}} + \frac{\lambda}{M} + \lambda L^2} \\
    & = \tilde\Ob\brk*{\lambda^{-1} L M^{5/2}SAH
     + \lambda^{1/2} L M^{-1/2}+ \lambda^{1/2} L^2}  \\
    &  + \tilde\Ob\brk*{\lambda^{-3/2} L^{1/2} M^{11/2}S^2A^2H^{5/2}
    + \lambda^{1/2}L^{1/2}M^{-1/2} H^{1/2}  + \lambda^{1/2} L^{5/2}M^{1/2}H^{1/2}} \\
    & \leq 
    \tilde\Ob\Big(\lambda^{-1} L M^{5/2}SAH
     + \lambda^{1/2} L M^{-1/2}+  
    \lambda^{-3/2} L^{1/2} M^{11/2}S^2A^2H^{5/2} \\
    &~~~~~~~~+ \lambda^{1/2}L^{1/2}M^{-1/2} H^{1/2}  + \lambda^{1/2} L^{5/2}M^{1/2}H^{1/2} \Big) \\
    & \leq 
    \tilde\Ob\brk*{L^2 M^{7/4}S^{1/2}A^{1/2}H},
\end{align*}
where we used $\lambda = \frac{M^{2.5}SAH}{L}$ to minimize the above term.

Finally, plugging in this expression we have that
\begin{align*}
    5H\sum_{k=1}^K\sum_{h=1}^H\E{\norm{Z^{\bar \paramvec_k}_h - Z^{\paramvectrue}_h}_1 | F_{k-1}} &\leq \Ob\brk*{\brk*{\frac{1}{\sqrt{\lambda}}\vee1}\gamma_K(4\delta)\sqrt{\kappamin H^6SM^3AK}}\\
    & \leq \tilde \Ob\brk*{L^2 M^{7/4}S^{1/2}A^{1/2}H\sqrt{\kappamin H^6SM^3AK}} \\
    & = \tilde \Ob\brk*{\sqrt{H^8 S^2 A^2 M^{6.5} L^4 \kappamin K}},
\end{align*}
where we assumed that $\lambda \geq 1 $
\end{proof}


\subsection{Optimism in Tractable Logistic DCMDPs -- Proof of \Cref{prop: optimism tractable ldc-ucb}}
\label{appendix: proof of optimistic value proposition}

In this section, we prove \Cref{prop: optimism tractable ldc-ucb}, which allows us to apply the regret decomposition (\Cref{lemma: regret decomposition}) necessary for proving \Cref{thm: regret tractable ldc-ucb}.

\OptimisticValue*
\begin{proof}

We divide the proof into two steps. Defining $\bar{V}_h(s,\hist_h)$ the optimistic value function which follows the equations
\begin{align*}
    &\bar{V}_{H+1}^k(s,\hist_{H+1})=0, \quad \forall s\in\s,\hist_{H+1}\in\mathcal{H}_{H+1}, \text{ and}\\
    &\bar{V}_h^k(s,\hist_h)=\min\brk[c]*{H, \max_a \max_{\bar \suff\in\ci(\suff(\hist_h))}  \brk[c]*{ \sum_{i=1}^{M+1}z_i(\bar \suff)\bar r_{i,h}^k(s, a) + \sum_{i=1}^{M+1}z_i(\bar \suff) \hat P_{i,h}^k(\cdot | s,a)^T\bar{V}_{h+1}^k(\cdot, \hist_{h+1})} }, \\
    &\forall h\in[H],s\in\s,\hist_h\in\mathcal{H}_h,
\end{align*}

we first show that $\bar{V}_h^k(s,\ci_h^k(\hist_h))= \bar{V}_{h}^k(s,\hist_h)$ for all $k\in[K],h\in[H],s\in\s$ and $\hist_h\in\mathcal{H}_h$. This follows due to a simple induction; first notice that the claim trivially holds when $h=H+1$, where both values are $0$. Now fix $h\in[H]$ and assume that $\bar{V}_{h+1}^k(s,\ci_h^k(\hist_{h+1}))= \bar{V}_{h+1}^k(s,\hist_{h+1})$ for all $s\in\s$ and $\hist_{h+1}\in\mathcal{H}_{h+1}$.
In the following, we prove that this implies $\bar{V}_h^k(s,\ci_h^k(\hist_h))= \bar{V}_{h}^k(s,\hist_h)$ for any $s\in\s$ and $\hist_h\in\mathcal{H}_h$, which prove the claim.
\begin{align*}
    \bar{V}_h(s,\ci_h^k(\hist_h))
    &= \min\brk[c]*{\max_{a \in \A} \max_{t \in \thrset{\bar{\bs{Q}}}} \sum_{i=0}^M z_i\brk*{\thr{t}{\bar{\bs{Q}},\ci_h^k(\hist_h)}} \bar Q_i(s,a, \ci_h^k(\hist_h)),H} \\
    & \overset{(1)}{=} \min\brk[c]*{\max_{a \in \A} \max_{\bar \suff\in\ci_h^k(\hist_h)} \sum_{i=0}^M z_i\brk*{\bar \suff} \bar Q_i(s,a, \ci_h^k(\hist_h)),H} \\
    & = \min\brk[c]*{\max_{a \in \A} \max_{\bar \suff\in\ci_h^k(\hist_h)} \sum_{i=0}^M z_i\brk*{\bar \suff} \brk*{\bar r_i(s,a) + \expect*{s' \sim \hat P_i(\cdot|s,a)}{\bar V_{h+1}(s',\ci_{h+1}(a,i))}},H}\\
    & \overset{(2)}= \min\brk[c]*{\max_{a \in \A} \max_{\bar \suff\in\ci_h^k(\hist_h)} \sum_{i=0}^M z_i\brk*{\bar \suff} \brk*{\bar r_i(s,a) + \expect*{s' \sim \hat P_i(\cdot|s,a)}{\bar V_{h+1}(s',\ci_{h+1}^k)}},H}\\
    & \overset{(3)}= \min\brk[c]*{\max_{a \in \A} \max_{\bar \suff\in\ci_h^k(\hist_h)} \sum_{i=0}^M z_i\brk*{\bar \suff} \brk*{\bar r_i(s,a) + \expect*{s' \sim \hat P_i(\cdot|s,a)}{\bar{V}_{h+1}^k(s,\hist_h)}},H} \\
    & = \bar{V}_h^k(s,\hist_h)
\end{align*}
Relation $(1)$ is by \Cref{lemma: threshold optimism}, which proves that when the confidence interval of a multinomial function is rectangular, one of the maximizers of an linear combination w.r.t. this function is a threshold function; therefore, the maximum over threshold functions achieves the same value at the rectangular set. Relation $(2)$ is by the definition of $\ci_{h+1}(a_h,x_h)$ at \Cref{alg: optimistic DP}, which implies that $\ci_{h+1}(a_h,x_h)= \ci_{h+1}^k$. Finally, $(3)$ is by the induction hypothesis.

Next, we prove that under the good event, $\bar{V}_h^k(s,\hist_h) \geq V_h^*(s,\hist_h)$ for all $k\in[K],h\in[H], s\in\s$ and $\hist_h\in\mathcal{H}_h$. This claim is also proved by induction and clearly holds when $h=H+1$, when all values equal zero. Assume that $\bar{V}_{h+1}^k(s,\hist_{h+1}) \geq V_{h+1}^*(s,\hist_{h+1})$ for all $s\in\s$, $\hist_{h+1}\in\mathcal{H}_{h+1}$. Also, assume w.l.o.g. that $\bar{V}_h^k(s,\hist_h)<H$, otherwise the claim trivially holds. Then, denoting
$$a^*\in\arg\max_a \brk[c]*{\brk[s]*{Z^{\paramvectrue}_h r_h}(s, a, \hist_h) + \brk[s]*{T_hV_{h+1}^{*,\paramvectrue}}(s,a, \hist_h)},$$
and under $\G$,
\begin{align*}
    \bar{V}_h^k&(s,\hist_h) - V_h^*(s,\hist_h) \\
    &= 
    \max_a \max_{\bar \suff\in\ci_h^k}  \brk[c]*{ \sum_{i=1}^{M+1}z_i(\bar \suff)\bar r_{i,h}^k(s, a) + \sum_{i=1}^{M+1}z_i(\bar \suff) \hat P_{i,h}^k(\cdot | s,a)^T\bar{V}_{h+1}^k(\cdot, \hist_{h+1})} \\
    &\quad - \max_a \brk[c]*{\brk[s]*{Z^{\paramvectrue}_h r_h}(s, a, \hist_h) + \brk[s]*{T_hV_{h+1}^{*,\paramvectrue}}(s,a, \hist_h)}\\
    & \overset{(1)}\geq 
    \max_{\bar \suff\in\ci_h^k}  \brk[c]*{ \sum_{i=1}^{M+1}z_i(\bar \suff)\bar r_{i,h}^k(s, a^*) + \sum_{i=1}^{M+1}z_i(\bar \suff) \hat P_{i,h}^k(\cdot | s,a^*)^T\bar{V}_{h+1}^k(\cdot, \hist_{h+1})} \\
    &~~~~~~-
    \brk[s]*{Z^{\paramvectrue}_h r_h}(s, a^*, \hist_h) - \brk[s]*{T_hV_{h+1}^{*,\paramvectrue}}(s,a^*, \hist_h) \\
    & \overset{(2)}\geq 
    \brk[s]*{Z^{\paramvectrue}_h \bar r_h^k}(s, a^*, \hist_h) - \brk[s]*{T_h^{\hat P_h^k, \paramvectrue}\bar{V}_{h+1}^k}(s,a^*, \hist_h)
    -
    \brk[s]*{Z^{\paramvectrue}_h r_h}(s, a^*, \hist_h) - \brk[s]*{T_hV_{h+1}^{*,\paramvectrue}}(s,a^*, \hist_h) \\
    &=
    \brk[s]*{Z^{\paramvectrue}_h \bar r_h^k}(s, a^*, \hist_h) - \brk[s]*{Z^{\paramvectrue}_h r_h}(s, a^*, \hist_h)
     +
    \brk[s]*{\brk*{T_h^{\hat P_h^k, \paramvectrue} - T_h}\bar{V}_{h+1}^{k,\paramvectrue}}(s,a^*, \hist_h) 
    +
    \brk[s]*{T_h\brk*{\bar{V}_{h+1}^{k,\paramvectrue} - V_{h+1}^{*,\paramvectrue}}}(s,a^*, \hist_h) \\
    & \overset{(3)}\geq 
     \brk[s]*{Z^{\paramvectrue}_h \bar r_h^k}(s, a^*, \hist_h) - \brk[s]*{Z^{\paramvectrue}_h r_h}(s, a^*, \hist_h)
     +
    \brk[s]*{\brk*{T_h^{\hat P_h^k, \paramvectrue} - T_h}\bar{V}_{h+1}^{k,\paramvectrue}}(s,a^*, \hist_h) \\
    & = \brk[s]*{Z^{\paramvectrue}_h \brk*{\bar r_h^k - r_h}}(s, a^*, \hist_h) + \brk[s]*{ Z^{\paramvectrue}_h \brk*{\hat P_h^k - P_h}\bar{V}_{h+1}^{k,\paramvectrue}}(s,a^*, \hist_h).
\end{align*}
In $(1)$ we used the definition of the max operator. Relation $(2)$ holds since under the good event,
\begin{align*}
\ci(\suff(\hist_h^k; \paramvectrue))
 &= \sum_{t=0}^{h-1} \alpha^{h-t-1}\paramvec_t(s_t, a_t, x_t) \\
 & \in \brk[s]*{\suff(\hist_h; \estparamvec^k) - \sum_{t=0}^{h-1}\alpha^{h-t-1}b_{x_t^k,t}^{\paramvec,k}(s_t^k,a_t^k),\quad \suff(\hist; \estparamvec^k) + \sum_{t=0}^{h-1}\alpha^{h-t-1}b_{x_t^k,t}^{\paramvec,k}(s_t^k,a_t^k)}\tag{Under $\G$}
    \\
    &= \ci_h^k.
\end{align*}
In $(3)$, we used the induction step the induction step. 
Overall, replacing $\bar{r}_h^k$ with its definition, we get that
\begin{align*}
    \bar V_h^k\brk*{s, \hist_h} - V_h^*\brk*{s, \hist_h}
    &\geq
    \brk[s]*{Z^{\paramvectrue}_h \brk*{\hat r_h^k + b_h^{r,k} - r_h}}(s, a^*, \hist_h) +
    \brk[s]*{ Z^{\paramvectrue}_h \brk*{\brk*{\hat P_h^k - P_h}\bar{V}_{h+1}^{k,\paramvectrue} + b_h^{p,k}}}(s,a^*, \hist_h)\\
    &\ge 0,
\end{align*}
where the second inequality is by \Cref{lemma: reward optimism} and \Cref{lemma: transition optimism}, which hold under $\G$.
\end{proof}

\clearpage

\section{Useful Lemmas}\label{appendix: useful lemmas}

\subsection{Optimism Lemmas}
\begin{lemma}[Reward Optimism]
\label{lemma: reward optimism}

For any $k\ge1$, define the event
\begin{align*}
    F^r_k=\brk[c]*{\exists s\in \s,a \in \A,i\in [M],h \in [H]:\ |r_{i,h}(s,a) - \hat{r}_{i,h}^{k,r}(s,a)| > b_{i,h}^k(s,a) }.
\end{align*}
Then, under $\bar{F}_k^r$, for any $\paramvec\in\paramset$, $h\in[H]$, $s\in\s$, $a\in\A$ and $\hist_h\in\mathcal{H}_h$, it holds that 
\begin{align*}
    \brk[s]*{Z^{\paramvec}_h (\hat r_h^{k} + b_h^{r,k} - r)}(s,a,\hist_h) \ge0
\end{align*}
\end{lemma}
\begin{proof}
    The result directly follows by the definition of $\bar{F}_k^r$, since
    \begin{align*}
        \brk[s]*{Z^{\paramvec}_h (\hat r_h^{k} + b_h^{r,k} - r)}(s,a,\hist_h)
        &\ge \min_i\brk[c]*{\brk*{\hat r_{i,h}^{k}(s,a) - r_i(s,a)} + b_{i,h}^{r,k}(s,a)} \\
        &\ge \min_i\brk[c]*{-b_{i,h}^{r,k}(s,a) + b_{i,h}^{r,k}(s,a)} \tag{Under $\bar{F}_k^r$}\\
        &=0
    \end{align*}
\end{proof}

\begin{lemma}[Transition Optimism]
\label{lemma: transition optimism}

For any $k\ge1$, define the event
\begin{align*}
    F^p_k=\brk[c]*{\exists s\in \s,a \in \A,i\in [M],h \in [H]:\ \norm{P_{i,h}(\cdot\mid s,a)- \hat{P}_{i,h}^k(\cdot\mid s,a)}_1 > \frac{1}{H}b_{i,h}^{p,k}(s,a) }.
\end{align*}
Then, under $\bar{F}_k^p$, for any $\paramvec\in\paramset$, $h\in[H]$, $s\in\s$, $a\in\A$, $\hist_h\in\mathcal{H}_h$ and $V\in[0,H]^{S}$, it holds that 
\begin{align*}
    \brk[s]*{ Z^{\paramvec}_h \brk*{\brk*{\hat P_h^k - P_h}V + b_h^{p,k}}}(s,a, \hist_h) \ge0
\end{align*}
\end{lemma}
\begin{proof}
    The result directly follows by the definition of $\bar{F}_k^p$ and Cauchy-Schwartz inequality, since
    \begin{align*}
        \brk[s]*{ Z^{\paramvec}_h \brk*{\brk*{\hat P_h^k - P_h}V + b_h^{p,k}}}(s,a, \hist_h)
        &\ge \min_i\brk[c]*{\brk[s]*{\brk*{\hat P_h^k - P_h}V}(s,a) + b_{i,h}^{p,k}(s,a)} \\
        & \ge \min_i\brk[c]*{-\norm{P_{i,h}(\cdot\mid s,a)- \hat{P}_{i,h}^k(\cdot\mid s,a)}_1\norm{V}_{\infty} + b_{i,h}^{p,k}(s,a)} \tag{C.S}\\
        & \ge \min_i\brk[c]*{-\frac{1}{H}b_{i,h}^{p,k}(s,a)\cdot H + b_{i,h}^{p,k}(s,a)}\tag{Under $\bar{F}_k^p$}\\
        &=0
    \end{align*}
\end{proof}

\subsection{Decomposition Lemmas}

\begin{lemma}
\label{lemma: difference decomposition}
    \begin{align*}
         V&: \s \times \mathcal{H} \mapsto \R, \\
         Z^{(1)}, Z^{(2)}&: \s \times \A \times \mathcal{H} \mapsto \Delta_{\X}, \\
        r^{(1)}, r^{(2)}& : \s \times \A \times \mathcal{H} \times \X \mapsto \R, \text{ and} \\
        P^{(1)}, P^{(2)}& :\s \times \A \times \mathcal{H} \times \X \mapsto \Delta_{\s}.
    \end{align*}
    Then, for any $s \in \s, a \in \A, h \in [H], \hist_h \in \mathcal{H}$
    \begin{align*}
    \brk[s]*{Z^{(1)}_h r_h^{(1)}+T_h^{P^{(1)}, Z^{(1)}}V_{h+1}}(s, a, \hist_h)
    &-
    \brk[s]*{Z^{(2)}_h r_h^{(2)}+T_h^{P^{(2)}, Z^{(2)}}V_{h+1}}(s, a, \hist_h) \\
    &\qquad=
    \brk[s]*{Z^{(2)}_h \brk*{r_h^{(1)} - r_h^{(2)}}}(s, a, \hist_h)   \\
    &\qquad\quad+
    \brk[s]*{\brk*{Z^{(1)}_h - Z^{(2)}_h}\brk*{r_h^{(1)} + P_h^{(1)}V_{h+1}}}(s,a, \hist_h) \\
    &\qquad\quad+
    \brk[s]*{ Z^{(2)}_h \brk*{P_h^{(1)} - P_h^{(2)}}{V}_{h+1}}(s,a, \hist_h).
    \end{align*}
\end{lemma}
\begin{proof}
    We have that
    \begin{align*}
        &\brk[s]*{Z^{(1)}_h r_h^{(1)}+T_h^{P^{(1)}, Z^{(1)}}V_{h+1}}(s, a, \hist_h)
        -
        \brk[s]*{Z^{(2)}_h r_h^{(2)}+T_h^{P^{(2)}, Z^{(2)}}V_{h+1}}(s, a, \hist_h) \\
        &=
        \brk[s]*{\brk*{Z^{(1)}_h - Z^{(2)}_h} r_h^{(1)} }(s, a, \hist_h)
        +
        \brk[s]*{Z^{(2)} \brk*{r_h^{(1)} - r_h^{(2)}}}(s, a, \hist_h) 
        +
        \brk[s]*{\brk*{T_h^{P^{(1)}, Z^{(1)}} - T_h^{P^{(2)}, Z^{(2)}}}V_{h+1}}(s,a, \hist_h) \\
        &=
        \brk[s]*{\brk*{Z^{(1)}_h - Z^{(2)}_h} r_h^{(1)} }(s, a, \hist_h)
        +
        \brk[s]*{Z^{(2)} \brk*{r_h^{(1)} - r_h^{(2)}}}(s, a, \hist_h) 
        +
        \brk[s]*{\brk*{Z^{(1)}_h P_h^{(1)} - Z^{(2)}_h P_h^{(2)}}V_{h+1}}(s,a, \hist_h) \\
        &=
        \brk[s]*{\brk*{Z^{(1)}_h - Z^{(2)}_h} r_h^{(1)} }(s, a, \hist_h)
        +
        \brk[s]*{Z^{(2)} \brk*{r_h^{(1)} - r_h^{(2)}}}(s, a, \hist_h) \\
        &\quad +
        \brk[s]*{\brk*{Z^{(1)}_h - Z^{(2)}_h} P_h^{(1)} V_{h+1}}(s,a, \hist_h) 
        +
        \brk[s]*{ Z^{(2)}_h \brk*{P_h^{(1)} - P_h^{(2)}} V_{h+1}}(s,a, \hist_h) \\
        &=\brk[s]*{Z^{(2)}_h \brk*{r_h^{(1)} - r_h^{(2)}}}(s, a, \hist_h)   
        +
        \brk[s]*{\brk*{Z^{(1)}_h - Z^{(2)}_h}\brk*{r_h^{(1)} + P_h^{(1)}V_{h+1}}}(s,a, \hist_h) 
        +
        \brk[s]*{ Z^{(2)}_h \brk*{P_h^{(1)} - P_h^{(2)}}{V}_{h+1}}(s,a, \hist_h).
    \end{align*}

This completes the proof.
\end{proof}

Next, recall that by embedding the history into the state, every DCMDP can be represented as an MDP. This equivalence will allow us to apply the following lemma on DCMDPs:

\begin{lemma}[Value difference lemma, e.g., \citet{dann2017unifying}, Lemma E.15]\label{lemma: value difference lemma}
    Consider two MDPs $\M=\brk*{\sset,\aset,P,r,H}$ and $\M'=\brk*{\sset,\aset,P',r',H}$. For any policy $\pi$ and any $s,h$, the following relation holds:
    \begin{align*}
        &V_h^{\pi}(s;\M') - V_h^{\pi}(s;\M) \\
            & \quad = \E{\sum_{t=h}^H \brk*{r_t'(s_t,a_t) - r_t(s_t,a_t)}  + \brk*{P'-P}(\cdot \mid s_t,a_t)^T V_{t+1}^{\pi}\brk*{\cdot ; \M'} \vert s_h = s,\pi,P}
    \end{align*}
\end{lemma}

\begin{corollary}
    [Truncated value difference lemma]\label{corollary: truncated value difference lemma}
    Consider two MDPs $\M=\brk*{\sset,\aset,P,r,H}$ and $\M'=\brk*{\sset,\aset,P',r',H}$. Also, for any $C\in\R$, define the truncated value of a policy $\pi$ under MDP $\M$ by the solution to the truncated dynamic programming problem
    \begin{align*}
        &V_{H+1}^{\pi}(s;\M,C)=0, &\forall s\in\s\\
        &V_h^{\pi}(s;\M,C)=\expect*{a\sim\pi}{\min\brk[c]*{C, r_h(s,a)+P(\cdot \mid s,a)^TV_{h+1}^{\pi}(\cdot;\M,C)}}, &\forall h\in[H],s\in\s.
    \end{align*}
    Then, for any policy $\pi$, any $s\in\s,h\in[H]$ and any $C\in\R$, the following relation holds:
    \begin{align*}
        &V_h^{\pi}(s;\M',C) - V_h^{\pi}(s;\M) \\
        & \quad \leq \E{\sum_{t=h}^H \brk*{r_t'(s_t,a_t) - r_t(s_t,a_t)}  + \brk*{P'-P}(\cdot \mid s_t,a_t)^T V_{t+1}^{\pi}\brk*{\cdot ; \M',C} \vert s_h = s,\pi,P}
    \end{align*}
\end{corollary}
\begin{proof}
We build an MDP whose value (without truncation) is $V_h^{\pi}(s;\M',C)$ and its reward are always smaller than the rewards of $\M'$. In particular, for any $h\in[H],s\in\s$ and $a\in\A$, define the new reward function
\begin{align}
    \label{eq:truncated reward}
    \bar{r}_h(s,a) = r'_h(s,a) - \max\brk[c]*{0, r'_h(s,a)+P'(\cdot \mid s,a)^TV_{h+1}^{\pi}(s;\M',C)-C} \leq r'_h(s,a),
\end{align}
and denote $\bar{\M}=\brk*{\sset,\aset,P',\bar r,H}$. Clearly, $V_{H+1}^{\pi}(s;\bar\M)=V_{H+1}^{\pi}(s;\M',C)=0$. Now assume by induction the the equality holds for all $t>h$ and all $s\in\s$. 

Let $s\in\s$ be some state. If for some $a\in\A$, the maximizer in \Cref{eq:truncated reward} equals zero, then there was no truncation in the value iteration and so, by the induction hypothesis, we get
\begin{align*}
    \bar{r}_h(s,a) + P'(\cdot \mid s,a)^TV_{h+1}^{\pi}(\cdot;\bar\M) 
    &= r'_h(s,a) + P'(\cdot \mid s,a)^TV_{h+1}^{\pi}(\cdot;\M',C) \\
    &= \min\brk[c]*{C, r_h(s,a)+P(\cdot \mid s,a)^TV_{h+1}^{\pi}(\cdot;\M',C)}.
\end{align*}
On the other hand, if the maximizer in \Cref{eq:truncated reward} is not zero, then one can easily verify that
\begin{align*}
    \bar{r}_h(s,a) + P'(\cdot \mid s,a)^TV_{h+1}^{\pi}(\cdot;\bar\M)
    = \min\brk[c]*{C, r_h(s,a)+P'(\cdot \mid s,a)^TV_{h+1}^{\pi}(\cdot;\M',C)}
    =C
\end{align*}
Therefore, this equality holds for all $a\in\A$ and thus
\begin{align*}
    V_h^{\pi}(s;\bar\M)
    &= \expect*{a\sim\pi}{\bar{r}_h(s,a) + P'(\cdot \mid s,a)^TV_{h+1}^{\pi}(\cdot;\bar\M)} \\
    & = \expect*{a\sim\pi}{\min\brk[c]*{C, r_h(s,a)+P'(\cdot \mid s,a)^TV_{h+1}^{\pi}(\cdot;\M;,C)}} \\
    &= V_h^{\pi}(s;\M',C),
\end{align*}
and by induction, this equality holds for all $h\in[H]$ and $s\in\s$. Now, using this fact with \Cref{lemma: value difference lemma} on $\M$ and $\bar\M$, we get:
\begin{align*}
    &V_h^{\pi}(s;\M',C) - V_h^{\pi}(s;\M) \\
        & \quad = \E{\sum_{t=h}^H \brk*{\bar{r}_t(s_t,a_t) - r_t(s_t,a_t)}  + \brk*{P'-P}(\cdot \mid s_t,a_t)^T V_{t+1}^{\pi}\brk*{\cdot ; \M',C} \vert s_h = s,\pi,P} \\
        & \quad \leq \E{\sum_{t=h}^H \brk*{r_t'(s_t,a_t) - r_t(s_t,a_t)}  + \brk*{P'-P}(\cdot \mid s_t,a_t)^T V_{t+1}^{\pi}\brk*{\cdot ; \M',C} \vert s_h = s,\pi,P}
\end{align*}
where the inequality is since $\bar{r}_h(s,a) \leq r'_h(s,a)$ for all $h,s,a$.
\end{proof}


We are now ready to present the general regret decomposition lemma.
\begin{lemma}[Regret Decomposition] \label{lemma: regret decomposition}
Assume that there exist an optimistic value function $\bar{V}_h^k$ such that the following hold:
\begin{enumerate}
    \item \textbf{Value representation.} For all $k\in[K]$ and $h\in[H]$, there exist $\bar{Z}_h^k:\s\times\A\times\mathcal{H}_h \mapsto \Delta_{\X}$, $\bar{r}_h^k:\s\times\A\times\X \mapsto \R$ and $\bar{P}_h^k:\s\times\A\times\X \mapsto \Delta_{\s}$ such that for all $k\ge1$, $h\in[H]$, $s,\in\s$ and $\hist_h\in\mathcal{H}_h$, it holds that 
    $$\bar{V}_h^k(s,\hist_h)\leq \bar{Z}_h^k \bar{r}_h^k + T_h^{\bar{P}_h^k, \bar \paramvec_k}  (\cdot\vert s_h^k,a_h^k,\hist_h^k)^T \bar{V}_{h+1}^{k}(\cdot, \hist_{h+1}^k).$$
    \item \textbf{Boundedness.} For all $k\ge1$, $h\in[H]$, $s,\in\s$, $a\in\A$ $i\in[M+1]$ and $\hist_h\in\mathcal{H}_h$, it holds that $0\le \bar{V}_h^k(s,\hist_h) \le H$ and $0\le\bar{r}_{i,h}^k(s,a,\hist_h)\le cH$ for some $c>0$.
    \item \textbf{Optimism.}  For all $k\ge1$, it holds that $\bar{V}_h^k(s_1^k) \ge V_1^*(s_1^k)$.
\end{enumerate}
Then, the regret can be bounded by
\begin{align*}
    \Reg{K} 
    &\leq \sum_{k=1}^K \sum_{h=1}^H\sum_{i=1}^{M+1}\E{ z_{i,h}^k\abs{\bar r_{i,h}^k(s_h^k, a_h^k) - r_{i,h}(s_h^k, a_h^k)} | F_{k-1}} \\
    &\quad+ H\sum_{k=1}^K \sum_{h=1}^H\sum_{i=1}^{M+1}\E{ z_{i,h}^k\norm{\brk*{\bar{P}_h^k - P_h}(\cdot|s_h^k, a_h^k)}_1 | F_{k-1}} \\
    &\quad+ (c+1)H\sum_{k=1}^K \sum_{h=1}^H\E{\norm{\bar{Z}_h^k - Z^{\paramvectrue}_h}_1 | F_{k-1}}
\end{align*}
\end{lemma}
\begin{proof}

\begin{align*}
    \Reg{K}
    &= \sum_{k=1}^K V_1^*(s_1^k) - V_1^{\pi^k}(s_1^k) \\
    &\leq
    \sum_{k=1}^K \bar{V}_1^k(s_1^k) - V_1^{\pi^k}(s_1^k) \tag{Optimism}\\
    & \leq
    \sum_{k=1}^K \sum_{h=1}^H 
    \E{(\bar{Z}_h^k \bar{r}_h^k-Z^{\paramvectrue}_h r_h)(s_h^k, a_h^k, \hist_h^k) + (T_h^{\bar{P}_h^k, \bar \paramvec_k}  -T_h)(\cdot\vert s_h^k,a_h^k,\hist_h^k)^T \bar{V}_{h+1}^{k}(\cdot, \hist_{h+1}^k) | \F_{k-1}} \tag{\Cref{corollary: truncated value difference lemma}}\\
    & \leq
    \underbrace{\sum_{k=1}^K \sum_{h=1}^H
    \E{\brk[s]*{Z^{\paramvectrue}_h \brk*{\bar r_h^k - r_h}}(s_h^k, a_h^k, \hist_h^k) | F_{k-1}}}_{(i)}\\
    &\quad +
    \underbrace{\sum_{k=1}^K \sum_{h=1}^H
    \E{\brk*{\bar{Z}_h^k - Z^{\paramvectrue}_h}\brk*{\bar r_h^k(s_h^k, a_h^k) + \bar{P}_h^k(\cdot|s_h^k, a_h^k) \bar{V}_{h+1}^{k}(\cdot, \hist_{h+1}^k)} | F_{k-1}}}_{(ii)} \\
    &\quad +
    \underbrace{\sum_{k=1}^K \sum_{h=1}^H \E{Z^{\paramvectrue}_h \brk[s]*{\bar{P}_h^k - P_h}\brk{\cdot | s_h^k, a_h^k, \hist_h^k}^T\bar{V}_{h+1}^{k}(\cdot, \hist_{h+1}^k) | F_{k-1}}}_{(iii)}
    \tag{\Cref{lemma: difference decomposition}} 
\end{align*}
Notice that in the application of \Cref{lemma: value difference lemma}, which was applied w.r.t. $\pi^k$, we used the fact that any DCMDP can be represented as an MDP whose history was embedded into the state. We now bound each of the terms of the decomposition.

\paragraph{Reward error}
\begin{align*}
    (i) 
    &= \sum_{k=1}^K \sum_{h=1}^H
    \E{\sum_{i=1}^{M+1} z_{i,h}^k\brk*{\bar r_{i,h}^k(s_h^k, a_h^k) - r_{i,h}(s_h^k, a_h^k)} | F_{k-1}} \\
    & \leq \sum_{k=1}^K \sum_{h=1}^H\sum_{i=1}^{M+1}\E{ z_{i,h}^k\abs{\bar r_{i,h}^k(s_h^k, a_h^k) - r_{i,h}(s_h^k, a_h^k)} | F_{k-1}}
\end{align*}

\paragraph{Latent features error}
\begin{align*}
    (ii) 
    &\leq \sum_{k=1}^K \sum_{h=1}^H
    \E{\norm{\bar{Z}_h^k - Z^{\paramvectrue}_h}_1\norm{\bar r_h^k(s_h^k, a_h^k) + \hat P_h^k(\cdot|s_h^k, a_h^k)^T \bar{V}_{h+1}^{k}(\cdot, \hist_{h+1}^k)}_{\infty} | F_{k-1}} \tag{H\"older} \\
    & \leq (c+1)H\sqrt{M+1}\sum_{k=1}^K \sum_{h=1}^H\E{\norm{\bar{Z}_h^k - Z^{\paramvectrue}_h}_1 | F_{k-1}}
\end{align*}
where the last inequality is since the optimistic value is bounded in $[0,H]$ and the reward is in $[0,cH]$.

\paragraph{Transition error}
\begin{align*}
    (iii) 
    &= \sum_{k=1}^K \sum_{h=1}^H
    \E{\sum_{i=1}^{M+1} z_{i,h}^k\brk*{\brk*{\hat P_h^k - P_h}(\cdot|s_h^k, a_h^k)^T\bar{V}_{h+1}^{k}(\cdot, \hist_{h+1}^k)} | F_{k-1}} \\
    & \leq \sum_{k=1}^K \sum_{h=1}^H\sum_{i=1}^{M+1}\E{ z_{i,h}^k\norm{\brk*{\hat P_h^k - P_h}(\cdot|s_h^k, a_h^k)}_1\norm{\bar{V}_{h+1}^{k}(\cdot, \hist_{h+1}^k)}_{\infty} | F_{k-1}} \tag{H\"older} \\
    & \leq H\sum_{k=1}^K \sum_{h=1}^H\sum_{i=1}^{M+1}\E{ z_{i,h}^k\norm{\brk*{\hat P_h^k - P_h}(\cdot|s_h^k, a_h^k)}_1 | F_{k-1}}  \tag{Boundedness}
\end{align*}

Combining all bounds concludes the proof.
    
\end{proof}

\subsection{Visitation-Summation Lemmas}

\begin{lemma}[Expected Cumulative Visitation Bound, Lemma 22, \citet{efroni2020reinforcement}, adapted to DCMDPs] \label{lemma: expected cumulative visitation bound}
    Let $\brk[c]*{\F_{k}}_{k=1}^K$ be the natural filtration. Then, with probability greater than $1-\delta$ it holds that
    \begin{align*}
        \sum_{k=1}^K \sum_{h=1}^H \sum_{i=1}^{M+1}
        \E{ \frac{z_{i,h}^k}{\sqrt{n_h^k(s_h^k,a_h^k,i)\vee 1}} | F_{k-1}} 
        &= \sum_{k=1}^K\E{  \sum_{h=1}^H \frac{1}{\sqrt{n_h^k(s_h^k,a_h^k,x_h^k)\vee 1}} | F_{k-1}}\\
        &\leq 18H^2\log\brk*{\frac{1}{\delta}}+2HS(M+1)A +4\sqrt{H^2S(M+1)AK}\\
        &=\Ob\brk*{H\brk*{SMA + H\log\brk*{\frac{1}{\delta}}} + \sqrt{H^2SMAK}} \\
        &= \tilde \Ob \brk*{\sqrt{H^2SMAK}} 
    \end{align*}
\end{lemma}
\begin{proof}
We start by rewriting the sum as follows:
\begin{align*}
    \sum_{k=1}^K \sum_{h=1}^H \sum_{i=1}^{M+1}
    \E{ \frac{z_{i,h}^k}{\sqrt{n_h^k(s_h^k,a_h^k,i)\vee 1}} | F_{k-1}}
    &= \sum_{k=1}^K \sum_{h=1}^H \E{ \sum_{i=1}^{M+1}z_{i,h}^k
    \frac{1}{\sqrt{n_h^k(s_h^k,a_h^k,i)\vee 1}} | F_{k-1}} \\
    & =\sum_{k=1}^K \sum_{h=1}^H \E{ \expect*{x_h^k\sim z_{i,h}^k}{\frac{1}{\sqrt{n_h^k(s_h^k,a_h^k,x_h^k)\vee 1}}} | F_{k-1}} \\
    & =\sum_{k=1}^K \sum_{h=1}^H \E{ \frac{1}{\sqrt{n_h^k(s_h^k,a_h^k,x_h^k)\vee 1}} | F_{k-1}} \\
    & =\sum_{k=1}^K\E{  \sum_{h=1}^H \frac{1}{\sqrt{n_h^k(s_h^k,a_h^k,x_h^k)\vee 1}} | F_{k-1}},
\end{align*}
which proves the first equality. Now, defining $Y_k=\sum_{h=1}^H \frac{1}{\sqrt{n_h^k(s_h^k,a_h^k,x_h^k)\vee 1}}$, which is $\F_k$-measurable and bounded almost surely in $[0,H]$, we can apply Lemma 27 of \citep{efroni2021confidence} and get that for any $\delta>0$, with probability at least $1-\delta$,
\begin{align*}
    \sum_{k=1}^K \sum_{h=1}^H \sum_{i=1}^{M+1}
    \E{ \frac{z_{i,h}^k}{\sqrt{n_h^k(s_h^k,a_h^k,x_h^k))\vee 1}} | F_{k-1}}
    &\leq \brk*{1+\frac{1}{2H}}\sum_{k=1}^K\sum_{h=1}^H \frac{1}{\sqrt{n_h^k(s_h^k,a_h^k,x_h^k)\vee 1}} + 2(2H+1)^2\log\frac{1}{\delta} \\
    & \leq 2\sum_{k=1}^K\sum_{h=1}^H \frac{1}{\sqrt{n_h^k(s_h^k,a_h^k,x_h^k)\vee 1}} + 18H^2\log\frac{1}{\delta}.
\end{align*}
Finally, observing that every time a context-state-action is visited, its count increases, we can bound the sum by
\begin{align*}
    \sum_{k=1}^K\sum_{h=1}^H \frac{1}{\sqrt{n_h^k(s_h^k,a_h^k,x_h^k)\vee 1}}
    &= \sum_{k=1}^K\sum_{h=1}^H\sum_{x\in\X}\sum_{s\in\s}\sum_{a\in\A} \frac{\indicator{x_h^k=x,s_h^k=s,a_h^k=a}}{\sqrt{n_h^k(s_h^k,a_h^k,x_h^k)\vee 1}}\\
    &\leq \sum_{h=1}^H\sum_{x\in\X}\sum_{s\in\s}\sum_{a\in\A}\brk*{1+\sum_{n=1}^{n_h^K(s,a,x)}\frac{1}{\sqrt{n}}} \\
    & \leq HS(M+1)A + \sum_{h=1}^H\sum_{x\in\X}\sum_{s\in\s}\sum_{a\in\A} 2\sqrt{n_h^K(s,a,x)} \\
    & \leq HS(M+1)A + 2\sqrt{HS(M+1)A \underbrace{\sum_{h=1}^H\sum_{x\in\X}\sum_{s\in\s}\sum_{a\in\A}n_h^K(s,a,x)}_{=HK} }\tag{Cauchy Schwartz} \\
    & = HS(M+1)A + 2\sqrt{H^2S(M+1)AK}.
\end{align*}
Substituting this bound concludes the proof.
\end{proof}

\begin{lemma}[Elliptical potential lemma, \citet{abbasi2011improved}]
\label{lemma: elliptical potential lemma}
    Let $\brk[c]{x_t}_{t=1}^\infty$ be a sequence in $\R^d$ such that $\norm{x_t}_2\le L$ for all $t\ge1$ and let $\bs{V}_t = \lambda_I+\sum_{s=1}^{t-1}x_sx_s^T$. Then,
    \begin{align*}
        \sum_{t=1}^n \min\brk[c]*{\norm{x_t}_{\bs{V}_t^{-1}}^2,1} \leq 2d\log\frac{\lambda d+nL^2}{\lambda d}
    \end{align*}
\end{lemma}
\begin{corollary}
    \label{corollary: elliptical potential lemma}
    Let $\brk[c]{x_h^k}_{k\ge1,h\in[H]}$ be a sequence in $\R^d$ such that $\norm{x_t}_2\le L$ for all $k,h$ and let $\bs{V}_k = \lambda_I+\sum_{k'=1}^{k-1}\sum_{t=1}^H x_t^{k'}{x_t^{k'}}^T$. Then,
    \begin{align*}
        \sum_{k'=1}^k\sum_{h=1}^H \norm{x_h^{k'}}_{\bs{V}_{k'}^{-1}} \leq \frac{\sqrt{2KH^2d\log\frac{\lambda d+kL^2}{\lambda d}}}{\max\brk[c]{1,L/\sqrt{\lambda}}}.
    \end{align*}
\end{corollary}
\begin{proof}
    Define the matrices  $\bs{V}_{k,h} = \sum_{k'=1}^{k-1} x_{h}^k{x_{h}^k}^T$; clearly, it holds that $\bs{V}_{k}\preceq \bs{V}_{k,h}$ for all $k,h$, and thus, by applying \Cref{lemma: elliptical potential lemma} for each of these matrices, we get
    \begin{align*}
        \sum_{k'=1}^k\sum_{h=1}^H \min\brk[c]*{\norm{x_h^{k'}}_{\bs{V}_{k'}^{-1}}^2,1}
        \leq \brk*{\sum_{k'=1}^k \min\brk[c]*{\norm{x_h^{k'}}_{\bs{V}_{k',h}^{-1}}^2,1}}
        \leq \sum_{h=1}^H 2d\log\frac{\lambda d+kL^2}{\lambda d}
        = 2dH\log\frac{\lambda d+kL^2}{\lambda d}.
    \end{align*}
    Also, notice that if $\norm{x}_2\le L$, then $\norm{x}_{\bs{V}_k^{-1}}^2\le \frac{L^2}{\lambda_{\min}(\bs{V}_k)}\leq\frac{L^2}{\lambda}$, and thus 
    \begin{align*}
        \norm{x_h^{k'}}_{\bs{V}_{k'}^{-1}}^2
        \leq \frac{\min\brk[c]*{\norm{x_h^{k'}}_{\bs{V}_{k'}^{-1}}^2,1}}{\max\brk[c]{1,L^2/\lambda}}.
    \end{align*}
    Finally, the desired result is achieved by the Cauchy-Schwartz inequality:
    \begin{align*}
        \sum_{k'=1}^k\sum_{h=1}^H \norm{x_h^{k'}}_{\bs{V}_{k'}^{-1}}^2
        &\leq \sqrt{KH\sum_{k'=1}^k\sum_{h=1}^H\norm{x_h^{k'}}_{\bs{V}_{k'}^{-1}}^2} \\
        & \leq \frac{\sqrt{KH\sum_{k'=1}^k\sum_{h=1}^H \min\brk[c]*{\norm{x_h^{k'}}_{\bs{V}_{k'}^{-1}}^2,1}}}{\max\brk[c]{1,L/\sqrt{\lambda}}} \\
        & \leq \frac{\sqrt{2KH^2d\log\frac{\lambda d+kL^2}{\lambda d}}}{\max\brk[c]{1,L/\sqrt{\lambda}}}
    \end{align*}
\end{proof}

\section{Threshold Optimistic Planning}
\label{appendix: threshold optimistic planning}

\begin{algorithm*}[t!]
\caption{Optimistic Threshold Planner for Logistic DCMDPs}
\label{alg: optimistic DP}
\begin{algorithmic}[1]
\STATE{ \textbf{require:} Optimistic reward $\bar r$, estimated transition $\hat P$, and rectangular confidence set $\rectset$ for $\estparamvec_T$.} 
\STATE \textbf{init:} $\bar V_H(s, \ci) \gets 0$, for all $s, \ci \in \s \times \ciset^k$
\FOR{$h = H-1, \hdots, 1$}
    \FOR{\textbf{each} $s \in \s, \ci_h \in \brk[c]*{\ci(\suff(\hist_h; \estparamvec_T)): \hist_h \in \mathcal{H}_h}$}
        \STATE $\ci_{h+1} := \alpha\ci_h + \brk[s]*{\bs{l}_h^k, \bs{u}_h^k}$ 
        \STATE $\bar Q_i(s,a, \ci_h) 
            =
            \bar r_i(s,a) + \expect*{s' \sim \hat P_i(\cdot|s,a)}{\bar V_{h+1}(s',\ci_{h+1})}$
            \hfill {\color{gray}// State-action optimistic value}
        \STATE $\bar V_h(s, \ci_h) 
                =
                \min
                \brk[c]*{
                \max_{a \in \A, t \in \thrset{\bar{\bs{Q}}}}
                \sum_{i=0}^M z_i\brk*{\thr{t}{\bar{\bs{Q}},\ci_h}}
                \bar Q_i(s,a, \ci_h),
            H
            }$
            \hfill {\color{gray}// \Cref{lemma: threshold optimism}}
        \STATE $\bar \pi(s, \ci_h) \in \arg\max_{a \in \A} \max_{t \in \thrset{\bar{\bs{Q}}}}
                \sum_{i=0}^M z_i\brk*{\thr{t}{\bar{\bs{Q}},\ci_h}}
                \bar Q_i(s,a, \ci_h)$
    \ENDFOR
\ENDFOR
\STATE Output $\bar \pi(s, \hist) = \bar \pi (s, \ci(\suff(\hist)))$
\end{algorithmic}
\end{algorithm*}

\subsection{Proof of Threshold Optimism -- \Cref{lemma: threshold optimism}}
\label{appendix: proof of threshold optimism}
\ThresholdOptimism*
\begin{proof}
    For brevity, throughout the proof, we assume that $\eta=1$, namely, $z_i(\bs{x})
    =
    \frac{\exp \brk*{x_i}}
    {1 + \sum_{m=1}^M\exp \brk*{x_m}}$. This has no impact on the proof, since one can always denote $\brk[s]*{\tilde{\bs{l}}, \tilde{\bs{u}}}=\brk[s]*{\eta\bs{l}, \eta\bs{u}}$ and follow the rest of the proof with the modified intervals.
    
    Let $X^* = \arg\max_{\bs{x} \in \mathcal{R}} f(\bs{x})$. We start by showing that there exists at least one solution at the extreme points of $\mathcal{R}$. We then show that solutions at the extreme points have a thresholding behavior.
    
    \paragraph{Part 1.} We first show that $X^* \cap \text{ext}(\mathcal{R}) \neq \emptyset$, i.e., there exists $\bs{x}^* \in X^*$ that is an extreme point of the set $\mathcal{R}$. Note that $f$ is continuous and $\mathcal{R}$ is a compact set, therefore $X^*$ is nonempty. 
    
    Let $\bs{x}^* \in X^*$ and choose some $k \in [M]$. We show that by replacing $x_k^*$ by either $l_k$ or $u_k$, we get another solution at $X^*$. Repeatedly doing so for all $k\in\brk[s]{M}$ will lead to $\bs{x}^* \in \text{ext}(\mathcal{R})$ and conclude this part of the proof.
    
    We now fix $x_1^*, \hdots, x_{k-1}^*, x_{k+1}^*, \hdots, x_M^*$ and study $f(\bs{x}^*)$ only as a function of $x_k^*$. We also use the convention, $x_0^*=0$. Then,
    \begin{align*}
        f(\bs{x}^*) 
        &= 
        \sum_{i=0}^M z_i(\bs{x}^*)v_i \\
        &=
        \sum_{i=0}^M 
        \frac{\exp \brk*{x_i^*}}
        {\sum_{j=0}^M\exp \brk*{x_j^*}}
        v_i \\
        &=
        \frac{\exp \brk*{x_k^*}v_k}
        {\sum_{j=0}^M\exp \brk*{x_j^*}}
        +
        \frac{\sum_{i \neq k} \exp \brk*{x_i^*} v_i}
        {\sum_{j=0}^M\exp \brk*{x_j^*}} \\
        &=
        \frac{\exp \brk*{x_k^*}v_k }
        {\sum_{j=0}^M\exp \brk*{x_j^*}}
        +
        \brk*{1 - \frac{\exp \brk*{x_k^*}}
        {\sum_{j=0}^M\exp \brk*{x_j^*}}}
        \frac{\sum_{i \neq k} \exp \brk*{x_i^*} v_i}
        {\sum_{j \neq k} \exp \brk*{x_j^*}}.
    \end{align*}
    Denote $\lambda(x_k^*) = \frac{\exp \brk*{x_k^*}}
        {\sum_{j=0}^M\exp \brk*{x_j^*}}$, and $v_{\text{ref}} = \frac{\sum_{i \neq k} \exp \brk*{x_i^*} v_i}
        {\sum_{j \neq k} \exp \brk*{x_j^*}}$. Then,
    \begin{align*}
        f(\bs{x}^*) = \lambda(x_k^*) v_k + (1-\lambda(x_k^*))v_{\text{ref}}.
    \end{align*}
    Note that, since we fixed $x_1^*, \hdots, x_{k-1}^*, x_{k+1}^*, \hdots, x_M^*$, then $v_{\text{ref}}$ is constant (does not depend on $x_k^*$). Also, $\lambda(x_k^*)$ is a strictly monotonically increasing function in $x_k^*$ and $f(\bs{x}^*)$ is linear in $\lambda(x_k^*)$. Hence $\max_{x_k^*} f(\bs{x}^*)$ is achieved either for $\tilde{x}_k^* = \arg\min_{x_k \in [l_k, u_k]} \lambda(x_k) = l_k$ or $\tilde{x}_k^* = \arg\max_{x_k \in [l_k, u_k]} \lambda(x_k) = u_k$. Denoting the solution that replaces $x_k^*$ with the maximizer $\tilde{x}_k^*$ by $\tilde{\bs{x}}^*$, we get that $f(\tilde{\bs{x}}^*)\ge f(\bs{x}^*)$, but since $\bs{x}^*\in X^*$, so does $\tilde{\bs{x}}^*\in X^*$. Following this process for all $k\in\brk[s]{M}$ leads to an optimal $\tilde{\bs{x}}^*\in \text{ext}(\mathcal{R})$ and thus $X^*\cap \text{ext}(\mathcal{R}) \neq \emptyset$.
    
    \paragraph{Part 2.} For the next part of the proof, we show that there exists an optimal solution that is a threshold function. Without loss of generality, assume that $(v_1, \hdots, v_M)$ are sorted in ascending order, such that $v_1 \leq v_2 \leq \hdots \leq v_M$. 
    Let $\bs{x}^* \in X^*\cap \text{ext}(\mathcal{R})$, and assume by contradiction there exists $i, j \in [M], i < j$, such that $x_i^* = u_i, x_j^* = l_j$ and $v_i<v_j$. Denote
    \begin{align*}
        \epsilon_i 
        &= 
        \min \brk[c]*{x_i^* - \log \brk*{ \exp \brk*{x_i^*} + \exp \brk*{x_j^*} - \exp\brk*{u_j}}, x_i^* - l_i} \\
        \epsilon_j 
        &=
        \log \brk*{\exp \brk*{x_i^*} + \exp \brk*{x_j^*} - \exp \brk*{x_i^*-\epsilon_i}} - x_j^*
    \end{align*}
    and let $\tilde{\bs{x}} = \bs{x}^* - \epsilon_i \bs{e}_i + \epsilon_j \bs{e}_j$. Then, $\epsilon_i, \epsilon_j$ enjoy the folllowing properties.
    \begin{enumerate}
        \item $\epsilon_i, \epsilon_j > 0$, since
        \begin{align*}
            &\epsilon_i 
            \ge 
            x_i^* - \log \brk*{ \exp \brk*{x_i^*} + \underbrace{\exp \brk*{x_j^*} - \exp\brk*{u_j}}_{>0}} 
            > 0, \\
            &\epsilon_j 
            =
            \log \brk*{\exp \brk*{x_i^*} + \underbrace{\exp \brk*{x_j^*} - \exp \brk*{x_i^*-\epsilon_i}}_{>0}} - x_j^*
            > 0.
        \end{align*}
        \item By definition, $\epsilon_i\le x_i^*-l_i$ by definition, and $\epsilon_j \le u_j-x_j^*$, since by substituting $\epsilon_i$, we get
        \begin{align*}
            \epsilon_j 
            &\le 
            \log \brk*{\exp \brk*{x_i^*} + \exp \brk*{x_j^*} - \exp \brk*{x_i^*-\brk[s]*{x_i^* - \log \brk*{ \exp \brk*{x_i^*} + \exp \brk*{x_j^*} - \exp\brk*{u_j}}}}} - x_j^* \\
            & = \log \brk*{\exp \brk*{x_i^*} +  \exp \brk*{x_j^*} - \brk[s]{\exp \brk*{x_i^*} + \exp \brk*{x_j^*} - \exp\brk*{u_j}}} - x_j^* \\
            & = u_j-x_j^*.
        \end{align*}
        In particular, given that $\epsilon_i,\epsilon_j>0$, it implies that $l_i\le x_i^*-\epsilon_i \le u_i$ and $l_j\le x_j^*+\epsilon_j \le u_j$.
        \item The total weight of $i,j$ is preserved
        \begin{align*}
            \exp \brk*{x_i^*-\epsilon_i} + \exp \brk*{x_j^*+\epsilon_j} 
            & = 
            \exp \brk*{x_i^*-\epsilon_i}  + \brk[s]*{\exp \brk*{x_i^*} + \exp \brk*{x_j^*} - \exp \brk*{x_i^*-\epsilon_i}} \\
            & = 
            \exp \brk*{x_i^*} + \exp \brk*{x_j^*}
        \end{align*}
    \end{enumerate}
    Given these properties, $\tilde{\bs{x}}$ is a valid solution for which we have that
    \begin{align*}
        f(\tilde{\bs{x}})
        &=
        \frac{\exp \brk*{\tilde x_i}v_i }
        {1+\sum_{k=0}^M\exp \brk*{\tilde x_k}}
        +
        \frac{\exp \brk*{\tilde x_j}v_j }
        {1+\sum_{k=0}^M\exp \brk*{\tilde x_k}}
        +
        \frac{\sum_{k \neq i,j} \exp \brk*{\tilde x_k} v_k}
        {1+\sum_{k=1}^M \exp \brk*{\tilde x_k}} \\
        &=
        \frac{\exp \brk*{\tilde x_i}v_i }
        {1+\sum_{k=0}^M\exp \brk*{x_k^*}}
        +
        \frac{\exp \brk*{\tilde x_j}v_j }
        {1+\sum_{k=0}^M\exp \brk*{x_k^*}}
        +
        \frac{\sum_{k \neq i,j} \exp \brk*{x_k^*} v_k}
        {1+\sum_{k=1}^M \exp \brk*{x_k^*}}. \tag{By property $(3)$}
    \end{align*}
    Therefore,
    \begin{align*}
        f(\tilde{\bs{x}})
        -
        f(\bs{x}^*)
        =
        \frac{\exp \brk*{x_i^*-\epsilon_i}v_i
        + 
        \exp \brk*{x_j^*+\epsilon_j}v_j
        -
        \exp \brk*{x_i^*}v_i
        -
        \exp \brk*{x_j^*}v_j }
        {1+\sum_{k=0}^M\exp \brk*{x_k^*}}.
    \end{align*}
    Considering the numerator, we have that
    \begin{align*}
        \exp &\brk*{x_i^*-\epsilon_i}v_i
        + 
        \exp \brk*{x_j^*+\epsilon_j}v_j
        -
        \exp \brk*{x_i^*}v_i
        -
        \exp \brk*{x_j^*}v_j \\
        &=
        \exp \brk*{x_i^*-\epsilon_i}v_i
        + 
        \brk*{\exp \brk*{x_i^*}
        +
        \exp \brk*{x_j^*}
        -
        \exp \brk*{x_i^* - \epsilon_i}}v_j
        -
        \exp \brk*{x_i^*}v_i
        -
        \exp \brk*{x_j^*}v_j \tag{By property $(3)$} \\
        &=
        \exp \brk*{x_i^*-\epsilon_i}v_i
        + 
        \exp \brk*{x_i^*}v_j
        -
        \exp \brk*{x_i^* - \epsilon_i}v_j
        -
        \exp \brk*{x_i^*}v_i\\
        &=
        \brk*{
        \exp \brk*{x_i^*}
        -
        \exp \brk*{x_i^*-\epsilon_i}
        }
        \brk*{v_j - v_i} \\
        &>
        0,
    \end{align*}
     where the inequality is since $\epsilon_i>0$ and $v_i<v_j$. 
     That is, $f(\tilde{\bs{x}}) > f(\bs{x}^*)$, in contradiction to $\bs{x} \in X^*$. To summarize, we prove that for any $\bs{x}^* \in X^* \cap \text{ext}(\mathcal{R})$, if $v_i<v_j$, then $x_i^*\le x_j^*$, which corresponds to a thresholding function. All that is left is to prove that if $v_i=v_j=v$, there exists a solution $\bs{x}^* \in X^* \cap \text{ext}(\mathcal{R})$ such that either $x_i^*=u_i,x_j^*=u_j$ or $x_i^*=l_i,x_j^*=l_j$. To show this, we follow a similar path to the first part of the proof and write 
     \begin{align*}
        f(\bs{x}^*) 
        &= 
        \sum_{k=0}^M z_i(\bs{x}^*)v_i \\
        &=
        \frac{\exp \brk*{x_i^*}v_i + \exp \brk*{x_j^*}v_j }
        {\sum_{k=0}^M\exp \brk*{x_k^*}}
        +
        \brk*{1 - \frac{\exp \brk*{x_i^*} + \exp \brk*{x_j^*}}
        {\sum_{k=0}^M\exp \brk*{x_k^*}}}
        \frac{\sum_{k \neq i,j} \exp \brk*{x_k^*} v_k}
        {\sum_{k \neq i,j} \exp \brk*{x_k^*}} \\
        & = 
        \frac{\exp \brk*{x_i^*} + \exp \brk*{x_j^*}}
        {\sum_{k=0}^M\exp \brk*{x_k^*}}v
        +
        \brk*{1 - \frac{\exp \brk*{x_i^*} + \exp \brk*{x_j^*}}
        {\sum_{k=0}^M\exp \brk*{x_k^*}}}
        \frac{\sum_{k \neq i,j} \exp \brk*{x_k^*} v_k}
        {\sum_{k \neq i,j} \exp \brk*{x_k^*}}. \tag{$v_i=v_j=v$} 
        .
    \end{align*}
    Now, denoting  
    $\lambda(x_i^*,x_j^*) = \frac{\exp \brk*{x_i^*} + \exp \brk*{x_j^*}}{\sum_{k=0}^M\exp \brk*{x_k^*}}$, 
    and $v_{\text{ref}} = \frac{\sum_{k \neq i,j} \exp \brk*{x_k^*} v_k}{\sum_{k \neq i,j} \exp \brk*{x_k^*}}$, 
    we can follow the exact same line of the proof as the first part, and conclude that there exist another solution $\tilde{\bs{x}}$. in which $\lambda(x_i^*,x_j^*)$ is either maximized or minimized -- either $x_i^*=u_i,x_j^*=u_j$ or $x_i^*=l_i,x_j^*=l_j$.

     This completes the proof.
    
\end{proof}

\newpage
\section{Confidence Sets}
\label{appendix: confidence sets}

We require the following quantities, used by \citet{amani2021ucb} or adapted from \citet{abeille2021instance}. First, recall that for any $\paramvec\in\R^{M(M+1)SAH}$ and $\bs{d}\in\R^{MSAH}$, we have
\begin{align*}
    \bs{A}(\bs{d}, \paramvec) := \text{diag}(\bs{z}(\bs{d}, \paramvec)) - \bs{z}(\bs{d}, \paramvec)\bs{z}(\bs{d}, \paramvec)^T
\end{align*}

Also, for any $\paramvec\in\R^{M(M+1)SAH}$, define
\begin{align*} 
\bs{g}_k(\paramvec) 
:= \lambda\paramvec + \sum_{k'=1}^{k-1}\sum_{h=1}^H z_i(\bs{d}_h^{k'}, \paramvec) \otimes \bs{d}_h^{k'},
\qquad \text{and}\qquad  
\bs{H}_k(\paramvec) 
:= \lambda \bs{I} + \sum_{k'=1}^{k-1}\sum_{h=1}^H \bs{A}(\bs{d}_h^{k'}, \paramvec)\otimes \bs{d}_h^{k'} {\bs{d}_h^{k'}}^T.
\end{align*}
Note that by definition
\begin{align}
\label{eq: derivative def}
    \nabla_{\paramvec}\Lcal_\lambda^k(\paramvec) =\sum_{k'=1}^{k-1}\sum_{h=1}^H \bs{m}_h^{k'}\otimes \bs{d}_h^{k'}  - \bs{g}_k(\paramvec)
    \qquad \text{and}\qquad  
    \nabla_{\paramvec}^2\Lcal_\lambda^k(\paramvec) =-\bs{H}_k(\paramvec)
\end{align}

\paragraph{Confidence Set}
\begin{align}\label{eq: c confidence set}
    \C_k(\delta)
    :=
    \brk[c]*{\paramvec \in \paramset : \norm{g_k(\paramvec) - g_k(\hat \paramvec_t)}_{\bs{H}_k^{-1}(\paramvec)} \leq \beta_k(\delta)},
\end{align}
    where $\beta_k(\delta) = \frac{M^{3/2}(M+1)SAH}{\sqrt{\lambda}}\brk*{\log\brk*{1 + \frac{k}{(M+1)SA\lambda}} + 2\log\brk*{\frac{2}{\delta}}} + \sqrt{\frac{\lambda}{4M}} + \sqrt{\lambda}L$. 
\paragraph{Other Notations} For any $\paramvec_1,\paramvec_2\in\R^{M(M+1)SAH}$ and $\bs{d}\in\R^{MSAH}$, define 
\begin{align*} 
& \bs{B}(\bs{d},\paramvec_1,\paramvec_2) 
:=\int_0^1 \bs{A}\brk*{\bs{d},v\paramvec_1 +(1-v)\paramvec_2}dv, \\
& \tilde{\bs{B}}(\bs{d},\paramvec_1,\paramvec_2) 
:=\int_0^1 (1-v) \bs{A}\brk*{\bs{d},v\paramvec_1 +(1-v)\paramvec_2}dv, \\
& \bs{G}_k(\paramvec_1,\paramvec_2) 
:=\lambda \bs{I} + \sum_{k'=1}^{k-1}\sum_{h=1}^H \bs{B}(\bs{d}_h^k, \paramvec_1,\paramvec_2)\otimes \bs{d}_h^k {\bs{d}_h^k}^T, \\
& \tilde{\bs{G}}_k(\paramvec_1,\paramvec_2) 
:=\lambda \bs{I} + \sum_{k'=1}^{k-1}\sum_{h=1}^H \tilde{\bs{B}}(\bs{d}_h^k, \paramvec_1,\paramvec_2)\otimes \bs{d}_h^k {\bs{d}_h^k}^T, \\
& \bs{V}_k := \lambda \bs{I}_{(M+1)SAH} + \sum_{k'=1}^{k-1}\sum_{h=1}^H\bs{d}_h^{k'} {\bs{d}_h^{k'}}^T.
\end{align*}
Note that we stray from the notation of $\bs{V}_k$ in \citep{amani2021ucb}, by removing the factor of $\kappamin$~from the regularization term.

\subsection{Useful Lemmas}
We now provide a list of lemmas required for providing confidence intervals for the logistic history dependent transition model in \Cref{appendix: regret analysis ldc-ucb,appendix: regret analysis tractable algorithm}. 
In what follows we will use the following expression: 
\begin{align}\label{eq: d2 bound}
d_2(\bs{d}, \paramvec_1,\paramvec_2) \triangleq \norm{\brk*{\paramvec_{1} - \paramvec_{2}}^T \bs{d}}_2 \leq  \norm{\paramvec_{1} - \paramvec_{2}}_2 \norm{\bs{d}}_2 \leq L,
\end{align}
To this end, the following properties hold:
\begin{lemma}[\citet{amani2021ucb}, Lemma 2]
\label{lemma:prop 1}
For any $\paramvec_1,\paramvec_2\in\R^{M(M+1)SAH}$ and $\bs{d}\in\R^{MSAH}$
\begin{align*}
    z(\bs{d}, \paramvec_1) - z(\bs{d}, \paramvec_2) = \brk[s]*{\bs{B}(\bs{d},\paramvec_1,\paramvec_2)\otimes \bs{d}} (\paramvec_1-\paramvec_2)
\end{align*}
\end{lemma}

\begin{lemma}[\citet{amani2021ucb}, Lemma 3]
\label{lemma:prop 2}
For any $\paramvec_1,\paramvec_2\in\R^{M(M+1)SAH}$,
\begin{align*}
    \bs{g}_k(\paramvec_1) - \bs{g}_k(\paramvec_2)  = \bs{G}_k(\paramvec_1,\paramvec_2) (\paramvec_1-\paramvec_2)
\end{align*}
\end{lemma}

\begin{lemma}[\citet{amani2021ucb}, Lemma 4]
\label{lemma:prop 3}
For any $\paramvec_1,\paramvec_2\in\F$, it holds that $(1+2L)^{-1}\bs{H}_k(\paramvec_1)\preceq \bs{G}_k(\paramvec_1,\paramvec_2)$ and $(1+2L)^{-1}\bs{H}_k(\paramvec_2)\preceq \bs{G}_k(\paramvec_1,\paramvec_2)$.
\end{lemma}
     
\begin{lemma}[\citet{amani2021ucb}, Lemma 5]
\label{lemma:prop 4}
For any $\paramvec\in\R^{M(M+1)SAH}$ and $\bs{d}\in\R^{MSAH}$, the matrix $\bs{A}(\bs{d}, \paramvec)$ is strictly diagonally dominant and thus positive definite.
\end{lemma}

\begin{lemma}[\citet{amani2021ucb}, Theorem 1]
\label{lemma: mnl confidence set}
Let $\delta\in(0,1)$. With probability at least $1-\delta$, for all $k\ge1$, it holds that $\paramvectrue\in \C_k(\delta)$.
\end{lemma}
\begin{remark}
\label{remark:positive definiteness}
Notice that \Cref{lemma:prop 4} also implies that all matrices $\bs{B}(\bs{d},\paramvec_1,\paramvec_2), \tilde{\bs{B}}(\bs{d},\paramvec_1,\paramvec_2), \bs{G}_k(\paramvec_1,\paramvec_2), \tilde{\bs{G}}_k(\paramvec_1,\paramvec_2)$ are positive definite. More over, it implies that $\bs{B}(\bs{d},\paramvec_1,\paramvec_2)\succeq \tilde{\bs{B}}(\bs{d},\paramvec_1,\paramvec_2)$ (since $1-v\in[0,1]$), and therefore, $\bs{G}_k(\paramvec_1,\paramvec_2)\succeq \tilde{\bs{G}}_k(\paramvec_1,\paramvec_2)$.
\end{remark}


Relying on our different definition for $V_k$, the next lemma allows us to gain dependence on the derivative at the \emph{real} latent features $\paramvectrue$, instead of the worst-case derivative as in \cite{amani2021ucb}.

\begin{lemma}[Connection between Local and Global Design Matrices]\label{lemma: local to global design matrix}
For any $\paramvec\in\F$, denote $\kappamin(\paramvec)=\frac{1}{\inf_{\bs{d} \in \mathcal{D}} \lambda_{\min}\brk[c]*{\bs{A}(\bs{d} ; \paramvec)}}$. It holds that
\begin{align*}
    \kappamin(\paramvec)\bs{H}_k(\paramvec) \succeq \bs{I{_M}}\otimes\bs{V}_k\enspace,
\end{align*}
and specifically for $\paramfunctrue\in\F$,
\begin{align*}
    \kappamin\bs{H}_k(\paramvectrue) \succeq \bs{I{_M}}\otimes\bs{V}_k\enspace.
\end{align*}
\end{lemma}
\begin{proof}
For clarity, we state the dimension of the identity matrices throughout the proof. Throughout the analysis, recall that if $\bs{A},\bs{B},\bs{C}\succeq0$ and $\bs{A}\succeq \bs{B}$ then $\bs{A}\otimes \bs{C}\succeq \bs{B}\otimes \bs{C}$. For any $\paramvec\in\paramset$,
\begin{align*}
    \bs{H}_k(\paramvec) & = \lambda \bs{I}_{M(M+1)SAH} + \sum_{k'=1}^{k-1}\sum_{h=1}^H \bs{A}(\bs{d}_h^{k'}, \paramvec)\otimes \bs{d}_h^{k'} {\bs{d}_h^{k'}}^T \\
    &\succeq \lambda \bs{I}_{M(M+1)SAH} + \sum_{k'=1}^{k-1}\sum_{h=1}^H\lambda_{\min}\brk[c]*{\bs{A}(\bs{d}_h^{k'} ; \paramvec)} \bs{I}_{M+1}\otimes \bs{d}_h^{k'} {\bs{d}_h^{k'}}^T\\
    & \succeq \lambda \bs{I}_{M(M+1)SAH} + \brk*{\inf_{\bs{d} \in \mathcal{D}} \lambda_{\min}\brk[c]*{\bs{A}(\bs{d} ; \paramvec)}}\sum_{k'=1}^{k-1}\sum_{h=1}^H\bs{I}_{M+1}\otimes \bs{d}_h^{k'} {\bs{d}_h^{k'}}^T\\
    &= \lambda \bs{I}_{M(M+1)SAH} + \frac{1}{\kappamin(\paramvec)}\sum_{k'=1}^{k-1}\sum_{h=1}^H\bs{I}_{M+1}\otimes \bs{d}_h^{k'} {\bs{d}_h^{k'}}^T\\
    &= \frac{1}{\kappamin(\paramvec)}\bs{I}_{M}\otimes\brk*{\kappamin(\paramvec)\lambda \bs{I}_{(M+1)SAH} + \sum_{k'=1}^{k-1}\sum_{h=1}^H\bs{d}_h^{k'} {\bs{d}_h^{k'}}^T}\\
    &\overset{(*)}{\succeq} \frac{1}{\kappamin(\paramvec)}\bs{I}_{M}\otimes\brk*{\lambda \bs{I}_{(M+1)SAH} + \sum_{k'=1}^{k-1}\sum_{h=1}^H\bs{d}_h^{k'} {\bs{d}_h^{k'}}^T}\\ 
    & 
    = \frac{1}{\kappamin(\paramvec)}\bs{I{_M}}\otimes\bs{V}_k\enspace ,
\end{align*}
where $(*)$ holds since $\kappamin(\paramvec)\ge1$ (see, e.g., eq. (29) of \citet{amani2021ucb} when fixing the set of possible parameters $\paramset$ to be a singleton $\paramset=\brk[c]*{\paramvec}$).

Finally, we conclude the proof by noting that for $\paramvectrue\in\F$, it holds by definition that $\kappamin = \kappamin(\paramvectrue)$. 
\end{proof}

\begin{lemma} \label{lemma: multinomial error norm}
For all $\bs{d}\in\R^{MSAH}$ such that $\norm{\bs{d}}_2\le 1$, $k\ge1$ and $\paramvec\in\C_k(\delta)$, if $\paramvectrue\in\C_k(\delta)$ then
\begin{align*}
    \norm{z(\bs{d}, \paramvec) - z(\bs{d}, \paramvectrue)}_2 
    \leq 2\beta_k(\delta) (1+2L) \sqrt{\kappamin}  \norm{\bs{d}}_{\bs{V}_k^{-1}} 
\end{align*}
\end{lemma}
\begin{proof}
Here, we closely follow the proof of Lemma 1 in \cite{amani2021ucb}, with the exception that we apply \Cref{lemma: local to global design matrix} to achieve dependence on $\kappamin$. Specifically, we let $\kappamax := \sup_{\bs{d}t \in \mathcal{D}, \paramvec \in \paramset} \lambda_{\text{max}}\brk[c]*{\bs{A}(\bs{d} ; \paramvec)}$. Notice that following \citep[][Section 3]{amani2021ucb}, it holds that $\kappamax\leq 1$.
\begin{align*}
&\norm{z(\bs{d}, \paramvec) - z(\bs{d}, \paramvectrue)}_2 \\
& = \norm{ \brk[s]*{\bs{B}(\bs{d},\paramvectrue,\paramvec)\otimes \bs{d}}(\paramvectrue-\paramvec)} \tag{\Cref{lemma:prop 1}} \\
& = \norm{ \brk[s]*{\bs{B}(\bs{d},\paramvectrue,\paramvec)\otimes \bs{d}}\bs{G}_k^{-1/2}(\paramvectrue, \paramvec)\bs{G}_k^{1/2}(\paramvectrue,\paramvec)(\paramvectrue-\paramvec)}  \\ 
& \leq \norm{ \brk[s]*{\bs{B}(\bs{d},\paramvectrue,\paramvec)\otimes \bs{d}}\bs{G}_k^{-1/2}(\paramvectrue, \paramvec)} \norm{\paramvectrue-\paramvec}_{\bs{G}_k(\paramvectrue,\paramvec)} \tag{Cauchy-Schwartz} \\
& = \norm{ \brk[s]*{\bs{B}(\bs{d},\paramvectrue,\paramvec)\otimes \bs{d}}\bs{G}_k^{-1/2}(\paramvectrue, \paramvec)} \norm{g_k(\paramvectrue)-g_k(\paramvec)}_{\bs{G}_k^{-1}(\paramvectrue,\paramvec)}  \tag{\Cref{lemma:prop 2}} \\
& = \sqrt{\lambdamax \brk*{ \brk[s]*{\bs{B}(\bs{d},\paramvectrue,\paramvec)\otimes \bs{d}}\bs{G}_k^{-1}(\paramvectrue, \paramvec) \brk[s]*{\bs{B}(\bs{d},\paramvectrue,\paramvec)\otimes \bs{d}}}} \norm{g_k(\paramvectrue)-g_k(\paramvec)}_{\bs{G}_k^{-1}(\paramvectrue,\paramvec)}  \\
& = \sqrt{\lambdamax \brk*{ \bs{G}_k^{-1/2}(\paramvectrue,\paramvec) \brk[s]*{\bs{B}^T(\bs{d},\paramvectrue,\paramvec)\otimes \bs{d} } \brk[s]*{\bs{B}(\bs{d},\paramvectrue,\paramvec) \otimes \bs{d}^T} \bs{G}_k^{-1/2}(\paramvectrue,\paramvec) }} \norm{ g_k(\paramvectrue)-g_k(\paramvec)}_{\bs{G}_k^{-1}(\paramvectrue,\paramvec)}  \tag{cyclic property of $\lambdamax$} \\
&= \sqrt{\lambdamax \brk*{ \bs{G}_k^{-1/2}(\paramvectrue,\paramvec) \brk[s]*{\bs{B}^T(\bs{d},\paramvectrue,\paramvec)\bs{B}(\bs{d},\paramvectrue,\paramvec) \otimes \bs{d} \bs{d}^T} \bs{G}_k^{-1/2}(\paramvectrue,\paramvec) }} \norm{ g_k(\paramvectrue)-g_k(\paramvec)}_{\bs{G}_k^{-1}(\paramvectrue,\paramvec)}  \tag{mixed-product property} \\
& \leq \kappamax \sqrt{\lambdamax \brk*{ \bs{G}_k^{-1}(\paramvectrue,\paramvec) \brk[s]*{\bs{I_M} \otimes \bs{d} \bs{d}^T} }} \norm{ g_k(\paramvectrue)-g_k(\paramvec)}_{\bs{G}_k^{-1}(\paramvectrue,\paramvec)} \tag{definition of $\kappamax$} \\
& =  \sqrt{\lambdamax \brk*{ \bs{G}_k^{-1}(\paramvectrue,\paramvec) \brk[s]*{\bs{I_M} \otimes \bs{d}} \brk[s]*{\bs{I_M} \otimes \bs{d}^T} }} \norm{ g_k(\paramvectrue)-g_k(\paramvec)}_{\bs{G}_k^{-1}(\paramvectrue,\paramvec)} \tag{cyclic property of $\lambdamax$, and $\kappamax\leq 1$} \\
& =  \sqrt{\lambdamax \brk*{  \brk[s]*{\bs{I_M} \otimes \bs{d}} \bs{G}_k^{-1}(\paramvectrue,\paramvec) \brk[s]*{\bs{I_M} \otimes \bs{d}^T} }} \norm{ g_k(\paramvectrue)-g_k(\paramvec)}_{\bs{G}_k^{-1}(\paramvectrue,\paramvec)} \tag{mixed-product property} \\
& \leq  \sqrt{1+2L}\sqrt{\lambdamax \brk*{  \brk[s]*{\bs{I_M} \otimes \bs{d}}  \bs{H}_k^{-1}(\paramvectrue) \brk[s]*{\bs{I_M} \otimes \bs{d}^T} }} \norm{ g_k(\paramvectrue)-g_k(\paramvec)}_{\bs{G}_k^{-1}(\paramvectrue,\paramvec)} \tag{\Cref{lemma:prop 3}} \\
& \leq   \sqrt{\kappamin(1+2L)} \sqrt{\lambdamax \brk*{  \brk[s]*{\bs{I_M} \otimes \bs{d}}  \brk[s]*{\bs{I_M} \otimes \bs{V}_k^{-1} (\paramvectrue)} \brk[s]*{\bs{I_M} \otimes \bs{d}^T} }} \norm{ g_k(\paramvectrue)-g_k(\paramvec)}_{\bs{G}_k^{-1}(\paramvectrue,\paramvec)} \tag{\Cref{lemma: local to global design matrix}} \\
& =   \sqrt{\kappamin(1+2L)} \sqrt{\lambdamax \brk*{  \bs{I_M} \otimes \norm{\bs{d}}_{\bs{V}_k^{-1}}^2}} \norm{ g_k(\paramvectrue)-g_k(\paramvec)}_{\bs{G}_k^{-1}(\paramvectrue,\paramvec)} \tag{mixed-product property} \\
& =  \sqrt{\kappamin(1+2L)} \norm{\bs{d}}_{\bs{V}_k^{-1}} \norm{ g_k(\paramvectrue)-g_k(\paramvec)}_{\bs{G}_k^{-1}(\paramvectrue,\paramvec)} \\
& =  \sqrt{\kappamin(1+2L)} \norm{\bs{d}}_{\bs{V}_k^{-1}} \norm{ g_k(\paramvectrue)-g_k(\hat \paramvec) + g_k(\hat \paramvec) - g_k(\paramvec)}_{\bs{G}_k^{-1}(\paramvectrue,\paramvec)} \\
& =  \sqrt{\kappamin(1+2L)}  \norm{\bs{d}}_{\bs{V}_k^{-1}} \brk[s]*{\norm{ g_k(\paramvectrue)-g_k(\hat \paramvec)}_{\bs{G}_k^{-1}(\paramvectrue,\paramvec)} + \norm{g_k(\hat \paramvec) - g_k(\paramvec)}_{\bs{G}_k^{-1}(\paramvectrue,\paramvec)}} \\
& \leq   (1+2L) \sqrt{\kappamin}  \norm{\bs{d}}_{\bs{V}_k^{-1}} \brk[s]*{\norm{ g_k(\paramvectrue)-g_k(\hat \paramvec)}_{\bs{H}_k^{-1}(\paramvectrue)} + \norm{g_k(\hat \paramvec) - g_k(\paramvec)}_{\bs{H}_k^{-1}(\paramvectrue)}} \tag{\Cref{lemma:prop 3}}\\
&\leq  2\beta_k(\delta) (1+2L) \sqrt{\kappamin}  \norm{\bs{d}}_{\bs{V}_k^{-1}} \tag{$\paramvec,\paramvectrue\in \C_k(\delta)$}.
\end{align*}
The \emph{cyclic property} of $\lambdamax$ refers to the fact that for two matrices $\bs{M}_1,\bs{M}_2$, the eigenvalues of $\bs{M}_1\bs{M}_2$ are the same as the eigenvalues of $\bs{M}_2\bs{M}_1$, and thus the same hold for the maximal eigenvalue.

\end{proof}

\begin{lemma}[Adaptation of \cite{abeille2021instance}, Lemma 8, to the multinomial case in \citet{amani2021ucb}, Lemma 13]\label{lemma: tilde G bound}
For any $\paramvec_1,\paramvec_2\in\F$, it holds that $(2+2L)^{-1}\bs{H}_k(\paramvec_1)\preceq \tilde{\bs{G}}_k(\paramvec_1,\paramvec_2)$ and $(2+2L)^{-1}\bs{H}_k(\paramvec_2)\preceq \tilde{\bs{G}}_k(\paramvec_1,\paramvec_2)$.
\end{lemma}
\begin{proof}

According to \citet{sun2019generalized}[Eq. 16], for any ${\bs d} \in \R^{MSAH}$, $\paramvec_1,\paramvec_2\in\R^{M(M+1)SAH}$, and for any $v\in[0,1]$, we have that $ \nabla^2 f(v x + (1-v) y) \succeq e^{-v d_2(\bs{d}, \paramvec_1,\paramvec_2)} \nabla^2 f(y) $, where $d_2(\bs{d}, \paramvec_1,\paramvec_2)$ is defined in \Cref{eq: d2 bound}.

Thus,
\begin{align*}
    \int_0^1 (1-v) \nabla^2 f(v x + (1-v) y) dv \succeq \nabla^2 f(y) \int_0^1 (1-v) e^{-v d_2(\bs{d}, \paramvec_1,\paramvec_2)}dv   ,
\end{align*}
and replacing with the notation $\nabla^2 f({\bs x}) \triangleq {\bs A}({\bs d, \bs x})$, we get
\begin{align*}
    \tilde{\bs{B}}(\bs{d},\paramvec_1,\paramvec_2) \triangleq \int_0^1 (1-v) \bs{A}(\bs{d}, v \paramvec_1 + (1-v) \paramvec_2) dv \succeq \bs{A}(\bs{d},  \paramvec_2) \int_0^1 (1-v) e^{-v d_2(\bs{d}, \paramvec_1,\paramvec_2)}dv   ,
\end{align*}
Integrating the RHS by parts,
\begin{align*}
    \tilde{\bs{B}}(\bs{d},\paramvec_1,\paramvec_2) & \succeq  \brk*{\frac{1}{d_2(\bs{d}, \paramvec_1,\paramvec_2)} + \frac{e^{-d_2(\bs{d}, \paramvec_1,\paramvec_2)} - 1}{\brk*{d_2(\bs{d}, \paramvec_1,\paramvec_2)}^2}}\nabla^2 f(y) = g(d_2(\bs{d}, \paramvec_1,\paramvec_2))\bs{A}(\bs{d},  \paramvec_2),
\end{align*}
where we defined $g(z)\triangleq \frac{1}{z}\brk*{1+\frac{e^{-z}-1}{z}}$.

Next, by \citet{abeille2021instance}[Lemma 10], for all $z\geq 0$, it holds that $g(z) \geq \frac{1}{2+z}$, and therefore,
\begin{align}\label{eq: bound on tilde B}
    \tilde{\bs{B}}(\bs{d},\paramvec_1,\paramvec_2)  & \succeq  \brk*{2+d_2(\bs{d}, \paramvec_1,\paramvec_2)}^{-1} \bs{A}(\bs{d},  \paramvec_2))  \succeq \brk*{2+2L}^{-1}\bs{A}(\bs{d},  \paramvec_2),
\end{align}
where the last inequality follows is due to \Cref{eq: d2 bound}.

Plugging in \Cref{eq: bound on tilde B} with the definition of $\tilde{\bs{G}}_k(\paramvec_1,\paramvec_2) $,
\begin{align*}
\tilde{\bs{G}}_k(\paramvec_1,\paramvec_2) &
=\lambda \bs{I} + \sum_{k'=1}^{k-1}\sum_{h=1}^H \tilde{\bs{B}}(\bs{d}_h^k, \paramvec_1,\paramvec_2)\otimes \bs{d}_h^k {\bs{d}_h^k}^T \\
& \succeq \brk*{2+2L}^{-1} \brk*{\lambda \bs{I} + \sum_{k'=1}^{k-1}\sum_{h=1}^H \bs{A}(\bs{d}_h^{k'},  \paramvec_2)\otimes \bs{d}_h^{k'} {\bs{d}_h^{k'}}^T } \\
& = \brk*{2+2L}^{-1} H_k( \paramvec_2).
\end{align*}

By the symmetry in $\paramvec_1,\paramvec_2$ in the definition of $\tilde B(\bs{d}, \paramvec_1,\paramvec_2)$, we can similarly prove that $\brk*{2+2L}^{-1} H_k( \paramvec_1) \succeq \tilde{\bs{G}}_k(\paramvec_1,\paramvec_2)$

\end{proof}

\subsection{Convex Relaxation}
\label{appendix: convex relaxation}

Similar to \citet{abeille2021instance}, we define the convex relaxation of the set $\C_k(\delta)$ by
\begin{align}\label{eq: epsilon confidence set}
    \Ecal_k(\delta) = \brk[c]*{\paramvec\in\F : \Lcal_\lambda^k(\hat{\paramvec}_k)-\Lcal_\lambda^k(\paramvec) \le \xi^2(\delta)} \qquad \text{where}\enspace \xi(\delta) = \beta_k(\delta)+\sqrt{\frac{HM}{\lambda}}\beta_k^2(\delta).
\end{align}

The next proposition is an adaptation of \cite{abeille2021instance}[Lemma 1] to the multinomial setting of \cite{amani2021ucb}. Importantly, this proposition provides a confidence interval for the convex relaxation set \Cref{eq: epsilon confidence set}, which serves as the basis for the tractable estimator in \Cref{section: tractable estimator}.
\begin{proposition}
\label{prop:convex relaxation} Let $\delta\in(0,1).$

\begin{enumerate}
    \item $\C_k(\delta)\subseteq \Ecal_k(\delta)$ for all $k\ge1$ and therefore, w.p. $1-\delta$, $\paramvectrue\in\Ecal_k(\delta)$  for all $k\ge1$.
    \item With probability $1-\delta$, it holds that 
    \begin{align*}
        \forall \paramvec\in\Ecal_k(\delta),\quad \norm{\paramvec-\paramvectrue}_{{\bs{H}}_k(\paramvectrue)} \le (2+2L)\beta_k(\delta)+\sqrt{2(1+L)}\xi_k(\delta).
    \end{align*}
    In particular, for $\bar{\paramvec}\in \arg\max _{\paramvec\in\F}\Lcal_\lambda^k(\paramvec)$, with probability $1-\delta$,
    \begin{align*}
        \norm{\bar\paramvec-\paramvectrue}_{{\bs{H}}_k(\paramvectrue)} \le (2+2L)\beta_k(\delta)+\sqrt{2(1+L)}\xi_k(\delta) \triangleq \gamma_k(\delta),
    \end{align*}
    where $\gamma_k(\delta) := \brk*{2+2L + \sqrt{2(1+L)}}\beta_k(\delta) + \sqrt{\frac{2(1+L)HM}{\lambda}}\beta_k^2(\delta)$.
\end{enumerate}
\end{proposition}

As in \citep{abeille2021instance}, in order to prove \Cref{prop:convex relaxation}, we first require the following side-lemma:
\begin{lemma}[Counterpart of \citet{abeille2021instance}, Lemma 2]
\label{lemma: side lemma for relaxation}
Let $\delta\in(0,1)$. For all $\paramvec\in\C_k(\delta)$, it holds that
\begin{align*}
    \norm{\bs{g}_k(\paramvec)-\bs{g}_k(\paramvectrue)}_{{\bs{G}}_k^{-1}(\paramvec,\hat{\paramvec})} \le \xi_k(\delta).
\end{align*}
\end{lemma}

\begin{proof}
First notice that by \Cref{remark:positive definiteness}, the norm w.r.t. the inverse is well-defined and we further have $\lambda_{\min}({\bs{G}}_k(\paramvec_1,\paramvec_2))\ge \lambda$. Next, we utilise \citet{amani2021ucb}, (eq. 61), which states that for any $\paramvec_1,\paramvec_2\in\F$, $\bs{d}\in\D$ (Recall that $\bs{B}$ is symmetric between $\paramvec_1,\paramvec_2$): 

\begin{align*}
    \bs{B}(\bs{d},\paramvec_1,\paramvec_2) 
    &\succeq \brk*{1+ \norm{\brk[s]*{\paramvec_{11}^T\bs{d}-\paramvec_{21}^T\bs{d},\dots,\paramvec_{1M}^T\bs{d}-\paramvec_{2M}^T\bs{d}}}}^{-1}  \bs{A}(\bs{d},\paramvec_1) \tag{\citet{amani2021ucb}, (eq. 61)} \\
    &  \succeq \brk*{1+ \norm{\mathbf{1}_M\otimes\bs{d}}_{{\bs{G}}_k^{-1}(\paramvec_1,\paramvec_2)}\norm{\paramvec_1-\paramvec_2}_{{\bs{G}}_k(\paramvec_1,\paramvec_2)} }^{-1} \bs{A}(\bs{d},\paramvec_1) \tag{Cauchy-Schwartz} \\
    & \succeq  \brk*{1+ \sqrt{\frac{HM}{\lambda}}\norm{\paramvec_1-\paramvec_2}_{{\bs{G}}_k(\paramvec_1,\paramvec_2)} }^{-1} \bs{A}(\bs{d},\paramvec_1) \tag{ $\bs{G}_k(\paramvec_1,\paramvec_2) \succeq \lambda\bs{I}$},
\end{align*}
where $j\in[1,...,M]$, $\paramvec_{ij}$ is the $j$-th coordinate of $\paramvec_i$ and $\mathbf{1}_M\in\R^M$ is a vector of ones.

Thus, we can write for any $\paramvec\in\C_k(\delta)$
\begin{align*}
    \bs{G}_k(\paramvec,\hat{\paramvec}_k) 
    &= \lambda \bs{I} + \sum_{k'=1}^{k-1}\sum_{h=1}^H \bs{B}(\bs{d}_h^k, \paramvec,\hat{\paramvec}_k)\otimes \bs{d}_h^k {\bs{d}_h^k}^T \\
    & \succeq \lambda \bs{I} + \brk*{1+ \sqrt{\frac{HM}{\lambda}}\norm{\paramvec-\hat{\paramvec}_k}_{{\bs{G}}_k(\paramvec,\hat{\paramvec}_k)} }^{-1} \sum_{k'=1}^{k-1}\sum_{h=1}^H \bs{A}(\bs{d},\paramvec)\otimes \bs{d}_h^k {\bs{d}_h^k}^T \tag{$\forall A,B \succ 0 \Rightarrow A\otimes B\succ 0$}
\\
    & \succeq  \brk*{1+ \sqrt{\frac{HM}{\lambda}}\norm{\paramvec-\hat{\paramvec}_k}_{{\bs{G}}_k(\paramvec,\hat{\paramvec}_k)}}^{-1}  \brk*{\lambda \bs{I} + \sum_{k'=1}^{k-1}\sum_{h=1}^H\bs{A}(\bs{d},\paramvec)\otimes \bs{d}_h^k {\bs{d}_h^k}^T} \\
    & = \brk*{1+ \sqrt{\frac{HM}{\lambda}}\norm{\paramvec-\hat{\paramvec}_k}_{{\bs{G}}_k(\paramvec,\hat{\paramvec}_k)}}^{-1}  \bs{H}_k(\paramvec) \\
    & = \brk*{1+ \sqrt{\frac{HM}{\lambda}}\norm{\bs{g}_k(\paramvec)-\bs{g}_k(\hat{\paramvec}_k)}_{{\bs{G}}_k^{-1}(\paramvec,\hat{\paramvec}_k)}}^{-1}  \bs{H}_k(\paramvec) \tag{\Cref{lemma:prop 2}},
\end{align*}

Using this inequality, we get 
\begin{align*}
    \norm{\bs{g}_k(\paramvec)-\bs{g}_k(\hat{\paramvec}_k)}_{{\bs{G}}_k^{-1}(\paramvec,\hat{\paramvec}_k)}^2 
    &\le \brk*{1+ \sqrt{\frac{HM}{\lambda}}\norm{\bs{g}_k(\paramvec)-\bs{g}_k(\hat{\paramvec}_k)}_{{\bs{G}}_k^{-1}(\paramvec,\hat{\paramvec}_k)}}\norm{\bs{g}_k(\paramvec)-\bs{g}_k(\hat{\paramvec}_k)}_{{\bs{H}}_k^{-1}(\paramvec)}^2 \\
    & \le  \brk*{1+ \sqrt{\frac{HM}{\lambda}}\norm{\bs{g}_k(\paramvec)-\bs{g}_k(\hat{\paramvec}_k)}_{{\bs{G}}_k^{-1}(\paramvec,\hat{\paramvec}_k)}}\beta_k^2(\delta) \tag{$\paramvec\in\C_k(\delta)$, see \Cref{eq: c confidence set}}.
\end{align*}
Solving this inequality finally yields the desired result  \citep[see, e.g.][Proposition 7]{abeille2021instance}:
\begin{align*}
    \norm{\bs{g}_k(\paramvec)-\bs{g}_k(\hat{\paramvec}_k)}_{{\bs{G}}_k^{-1}(\paramvec,\hat{\paramvec}_k)}
    \le \beta_k(\delta)+\sqrt{\frac{HM}{\lambda}}\beta_k^2(\delta)
\end{align*}
\end{proof}

We are now ready to prove \Cref{prop:convex relaxation}.
\begin{proof}[Proof of \Cref{prop:convex relaxation}]

\textbf{Part 1. } 
We start by writing the exact second-order Taylor expansion of the likelihood $\Lcal_\lambda^k(\paramvec)$, which holds for any $\paramvec\in\R^{M(M+1)SAH}$
\begin{align*}
    \Lcal_\lambda^k(\paramvec)
    = \Lcal_\lambda^k(\hat{\paramvec}_k)
    + \nabla_{\paramvec}\Lcal_\lambda^k(\hat{\paramvec}_k)^T(\paramvec - \hat{\paramvec}_k)
    + (\paramvec - \hat{\paramvec}_k)^T\brk*{\int_{v=0}^1 (1-v)\nabla_{\paramvec}^2\Lcal_\lambda^k(\hat{\paramvec}_k + v(\paramvec-\hat{\paramvec}_k)) dv}(\paramvec - \hat{\paramvec}_k).
\end{align*}
Since $\hat{\paramvec}_k$ is the solution to the unconstrained minimization of the concave likelihood $\Lcal_\lambda^k(\paramvec)$, we have that $\nabla_{\paramvec}\Lcal_\lambda^k(\hat{\paramvec}_k)=0$. Recalling that $\nabla_{\paramvec}^2\Lcal_\lambda^k(f) = -H_k(f)$, we get
\begin{align*}
    \Lcal_\lambda^k(\paramvec) -  \Lcal_\lambda^k(\hat{\paramvec}_k)
    &=
    \nabla_{\paramvec}\Lcal_\lambda^k(\hat{\paramvec}_k)^T(\paramvec^* - \hat{\paramvec}_k)
    + (\paramvec - \hat{\paramvec}_k)^T\brk*{\int_{v=0}^1 (1-v)\nabla_{\paramvec}^2\Lcal_\lambda^k(\hat{\paramvec}_k + v(\paramvec-\hat{\paramvec}_k)) dv}(\paramvec - \hat{\paramvec}_k) \\
    & = 
    - (\paramvec - \hat{\paramvec}_k)^T\brk*{\int_{v=1}^1 (1-v)\bs{H}_k(\hat{\paramvec}_k + v(\paramvec-\hat{\paramvec}_k)) dv}(\paramvec - \hat{\paramvec}_k) \\
    & =
    - \norm{\paramvec - \hat{\paramvec}_k}_{\tilde{\bs{G}}_k(\hat{\paramvec}_k,\paramvec)}^2 \tag{Def. of $\tilde{\bs{G}}_k(\hat{\paramvec}_k,\paramvec)$} \\
    & \ge - \norm{\paramvec - \hat{\paramvec}_k)}_{\bs{G}_k(\hat{\paramvec}_k,\paramvec)}^2 \tag{$\tilde{\bs{G}}_k \preceq \bs{G}_k$} \\
    & = - \norm{\bs{g}_k(\paramvec) - \bs{g}_k(\hat{\paramvec}_k))}_{\bs{G}_k^{-1}(\hat{\paramvec}_k,\paramvec)}^2 \tag{\Cref{lemma:prop 2}} \\
    & = - \norm{\bs{g}_k(\paramvec) - \bs{g}_k(\hat{\paramvec}_k))}_{\bs{G}_k^{-1}(\paramvec, \hat{\paramvec}_k)}^2 \tag{$\bs{G}_k(\hat{\paramvec}_k,\paramvec) = \bs{G}_k(\paramvec,\hat{\paramvec}_k)$} .
\end{align*}

Rearranging, we get that for any $\paramvec\in\R^{M(M+1)SAH}$, 
\begin{align*}
    \Lcal_\lambda^k(\hat{\paramvec}_k) - \Lcal_\lambda^k(\paramvec) \leq \norm{\bs{g}_k(\paramvec) - \bs{g}_k(\hat{\paramvec}_k))}_{\bs{G}_k^{-1}(\paramvec, \hat{\paramvec}_k)}^2,
\end{align*}
and thus, the above inequality holds for any $\paramvec \in \C_k(\delta) \subseteq \R^{M(M+1)SAH}$.

Finally, by \Cref{lemma: side lemma for relaxation}, for any $\paramvec\in\C_k(\delta)$, we have that 
\begin{align*}
    \Lcal_\lambda^k(\hat{\paramvec}_k) - \Lcal_\lambda^k(\paramvec) \leq \xi_k^2(\delta),
\end{align*}
which implies that $\C_k(\delta)\subseteq \Ecal_k(\delta)$ by the definition of \Cref{eq: epsilon confidence set}. In particular, by \Cref{lemma: mnl confidence set}, $\paramvectrue\in\C_k(\delta)$ with probability at least $1-\delta$, and thus the same holds for $\Ecal_k(\delta)$.

\textbf{Part 2.} For this part, assume that $\paramvectrue\in\C_k(\delta)\subseteq \Ecal_k(\delta)$ for all $k\ge1$, an event that holds with probability at least $1-\delta$. Also, let $\paramvec\in\Ecal_k(\delta)$. Writing the Taylor expansion of $\Lcal_\lambda^k(\paramvec)$, following the same derivation as the last part, we get 
\begin{align*}
    \Lcal_\lambda^k(\paramvec)
    &= \Lcal_\lambda^k(\paramvectrue)
    + \nabla_{\paramvec}\Lcal_\lambda^k(\paramvectrue)^T(\paramvec - \paramvectrue)
    + (\paramvec - \paramvectrue)^T\brk*{\int_{v=1}^1 (1-v)\nabla_{\paramvec}^2\Lcal_\lambda^k(\paramvectrue + v(\paramvec-\paramvectrue)) dv}(\paramvec - \paramvectrue) \\
    & = \Lcal_\lambda^k(\paramvectrue)
    + \nabla_{\paramvec}\Lcal_\lambda^k(\paramvectrue)^T(\paramvec - \paramvectrue)
    - \norm{\paramvec - \paramvectrue}_{\tilde{\bs{G}}_k(\paramvectrue,\paramvec)}^2. \\
    & \le \Lcal_\lambda^k(\paramvectrue)
    + \nabla_{\paramvec}\Lcal_\lambda^k(\paramvectrue)^T(\paramvec - \paramvectrue)
    - (2+2L)^{-1}\norm{\paramvec - \paramvectrue}_{\bs{H}_k(\paramvectrue)}^2.  \tag{\Cref{lemma: tilde G bound}}
\end{align*}

Rearranging this inequality, we get,
\begin{align*}
    \norm{\paramvec - \paramvectrue}_{\bs{H}_k(\paramvectrue)}^2
    &\le (2+2L)\brk*{\Lcal_\lambda^k(\paramvectrue) -\Lcal_\lambda^k(\paramvec)}
    + (2+2L)\nabla_{\paramvec}\Lcal_\lambda^k(\paramvectrue)^T(\paramvec - \paramvectrue) \\
    &\le (2+2L)\brk*{\Lcal_\lambda^k(\hat{\paramvec}_k) -\Lcal_\lambda^k(\paramvec)}
    + (2+2L)\nabla_{\paramvec}\Lcal_\lambda^k(\paramvectrue)^T(\paramvec - \paramvectrue) \tag{Def. of $\hat{\paramvec}_k$} \\
    &\le (2+2L)\xi_k^2(\delta)
    + (2+2L)\nabla_{\paramvec}\Lcal_\lambda^k(\paramvectrue)^T(\paramvec - \paramvectrue) \tag{$\paramvec\in \Ecal_k(\delta)$} \\
    &\le (2+2L)\xi_k^2(\delta)
    + (2+2L)\norm{\paramvec - \paramvectrue}_{\bs{H}_k(\paramvectrue)}\norm{\nabla_{\paramvec}\Lcal_\lambda^k(\paramvectrue)}_{\bs{H}_k^{-1}(\paramvectrue)} \tag{Cauchy-Schwartz} \\
    &\le (2+2L)\xi_k^2(\delta)
    + (2+2L)\beta_k(\delta)\norm{\paramvec - \paramvectrue}_{\bs{H}_k(\paramvectrue)}
\end{align*}
where the last inequality is since 
\begin{align*}
    \norm{\nabla_{\paramvec}\Lcal_\lambda^k(\paramvectrue)}_{\bs{H}_k^{-1}(\paramvectrue)}
    &= \norm{\nabla_{\paramvec}\Lcal_\lambda^k(\paramvectrue) - \underbrace{\nabla_{\paramvec}\Lcal_\lambda^k(\hat{\paramvec}_k)}_{=0}}_{\bs{H}_k^{-1}(\paramvectrue)} \\
    & = \norm{\bs{g}_k(\paramvectrue) - \bs{g}_k(\hat{\paramvec}_k)}_{\bs{H}_k^{-1}(\paramvectrue)} \tag{see \Cref{eq: derivative def}} \\
    & \le \beta_k(\delta). \tag{$\paramvectrue\in\C_k(\delta)$}
\end{align*}
Thus, we have the inequality 
\begin{align*}
    \norm{\paramvec - \paramvectrue}_{\bs{H}_k(\paramvectrue)}^2
    \le (2+2L)\xi_k^2(\delta)
    + (2+2L)\beta_k(\delta)\norm{\paramvec - \paramvectrue}_{\bs{H}_k(\paramvectrue)},
\end{align*}
which implies that  \citep[see, e.g.][Proposition 7]{abeille2021instance}
\begin{align*}
    \norm{\paramvec - \paramvectrue}_{\bs{H}_k(\paramvectrue)}^2
    \le \sqrt{(2+2L)}\xi_k(\delta)
    + (2+2L)\beta_k(\delta).
\end{align*}
To conclude the proof, notice that under the event that $\paramvectrue\in\C_k(\delta)$ for all $k\ge1$, we also have that $\paramvectrue\in\Ecal_k(\delta)$, and therefore, $\Ecal_k(\delta)$ is not empty for all $k\ge1$. Specifically, when the set is nonempty, 
by the definition of the set $\Ecal_k(\delta)$ (\Cref{eq: epsilon confidence set}), there exists a $f\in\F$ for which 
$
\Lcal_\lambda^k(\hat{\paramvec}_k)-\Lcal_\lambda^k(\paramvec) \le \xi_k^2(\delta)$.

Now, by definition, it holds for the constrained maximizer $\bar{\paramvec}_k\in \arg\max _{\paramvec\in\F}\Lcal_\lambda^k(\paramvec)$ that for any $\paramvec\in\F$, $\Lcal_\lambda^k(\bar{\paramvec}) \geq \Lcal_\lambda^k(\paramvec)$. Consequently, for any $\paramvec\in\F$, $\Lcal_\lambda^k(\hat{\paramvec}_k)-\Lcal_\lambda^k(\bar{\paramvec}_k) \leq \Lcal_\lambda^k(\hat{\paramvec}_k) - \Lcal_\lambda^k(\paramvec)$. Thus, when the set is nonempty, it must contain the constrained maximizer $\bar{\paramvec}_k$. A direct conclusion of the previous inequality is that w.p. at least $1-\delta$, for all $k\ge1$,
\begin{align*}
        \norm{\bar\paramvec-\paramvectrue}_{{\bs{H}}_k(\paramvectrue)} \le (2+2L)\beta_k(\delta)+\sqrt{2(1+L)}\xi_k(\delta).
    \end{align*}
\end{proof}

\subsection{Local Confidence Bound}
\label{appendix: local confidence bound}

To prove the local confidence bound, we adapt the proofs of \citep[][Appendix K]{tennenholtz2022reinforcement} to the multinomial case, while also taking into account the discounting of the latent features.



Next, we prove that the inverse of the Gram matrix of each episode is well behaved -- its diagonal is bounded at any visited state-action-context.

\begin{lemma}[Inverse Eigenvalues Bound]\label{lemma: Inverse Eigenvalues bound}
    Let $\bs{D}_k =  \sum_{h=1}^H\bs{d}_h^{k'} {\bs{d}_h^{k'}}^T$ be the Gram matrix that corresponds to the discounted visitations during episode $k$. 
    If $(s,a,x)\in \hist_h^{k}$ and $\bs{e}_{x,s,a,h}\in\R^{(M+1)SAH}$ is a unit vector in the coordinate $(x,s,a,h)$, then $\bs{e}_{x,s,a,h}^T\brk*{\lambda \bs{I} + \bs{D}_{k}}^{-1}\bs{e}_{x,s,a,h}\le \frac{1}{\frac{1}{4\Halpha}+\lambda}$ .
\end{lemma}

\begin{proof}
We closely follow the proof of \citep[][Lemma 7]{tennenholtz2022reinforcement}, while incorporating discount to the visitation vector. For brevity, and with some abuse of notations, we use $\bs{e}_n\in\R^{(M+1)SAH}$ to denote the unit vector in the $n$-th coordinate. In the following, we assume  w.l.o.g. that the $t$-th coordinate of the vector $\bs{d}_h^k$ represents the state that was visited on the $t$-th time step (while unvisited states can be arbitrarily ordered). As done by \citep[][Lemma 7]{tennenholtz2022reinforcement}, this can be done using any permutation matrix $\bs{P}_k$ such that $\bs{e}_{x_t^k,,s_t^k,a_t^k,t}=\bs{P}_k \bs{e}_t$ for all $t\in\brk[s]{H}$. Then, denoting $\bar{\bs{e}}_t = \Halphainvhalf\sum_{n=1}^t \alpha^{t-n} \bs{e}_{n}=\brk1{\underbrace{\alpha^{t-1},\alpha^{t-2},\dots,1}_{t-\mathrm{elements}},0,\dots,0}^T$, we can write $\bs{d}_t^k = \Halphainvhalf  \sum_{n=1}^t \alpha^{t-n}\bs{e}_{x_n^k,s_n^k,a_n^k,n} = \Halphainvhalf \sum_{n=1}^t \alpha^{t-n}\bs{P}_k\bs{e}_n =  \bs{P}_k\bar{\bs{e}}_t$.

Now, recalling that permutation matrices are orthogonal ($\bs{P}_k^{-1} = \bs{P}_k^T$) we can write
\begin{align*}
    \bs{e}_{x,s,a,h}^T\brk*{\lambda \bs{I} + \bs{D}_{k}}^{-1}\bs{e}_{x,s,a,h}
    &= \bs{e}_{x,s,a,h}^T\brk*{\lambda \bs{I} + \sum_{t=1}^H  \bs{d}_t^{k'} {\bs{d}_t^{k'}}^T}^{-1}\bs{e}_{x,s,a,h} \\
    &= \bs{e}_{x,s,a,h}^T\brk*{\lambda \bs{I} + \sum_{t=1}^H\bs{P}_k\bar{\bs{e}}_t{\bar{\bs{e}}_t}^T\bs{P}_k^T}^{-1}\bs{e}_{x,s,a,h} \\
    &= \bs{e}_{x,s,a,h}^T\brk*{\bs{P}_k\brk*{\lambda \bs{I} + \sum_{t=1}^H\bar{\bs{e}}_t{\bar{\bs{e}}_t}^T} \bs{P}_k^T}^{-1}\bs{e}_{x,s,a,h} \\
    &= \bs{e}_{x,s,a,h}^T\bs{P}_k\brk*{\lambda \bs{I} + \sum_{t=1}^H\bar{\bs{e}}_t{\bar{\bs{e}}_t}^T} ^{-1}\bs{P}_k^T\bs{e}_{x,s,a,h} \\
    & = \bs{e}_h^T\brk*{\lambda \bs{I} + \sum_{t=1}^H\bar{\bs{e}}_t{\bar{\bs{e}}_t}^T} ^{-1}\bs{e}_h\enspace.
\end{align*}

Next, notice that $\lambda \bs{I} + \sum_{t=1}^H\bar{\bs{e}}_t{\bar{\bs{e}}_t}^T$ is a block-diagonal matrix, whose first block is of size $H\times H$ (and the rest of the matrix is fully diagonal). We denote this first block of $\sum_{t=1}^H\bar{\bs{e}}_t{\bar{\bs{e}}_t}^T$ by $\bs{C}$. For block-diagonal matrices, each block can be inverted independently of the other blocks, and for any coordinate $h\in[H]$, if $\bs{u}_h\in\R^H$ is the unit vector at coordinate $h$, we thus have 
\begin{align}
    \bs{e}_{x,s,a,h}^T\brk*{\lambda \bs{I} + \bs{D}_{k}}^{-1}\bs{e}_{x,s,a,h}
    & = 
    \bs{e}_h^T\brk*{\lambda \bs{I} + \sum_{t=1}^H\bar{\bs{e}}_t{\bar{\bs{e}}_t}^T} ^{-1}\bs{e}_h \nonumber\\
    &= \bs{u}_h^T\brk*{\lambda \bs{I} + \bs{C}} ^{-1}\bs{u}_h \nonumber\\
    &\leq \frac{\norm{\bs{u}_h}_2^2}{\lambda_{\min}(\lambda \bs{I} + \bs{C})} \nonumber\\
    &= \frac{1}{\lambda + \lambda_{\min}(\bs{C})}
    \label{eq:norm at a unit direction bound}
\end{align}
In the rest of the proof, we focus on bounding $\lambda_{\min}(\bs{C})$. First, observe that for any $t\in[H]$, we have 
\begin{align*}
   (\bar{\bs{e}}_t\bar{\bs{e}}_t^T)(i,j)     
    = 
    \begin{cases}
        \Halphainv \alpha^{2t-i-j} &i\le t, j\le t \\
        0 & else
    \end{cases}
\end{align*}
and thus
\begin{align*}
  \bs{C}(i,j) 
  = \sum_{t=1}^H (\bar{\bs{e}}_t\bar{\bs{e}}_t^T)(i,j)
  = \Halphainv \sum_{t=\max\brk[c]{i,j}}^H \alpha^{2t-i-j}\enspace.
\end{align*}
In particular, notice that for any $i < j$ (above diagonal), we have $\bs{C}(i,j) = \alpha \bs{C}(i+1,j)$, while for $i\ge j$ (below and on diagonal), we have $\bs{C}(i,j) = \alpha \bs{C}(i+1,j) + \alpha^{i-j}$. 
Using this structure, we can calculate its inverse using diagonalization:
\begin{align*}
    &\brk*{
    \begin{array}{ccccc|ccccc}
        \sum_{t=1}^H \alpha^{2t-2} & \sum_{t=2}^H \alpha^{2t-3}    & \sum_{t=3}^H \alpha^{2t-4} & \dots & \alpha^{H-1} & 1 & 0 & 0 & \dots & 0\\
        \sum_{t=2}^H \alpha^{2t-3} & \sum_{t=2}^H \alpha^{2t-4} & \sum_{t=3}^H \alpha^{2t-5} & \dots & \alpha^{H-2} & 0 & 1 & 0 & \dots & 0\\
        \sum_{t=3}^H \alpha^{2t-4} & \sum_{t=3}^H \alpha^{2t-5} & \sum_{t=3}^H \alpha^{2t-6} & \dots & \alpha^{H-3} & 0 & 0 & 1 & \dots & 0\\
        \vdots & \vdots & \vdots & \vdots & \vdots & \vdots & \vdots & \vdots & \vdots & \vdots\\
        \alpha^{H-1} & \alpha^{H-2} & \alpha^{H-3} & \dots & 1 & 0 & 0 & 0 & \dots & 1
    \end{array}}\\
    &=\brk*{
    \begin{array}{ccccc|ccccc}
        1 & 0    & 0 & \dots & 0 & 1 & -\alpha & 0 & \dots & 0\\
        \alpha & 1 & 0 & \dots & 0 & 0 & 1 & -\alpha & \dots & 0\\
        \alpha^2 & \alpha & 1 & \dots & 0 & 0 & 0 & 1 & \dots & 0\\
        \vdots & \vdots & \vdots & \vdots & \vdots & \vdots & \vdots & \vdots & \vdots & \vdots\\
       \alpha^{H-1} & \alpha^{H-2} & \alpha^{H-3} & \dots & 1 & 0 & 0 & 0 & \dots & 1
    \end{array}} \\
    &=\brk*{
    \begin{array}{ccccc|ccccc}
        1 & 0    & 0 & \dots & 0 & 1 & -\alpha & 0 & \dots & 0\\
        0 & 1 & 0 & \dots & 0 & -\alpha & 1+\alpha^2 & -\alpha & \dots & 0\\
        0 & 0 & 1 & \dots & 0 & 0 & -\alpha & 1+\alpha^2 & \dots & 0\\
        \vdots & \vdots & \vdots & \vdots & \vdots & \vdots & \vdots & \vdots & \vdots & \vdots\\
        0 & 0 & 0 & 0 & 1 & 0 & 0 & 0 & \dots & 1+\alpha^2
    \end{array}}
\end{align*}
In the first relation, we subtracted $\alpha$-times the $i+1$ rows from the $i$ rows, while in the second one, we subtracted $\alpha$-times the $i-1$ rows from the $i$ rows. Thus, the inverse can be explicitly written as:
\begin{align*}
    \bs{C}^{-1}_{i,j} = \Halpha
    \left\{
    \begin{array}{ll}
         1 & i=j=1 \\
         1+\alpha^2 & i=j>1 \\
         -\alpha & i=j-1 \; \text{or} \; i=j+1 \\
         0 & \text{o.w.}
    \end{array}
    \right.
\end{align*}
Notice that the absolute values of all rows is smaller than $\Halpha(1+\alpha)^2\le 4\Halpha$. Then (e.g., by Gershgorin circle theorem), $\lambda_{\max}(\bs{C}^{-1})\le 4\Halpha$, and since $B$ is PSD, $\lambda_{\min}(\bs{C})\ge\frac{1}{4\Halpha}$. The proof is concluded by substituting this result back into \Cref{eq:norm at a unit direction bound}.
\end{proof}

We are now ready to prove the local concentration results for the latent features of a DCMDP:

\begin{lemma}[Local Estimation Confidence Bound]
\label{thm: local cost confidence front}
Let $\estparamvec^k \in \arg\max_{\paramvec \in \F} \mathcal{L}^k_\lambda(\paramvec)$ be the maximum likelihood estimate of the features. Then, for any $\delta>0$, with probability of at least $1-\delta$, for all $k\in [K], h\in[H], i\in [M]$ and  $s,a,x\in\s\times\A\times\X$, it holds that
\begin{align*}
    \abs{\hat \paramfunc_{i,h}^k(s,a,x) - \paramfunctrue_{i,h}(s,a,x) } 
    \leq  
    \frac{2\gamma_k(\delta)\sqrt{\kappamin\Halpha}}{\sqrt{n_h^k(s,a,x) + 4\lambda\Halpha}},
\end{align*}
where $\gamma_k(\delta)$ is defined in \Cref{prop:convex relaxation}.
\end{lemma}
\begin{proof}
The proof follows Lemma 6 of \citep{tennenholtz2022reinforcement}.

For any $k\in [K], h\in[H], i\in [M]$ and  $s,a,x\in\s\times\A\times\X$, let $\bs{e}_{x,s,a,h}\in\R^{(M+1)SAH}$ be a unit vector in the $(x,s,a,h)$ coordinate and denote $\paramvec_i\in\R^{(M+1)SAH}$ the latent features that correspond to a next latent state $i$.  
We start by bounding
\begin{align}
    \abs{\hat \paramfunc_{i,h}^k(s,a,x) - \paramfunctrue_{i,h}(s,a,x) }   
    &= \abs{\brk[a]*{\bs{e}_{x,s,a,h}, \hat \paramvec_i - \paramvectrue_i}}  \nonumber\\
    & \leq \norm{\bs{e}_{x,s,a,h}}_{\bs{V}_k^{-1}} \norm{\hat \paramvec_i - \paramvectrue_i}_{\bs{V}_k} \tag{Cauchy-Schwartz} \\
    &  \leq \norm{\bs{e}_{x,s,a,h}}_{\bs{V}_k^{-1}} \sqrt{\sum_{i'=1}^M\norm{\hat \paramvec_{i'} - \paramvectrue_{i'}}_{\bs{V}_k}^2} \nonumber\\
    & = \norm{\bs{e}_{x,s,a,h}}_{\bs{V}_k^{-1}}\norm{\hat \paramvec - \paramvectrue}_{\bs{I}_M\otimes \bs{V}_k} ,
    \label{eq: single hsa concentration part 1}
\end{align}
We now turn our focus to bound $\norm{\bs{e}_{x,s,a,h}}_{\bs{V}_k^{-1}}$. Using the notation $\bs{D}_k$, as defined in \Cref{lemma: Inverse Eigenvalues bound}, we have

\begin{align*}
    \bs{V}_k^{-1} 
    & = \brk*{\lambda \bs{I} + \sum_{k'=1}^{k-1} \bs{D}_{k'} \indicator{(s,a,x)\in \hist_h^{k'}}}^{-1}\\
    & \preceq 
    \brk*{\sum_{k'=1}^{k-1}  \brk*{\frac{\lambda }{n_h(s,a,x)} \bs{I} + \bs{D}_{k'}}\indicator{(s,a,x)\in \hist_h^{k'}}}^{-1}
    \\ & \preceq  \frac{1}{\brk*{n_h(s,a,x)}^2}
    \sum_{k'=1}^{k-1}  \brk*{\frac{\lambda}{n_h(s,a,x)} \bs{I} +  \bs{D}_{k'}}^{-1}\indicator{(s,a,x)\in \hist_h^{k'}}\enspace,
\end{align*}
where $n_h^k(s,a,x) = \sum_{k'=1}^{k-1} \indicator{(s,a,x)\in \hist_h^{k'}}$ and the third transition is due to HM-AM inequality for positive matrices \cite{bhagwat1978inequalities}. Next, we combine this result with  \Cref{lemma: Inverse Eigenvalues bound} and get
\begin{align*}
     \norm{\bs{e}_{x,s,a,h}}_{\bs{V}_k^{-1}}^2 
     &= \bs{e}_{x,s,a,h}^T \bs{V}_k^{-1} \bs{e}_{x,s,a,h}
     \\ 
     &
     \leq \frac{1}{\brk*{n_h^k(s,a,x)}^2}
    \sum_{k'=1}^{k-1}  \bs{e}_{x,s,a,h}^T\brk*{\frac{\lambda }{n_h(s,a,x)} \bs{I} + \bs{D}_{k'}}^{-1}\bs{e}_{x,s,a,h}\indicator{(s,a,x)\in \hist_h^{k'}} \\
    & \leq 
    \frac{1}{\brk*{n_h(s,a,x)}^2}
    \sum_{k'=1}^{k-1}  \frac{1}{\frac{1}{4\Halpha} + \frac{\lambda}{n_h(s,a,x)}}\indicator{(s,a,x)\in \hist_h^{k'}} \tag{\Cref{lemma: Inverse Eigenvalues bound}}\\
    & = \frac{n_h(s,a,x)}{\brk*{n_h(s,a,x)}^2}
    \frac{1}{\frac{1}{4\Halpha} + \frac{\lambda}{n_h(s,a,x)}} \\
    & = 
    \frac{1}{\frac{n_h(s,a,x)}{4\Halpha} + \lambda} \\
    & = \frac{4\Halpha}{n_h(s,a,x) + 4\lambda\Halpha}.
\end{align*}

By plugging into \cref{eq: single hsa concentration part 1}, we obtain that for any $k$ and any $h,s,a$
\begin{align*}
    \abs{\hat \paramfunc_{i,h}^k(s,a,x) - \paramfunctrue_{i,h}(s,a,x) } 
    &\leq  \norm{\bs{e}_{x,s,a,h}}_{\bs{V}_k^{-1}}\norm{\hat \paramvec - \paramvectrue}_{\bs{I}_M\otimes \bs{V}_k} \\
    &\leq \norm{\hat \paramvec - \paramvectrue}_{\bs{I}_M\otimes \bs{V}_k}\frac{2\sqrt{\Halpha}}{\sqrt{n_h(s,a,x) + 4\lambda\Halpha}} \\
    & \leq \norm{\hat \paramvec - \paramvectrue}_{\bs{H}_k(\paramvectrue)}\frac{2\sqrt{\kappamin\Halpha}}{\sqrt{n_h(s,a,x) + 4\lambda\Halpha}}.
\end{align*}
Finally by \Cref{prop:convex relaxation}, with probability $1-\delta$, for all $k\ge1$, it holds that  $\norm{\hat\paramvec-\paramvectrue}_{{\bs{H}}_k(\paramvectrue)} \le \gamma_k(\delta)$, and substituting this bound concludes the proof.
\end{proof}

\end{document}